\documentclass[review,hidelinks,onefignum,onetabnum]{siamart250211}

\usepackage[numbers,sort]{natbib} 
\definecolor{darkblue}{rgb}{0,0.22,0.66}
\definecolor{linkgray}{rgb}{0.36, 0.36, 0.36}
\hypersetup{colorlinks=true,
citecolor=cyan,linkcolor=darkblue,urlcolor=magenta}
\hypersetup{allcolors=linkgray}
\usepackage{amsmath,amsfonts,amssymb}
\usepackage[scr=rsfs]{mathalpha} 
\usepackage{mathtools}
\usepackage{dsfont} 
\usepackage{bm}

\usepackage{xurl,xcolor}
\usepackage{enumitem}
\setlist[enumerate]{label={\arabic*.}, leftmargin=2em, labelwidth=*, labelsep=0.5em,itemsep=4pt}

\usepackage{tabularx,multirow,array,booktabs}
\usepackage{colortbl}
\usepackage{graphicx}
\graphicspath{{figures/}}%
\usepackage{float}
\usepackage{subcaption} 

\usepackage{caption}
\usepackage{etoolbox,xparse,dsfont,bm,amssymb}
\newcommand{\myMFabc}[4]{\expandafter#1\csname#3#4\endcsname{{#2{#4}}}}
\newcommand{\myMFcmd}[4]{\expandafter#1\csname#3#4\endcsname{{#2{\csname#4\endcsname}}}}

\newcommand{\MFabc}[3][\newcommand]{
    \def\doOld##1##2{\forcsvlist{\myMFabc{#1}{##1}{##2}}{#3}}
    \providecommand{\do}{do}
    \RenewDocumentCommand \do { >{\SplitList{,}} m } { \doOld##1 }
    \docsvlist{#2}
}
\newcommand{\MFcmd}[3][\newcommand]{
    \def\doOld##1##2{\forcsvlist{\myMFcmd{#1}{##1}{##2}}{#3}}
    \providecommand{\do}{do}
    \RenewDocumentCommand \do { >{\SplitList{,}} m } { \doOld##1 }
    \docsvlist{#2}
}

\newcommand{\bmzero}{{\bm{0}}}

\let\one\bbone
\MFabc{ {\mathfrak,frak}, }{T,u,U,o,O,E,m}
\MFabc{ {\mathbb, }, }{A,N,Z,R,C,Q}

\newcommand{\hatbm}[1]{\widehat{\bm{#1}}}
\newcommand{\tildebm}[1]{\widetilde{\bm{#1}}}
\newcommand{\bmcal}[1]{\bm{\mathcal{#1}}}
\newcommand{\caltilde}[1]{\mathcal{\widetilde{#1}}}
\newcommand{\calhat}[1]{\mathcal{\widehat{#1}}} 
\newcommand{\scrtilde}[1]{\widetilde{\mathscr{#1}\mspace{1mu}\mspace{-1mu}}}
\newcommand{\scrhat}[1]{\mathscr{\widehat{#1}\mspace{1mu}\mspace{-1mu}}}
\newcommand{\bmcalhat}[1]{\bm{\mathcal{\widehat{#1}}}}
\newcommand{\bmcaltilde}[1]{\bm{\mathcal{\widetilde{#1}}}}
\MFabc{ {\mathcal,cal}, {\mathscr,scr}, {\mathbb,bb}, {\bmcal,bmcal}, {\bmcal,calbm}, {\caltilde,caltilde}, {\caltilde,tildecal}, {\calhat,calhat}, {\calhat,hatcal}, {\scrtilde,scrtilde}, {\scrtilde,tildescr}, {\scrhat,scrhat}, {\scrhat,hatscr}, {\bmcalhat,bmcalhat}, {\bmcalhat,bmhatcal}, 
{\bmcalhat,calbmhat}, {\bmcalhat,calhatbm}, {\bmcalhat,hatbmcal}, {\bmcalhat,hatcalbm}, {\bmcaltilde,bmcaltilde}, {\bmcaltilde,bmtildecal}, {\bmcaltilde,calbmtilde}, {\bmcaltilde,caltildebm}, {\bmcaltilde,tildebmcal}, {\bmcaltilde,tildecalbm}}{A,B,C,D,E,F,G,H,I,J,K,L,M,N,O,P,Q,R,S,T,U,V,W,X,Y,Z}

\MFabc{{\bm,bm}, {\mathsf,sf}, {\widehat,hat}, {\widetilde,tilde}, {\hatbm,hatbm}, {\hatbm,bmhat}, {\tildebm,tildebm}, {\tildebm,bmtilde}}{a,b,c,d,e,f,g,h,i,j,k,l,m,n,o,p,q,r,s,t,u,v,w,x,y,z,A,B,C,D,E,F,G,H,I,J,K,L,M,N,O,P,Q,R,S,T,U,V,W,X,Y,Z}

\MFcmd{{\bm,bm}, {\widehat,hat}, {\widetilde,tilde}, {\tildebm, tildebm}, {\tildebm, bmtilde}, {\hatbm, hatbm}, {\hatbm, bmhat}}{alpha,beta,gamma,delta,epsilon,zeta,eta,theta,iota,kappa,lambda,mu,nu,xi,omicron,pi,rho,sigma,tau,upsilon,phi,chi,psi,omega,Alpha,Beta,Gamma,Delta,Epsilon,Zeta,Eta,Theta,Iota,Kappa,Lambda,Mu,Nu,Xi,Omicron,Pi,Rho,Sigma,Tau,Upsilon,Phi,Chi,Psi,Omega,varrho,varphi,vartheta,varepsilon,varsigma,ell,eps}

\newcommand{\actdef}[1]{\expandafter\def\csname#1\endcsname{{\ensuremath{\mathtt{#1}}}}}
\forcsvlist{\actdef}{ReLU, LReLU, LeakyReLU, ELU, GELU, SiLU, Softplus, dGELU, dSiLU, dSoftplus, Tanh, Sigmoid, Arctan, Softsign, SRS, dSRS, Swish, dSwish, Mish, dMish, SELU, CELU, dSELU}

\newlength{\myLength}

\let\ts\intercal
\def\ts{\textsf{T}}
\newcommand{\din}{{d_\tn{in}}}
\newcommand{\dout}{{d_\tn{out}}}

\usepackage{tikz}

\definecolor{mylinenumbercolor}{HTML}{BEBEBE}

\linenumbers

\usepackage{doi}
\let\tilde\widetilde

\let\epsilon\varepsilon

\let\tn\textnormal

\let\cdots\customcdots
\let\dots\cdots
\let\myforall\forall
\def\forall{{\myforall\, }}
\let\myexists\exists
\def\exists{{\myexists\, }}

\definecolor{mygray}{RGB}{230,230,230}
\definecolor{myorange}{HTML}{ff7f0e}

\let\cite\citep
\usepackage{doi}

\usepackage{tikz}
\usetikzlibrary{svg.path}
\usepackage{xcolor}
\usepackage{hyperref}
\usepackage{graphicx}

\definecolor{orcidlogocol}{HTML}{A6CE39}
\tikzset{
	orcidlogo/.pic={
		\fill[orcidlogocol] svg{M256,128c0,70.7-57.3,128-128,128C57.3,256,0,198.7,0,128C0,57.3,57.3,0,128,0C198.7,0,256,57.3,256,128z};
		\fill[white] 
		svg{M86.3,186.2H70.9V79.1h15.4v48.4V186.2z}
		svg{M108.9,79.1h41.6c39.6,0,57,28.3,57,53.6c0,27.5-21.5,53.6-56.8,53.6h-41.8V79.1z 
			M124.3,172.4h24.5c34.9,0,42.9-26.5,42.9-39.7
			c0-21.5-13.7-39.7-43.7-39.7h-23.7V172.4z}
		svg{M88.7,56.8c0,5.5-4.5,10.1-10.1,10.1
			c-5.6,0-10.1-4.6-10.1-10.1c0-5.6,4.5-10.1,
			10.1-10.1C84.2,46.7,88.7,51.3,88.7,56.8z};
	}
}
\makeatletter
\newcommand{\@OrigHeightRecip}{0.00390625}
\newlength{\@curXheight}

\DeclareRobustCommand\ORCID[1]{%
	\texorpdfstring{%
		\setlength{\@curXheight}{\fontcharht\font`X}%
		\href{https://orcid.org/#1}{%
			\mbox{%
				\begin{tikzpicture}[yscale=-\@OrigHeightRecip*\@curXheight,
					xscale=\@OrigHeightRecip*\@curXheight,
					transform shape]
					\pic{orcidlogo};
				\end{tikzpicture}%
			}%
		}%
	}{}%
}

\def \@fpsadddefault {%
	\edef \@fps {\@fps\csname fps@\@captype \endcsname}%
}
\makeatother

\renewcommand{\footercopyright}{}
\nolinenumbers

\usepackage{lipsum}
\usepackage{amsfonts}
\usepackage{graphicx}
\usepackage{epstopdf}
\usepackage{algorithmic}
\ifpdf
  \DeclareGraphicsExtensions{.eps,.pdf,.png,.jpg}
\else
  \DeclareGraphicsExtensions{.eps}
\fi

\newsiamremark{remark}{Remark}
\newsiamremark{hypothesis}{Hypothesis}
\crefname{hypothesis}{Hypothesis}{Hypotheses}
\newsiamthm{claim}{Claim}
\newsiamremark{fact}{Fact}
\crefname{fact}{Fact}{Facts}
\headers{Multi-component and Multi-layer Neural Networks}{S. Zhang,
        H. Zhao,
        Y. Zhong, and H. Zhou}

\title{Structured and Balanced Multi-component\\ 
and Multi-layer Neural Networks
}

\author{Shijun Zhang\thanks{Corresponding author. Department of Applied Mathematics, Hong Kong Polytechnic University  (\href{mailto:shijun.zhang@polyu.edu.hk}{shijun.zhang@polyu.edu.hk})
}
\and Hongkai Zhao\thanks{Department of Mathematics, Duke University (\href{mailto:zhao@math.duke.edu}{zhao@math.duke.edu})}
\and Yimin Zhong\thanks{Department of Mathematics and Statistics, Auburn University (\href{mailto:yimin.zhong@auburn.edu}{yimin.zhong@auburn.edu}})
\and Haomin Zhou\thanks{School of Mathematics,
 Georgia Institute of Technology (\href{mailto:hmzhou@math.gatech.edu}{hmzhou@math.gatech.edu})}
 }
 \usepackage{amsopn}

\begin{document}

\maketitle

\begin{abstract}
    In this work, we propose a balanced multi-component and multi-layer neural network (MMNN) structure to accurately and efficiently approximate functions with complex features, in terms of both degrees of freedom and computational cost. The main idea is inspired by a multi-component approach, in which each component can be effectively approximated by a single-layer network, combined with a multi-layer decomposition strategy to capture the complexity of the target function. Although MMNNs can be viewed as a simple modification of fully connected neural networks (FCNNs) or multi-layer perceptrons (MLPs) by introducing balanced multi-component structures, they achieve a significant reduction in training parameters, a much more efficient training process, and improved accuracy compared to FCNNs or MLPs. Extensive numerical experiments demonstrate the effectiveness of MMNNs in approximating highly oscillatory functions and their ability to automatically adapt to localized features.
    Our code and implementations are available at \href{https://github.com/ShijunZhangMath/MMNN/}{GitHub}.
\end{abstract}



\section{Introduction}
\label{sec:intro}

The key use of neural networks is to approximate an input-to-output relation, i.e., a mapping or a function in mathematical terms. In this work, we continue our study of numerical understanding of neural network approximation of functions from representation to learning dynamics. In our earlier study \cite{ZZZZ-23}, we demonstrated that a one-hidden-layer (also known as a two-layer or shallow) network is essentially a ``low-pass filter" when approximating a function in practice. Due to the strong correlation among the family of activation functions (parameterized by the weight and bias), such as \ReLU\  (rectified linear unit), the Gram matrix, the element of which is the pairwise correlation (inner product) of the activation functions, has a fast spectral decay. If initialized randomly, the eigenvectors of the Gram matrix correspond to generalized Fourier modes from low frequency to high frequency ordered corresponding to decreasing eigenvalues. Due to the ill-conditioning of the representation, no matter how wide a one-hidden-layer network is, it can only learn and approximate smooth functions or sample low-frequency modes effectively and stably (with respect to noise or machine round-off errors). 

In this work, we propose a balanced multi-component and multi-layer neural network (MMNN) structure based on our previous understanding of a one-hidden-layer network. First, we show that a multi-layer network with a multi-component structure, each of which can be approximated well and effectively by a one-hidden-layer network, can overcome the limitation of a shallow network by smooth decomposition and transformation. Compared to a fully connected neural network of a similar structure, our proposed MMNN is much more effective in terms of representation, training, and accuracy in approximating functions, especially for functions containing complex features, e.g., high-frequency modes. The key idea of MMNNs is to view a linear combination of activation functions as randomly parameterized basis functions, called a \emph{component},  as a whole to represent a smooth function. Each layer has multiple components all sharing the common basis functions with different linear combinations. The number of components, called \emph{rank}, is typically much smaller than the layer's width and increases to enhance the flexibility of decomposition when dealing with more complex functions. These components are combined and composed (through layers) in a structured and balanced way in terms of network width, rank, and depth to approximate a complicated function effectively. 
Another important feature we used in practice is that weights and biases inside each activation function are randomly assigned and fixed during the optimization while the linear combination weights of activation functions in each component are trained. This leads to more efficient training processes motivated by our finding that a one-hidden-layer neural network can be trained effectively to approximate a smooth function well using random basis functions. We also demonstrate interesting learning dynamics based on Adam optimizer~\cite{DBLP:journals/corr/KingmaB14}, which is crucial for the successful and efficient training of MMNNs. An important remark is that a balanced and holistic approach needs to consider both representation and optimization as well as their interplay altogether.

The structure of this paper is as follows. Section~\ref{sec:MMNN} introduces and details the design of MMNNs. Section~\ref{sec:comparison} provides a comparison between FCNNs and MMNNs. In Section~\ref{sec:decomposition}, we present a mathematical framework for smooth decomposition and transformation based on the MMNN architecture, showing that each component can be effectively approximated by a single-hidden-layer network. Section~\ref{sec:experiments} presents extensive numerical experiments to validate our analysis and demonstrate the capability of MMNNs in approximating complex functions. Additional insights and implementation guidelines are discussed in Section~\ref{sec:further:discussion}. Finally, Section~\ref{sec:conclusion} concludes the paper with final remarks.




\section{Multi-component and multi-layer neural network (MMNN)}
\label{sec:MMNN}

In this section, we present a novel network architecture called the Multi-component and Multi-layer Neural Network (MMNN). 
Let's begin with some notations. Let $\R$ represent the set of real numbers. The indicator (or characteristic) function of a set $A$, denoted by $\one_{A}$, is a function that takes the value $1$ for elements in $A$ and $0$ for elements not in $A$. Vectors and matrices are denoted by bold lowercase and uppercase letters, respectively.
We use slicing notation for a vector $\bmx = (x_1, \cdots, x_d) \in \R^d$, where $\bmx{[n:m]}$ denotes a slice of $\bmx$ from its $n$-th to the $m$-th entries for any $n, m \in \{1, 2, \cdots, d\}$ with $n \le m$, and $\bmx{[n]}$ denotes the $n$-th entry of $\bmx$. For example, if $\bmx = (x_1, x_2, x_3) \in \R^3$, then $(5\bmx){[2:3]} = (5x_2, 5x_3)$ and $(6\bmx + 1){[3]} = 6x_3 + 1$. A similar notation is used for matrices. For instance, $\bmA[:, i]$ refers to the $i$-th column of $\bmA$, whereas $\bmA[i, :]$ indicates the $i$-th row of $\bmA$. Additionally, $\bmA[i, n:m]$ corresponds to $(\bmA[i, :])[n:m]$, extracting the entries from the $n$-th to the $m$-th in the $i$-th row.

Later in this section, we introduce the architecture of MMNNs in Section~\ref{sec:MMNN:architecture}. Following this, in Section~\ref{sec:MMNN:learning:stategy}, we outline the learning strategy of MMNN and highlight its advantages over other methods. 

\subsection{Architecture of MMNNs}
\label{sec:MMNN:architecture}

In this section, we introduce the architecture of our Multi-component and Multi-layer Neural Network (MMNN). Each layer of the MMNN is a (shallow) neural network of the form
\begin{equation*}
    \bmh(\bmx)=\bmA \sigma(\bmW\bmx +\bmb)+\bmc
\end{equation*}
to approximate a vector-valued function 
$\bmf:\R^\din\to\R^\dout$
where $\bmW \in \R^{n\times \din}, \bmA\in \R^{\dout\times n}$, and $n$ is the width of this network.
Here, \(\sigma:\mathbb{R}\to\mathbb{R}\) represents the activation function that can be applied elementwise to vector inputs. Throughout this paper, the activation function is chosen as \ReLU, unless otherwise specified.
One can also write it in a more compact form, 
\begin{equation}\label{eq:shallow}
    \bmh=\bmA \sigma(\bmW\bmx +\bmb)+\bmc
    =\bmtildeA\left[\begin{matrix}
        \sigma(\bmtildeW\bmtildex) \\
         1
    \end{matrix}\right],
\end{equation}
where 
\begin{equation*}
    \bmtildeW=\left[\begin{matrix}
        \bmW, \bmb
    \end{matrix}\right],\quad 
    \bmtildeA=\left[\begin{matrix}
        \bmA, \bmc
    \end{matrix}\right],\quad
    \bmtildex=\left[\begin{matrix}
        \bmx\\  1
    \end{matrix}\right]
    .
\end{equation*}
We call each element of $\bmh$,  i.e., $\bmh{[i]}=\tilde{\bm{A}}{[i,:]}\cdot\left[\begin{matrix}
        \sigma(\bmtildeW\bmtildex) \\
         1
    \end{matrix}\right]$ for $i=1,2,\cdots,d_{\tn{out}}$, a component. Here are a few key features of $\bmh$:
\begin{enumerate}
    \item Each component is viewed as a linear combination of basis functions $\sigma(\bmW{[i,:]}\cdot\bmx+\bmb[i]),\, i=1,2,\cdots,n$, which is a function in $\bmx$, as a whole.
    \item Different components of $\bmh$ share the same set of basis with different coefficients $\tilde{\bm{A}}[i,:]$.
    \item Only $\bmtildeA$ is trained while $\bmtildeW$ is randomly assigned and fixed. 
    \item The output dimension $\dout$ and network width $n$ can be tuned according to the intrinsic dimension and complexity of the problem.
\end{enumerate}
In comparison, each layer in a typical deep FCNN takes the form $\sigma(\bmtildeW\bmtildex)$, and each hidden neuron is individually a function of the input $\bmx$ or each point $\bmx\in \R^{d_{\tn{in}}}$ is mapped to $\R^{n}$, where $n$ is the layer width. All weights $\bmtildeW$ are training parameters. In MMNNs, each layer is composed of multiple components $\bmtildeA\sigma(\bmtildeW\bmtildex)$. Each component is a linear combination of randomly parameterized hidden neurons $\sigma(\bmtildeW \bmtildex)$, which can be more effectively and stably trained through $\bmtildeA$ as a smooth decomposition/transformation. Typically the number of components $d_{\tn{out}}$, the dimension of intermediate feature space, is (much) smaller than the layer width $n$, the number of random neurons (or basis functions) in the MMNN. On one hand, the intermediate feature space is compressed, and on the other hand, there are diverse random basis whose linear combinations can have enough representation power in the feature space.





An MMNN is a multi-layer composition of $\bmh_i$, i.e., $\bmh:\R^\din \mapsto\R^\dout$
\begin{equation*}
    \bmh=\bmh_m\circ \cdots \circ \bmh_2\circ \bmh_1,
\end{equation*}
where each $\bmh_i:\R^{d_{i-1}}\mapsto\R^{d_{i}}$ is a multi-component shallow network defined in~\eqref{eq:shallow} of width $n_i$, where 
\begin{equation*}
    d_0=\din,\qquad d_1,\cdots, d_{m-1}\ll   n_i,\qquad d_m=\dout.
\end{equation*}
The width of this MMNN is defined as $\max\{n_i: i=1,2,\cdots,m-1\}$, the rank as $\max\{d_i: i=1,2,\cdots,m-1\}$, and the depth as $m$.
MMNNs are designed to reduce the dimensionality of the intermediate feature space, thereby simplifying optimization while largely preserving the network's expressive power.
To simplify, we denote a network with width $w$, rank $r$, and depth $l$ using the compact notation $(w,r,l)$.
See 
Figure~\ref{fig:MMNN:vs:FC}(a)
for an illustration of an MMNN of size $(4,2,2)$.
In contrast, an FCNN $\bmphi$ can be expressed in the following composition form
	\begin{equation*}
		\bmphi =\calbmL_L\circ\sigma\circ
		\calbmL_{L-1}\circ 
		\ \cdots \  \circ 
		\sigma\circ
		\calbmL_1\circ\sigma\circ\calbmL_0,
	\end{equation*}
 where $\calbmL_i$ is an affine linear map given by $\calbmL_i(\bmy)=\bmW_i\cdot \bmy+\bmb_i$. Readers are referred to 
 Figure~\ref{fig:MMNN:vs:FC}(b)
 for an illustration and a comparison with the MMNN.

We clarify the structural differences between FCNNs and MMNNs, omitting bias terms for notational simplicity.
In MMNNs, each layer is defined by a composition
\[
\bmh: \mathbb{R}^r \to \mathbb{R}^r, \quad \bmh(\bmx) = \bmA\sigma(\bmW\bmx),
\]
while in standard FCNNs, layers are typically written as
\[
\bmh: \mathbb{R}^n \to \mathbb{R}^n, \quad \bmh(\bmx) = \sigma(\bmW\bmx),
\]
where \( r \ll n \), and \( r \) denotes the MMNN rank, an internal dimensionality that helps regulate network complexity.
The core principle behind MMNNs is to limit the dimensionality of intermediate space, which helps streamline the optimization process while maintaining the model's expressive power.


One might suggest that the product \(\bmA_j \bmW_{j+1}\) in MMNNs can be collapsed into a single matrix, making MMNNs a special case of FCNNs. However, this overlooks two critical distinctions:

\begin{enumerate}
    \item \textbf{Asymmetric Parameter Roles:} In MMNNs, \(\bmW\) is randomly initialized and fixed while \(\bmA\) is learnable. This asymmetry is central to our representation and training strategy and cannot be replicated by simply reinterpreting MMNNs as FCNNs with a low rank factorization of the weight matrix $\bmW$ which needs to be learned fully posing a challenging task for the optimization.
    
    \item \textbf{Architectural Interventions:} In practice, modern networks often employ techniques such as batch normalization, dropout, and residual connections, which fundamentally alter the layer-wise composition. For instance,
    \begin{itemize}
        \item With batch normalization layers $B_i$'s:
        \[
        \bmh_m \circ B_{m-1} \circ \bmh_{m-1} \circ \cdots \circ B_1 \circ \bmh_1;
        \]
        \item With dropout layers $D_i$'s:
        \[
        \bmh_m \circ D_{m-1} \circ \bmh_{m-1} \circ \cdots \circ D_1 \circ \bmh_1;
        \]
        \item With residual connections:
        \[
        \bmh_m \circ (\bmI + \bmh_{m-1}) \circ \cdots \circ (\bmI + \bmh_2)\circ   \bmh_1.
        \]
    \end{itemize}
    In these scenarios, the clean compositional structure required to merge adjacent matrices (e.g., \(\bmA_j\) and \(\bmW_{j+1}\)) is disrupted. Thus, MMNNs remain distinct in both form and training philosophy.
\end{enumerate}
For very deep MMNNs, one can borrow ideas from ResNets \cite{7780459} to address the gradient vanishing issue, making training more efficient. Incorporating this idea, we propose a new architecture given by a multi-layer composition of $\bmI+\bmh_i$, i.e., $\bmh:\R^\din \mapsto \R^\dout$
\begin{equation*}
    \bmh = \bmh_m \circ (\bmI + \bmh_{m-1}) \circ \cdots \circ (\bmI + \bmh_3) \circ (\bmI + \bmh_2) \circ \bmh_1,
\end{equation*}
where each $\bmh_i: \R^{d_{i-1}} \mapsto \R^{d_i}$ is a multi-component shallow network defined in \eqref{eq:shallow} with width $n_i$, 
\begin{equation*}
    d_0=\din,\qquad d_1=\cdots=d_{m-1}=r\ll   n_i,\qquad d_m=\dout,
\end{equation*}
and $\bmI$ is the identity map. We call this architecture ResMMNN.
See 
Figure~\ref{fig:MMNN:vs:FC}(c)
for an illustration of a ResMMNN of size (4,2,3).

The above definition of ResMMNNs requires \(d_1 = \cdots = d_{m-1} = r\). If this condition does not hold, we can alternatively define ResMMNNs via
\begin{equation*}
    \bmh = (\bmI \oplus \bmh_m) \circ (\bmI \oplus \bmh_{m-1}) \circ \cdots \circ (\bmI \oplus \bmh_3) \circ (\bmI \oplus \bmh_2) \circ (\bmI \oplus \bmh_1),
\end{equation*}
where \(\oplus\) is an operation defined as follows. For any functions \( \bmf:\mathbb{R}^d \mapsto \mathbb{R}^{d_\bmf} \) and \( \bmg:\mathbb{R}^d \mapsto \mathbb{R}^{d_\bmg} \), the \(\oplus\) operation is given by
\[
\bmf \oplus \bmg \coloneqq (\tildebmf + \tildebmg)[1:d_\bmg],
\]
where
\[
\tildebmf = \begin{bmatrix}
   \bmf \\  
   \bmzero 
\end{bmatrix}
\in \mathbb{R}^{\max\{d_\bmf,d_\bmg\}}
\quad \tn{and}\quad \tildebmg = \begin{bmatrix}
   \bmg \\  
   \bmzero 
\end{bmatrix}
\in \mathbb{R}^{\max\{d_\bmf,d_\bmg\}}.
\]





\begin{figure}
    \centering
    \begin{subfigure}[c]{0.98\textwidth}
    \centering                \includegraphics[width=0.758025\textwidth]{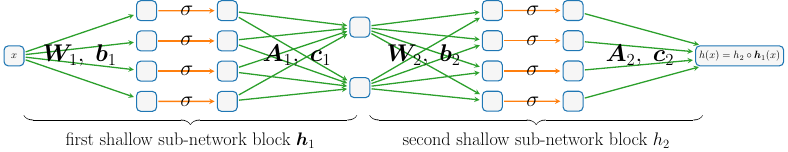}
    \subcaption{MMNN of size $(4,2,2)$, i.e., width $4$, rank 2, and depth $2$.}
    		\label{subfig:MMNN}
    \end{subfigure}\\
    \vspace{6pt}
    \begin{subfigure}[c]{0.98\textwidth}
    \centering                \includegraphics[width=0.64\textwidth]{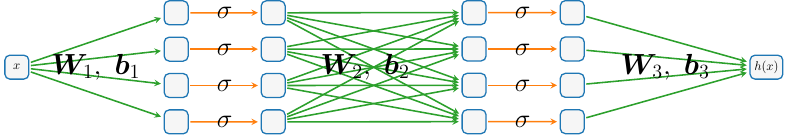}
    \subcaption{FCNN of size (4, --, 2), i.e., width $4$ and depth $2$.}
    		 \label{subfig:fc}
    \end{subfigure}
    \\
    \vspace{6pt}
    \begin{subfigure}[c]{0.98\textwidth}
    \centering                \includegraphics[width=0.9985\textwidth]{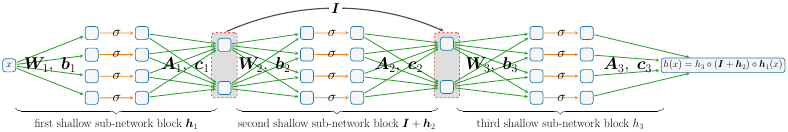}
    \subcaption{ResMMNN of size $(4,2,3)$, i.e., width $4$, rank 2, and depth $3$.}
    		\label{subfig:ResMMNN}
    \end{subfigure}
\caption{Illustrations of $\sigma$-activated MMNN, FCNN, and ResMMNN. 
}
    \label{fig:MMNN:vs:FC}
\end{figure}

\subsection{Learning strategy of MMNNs}
\label{sec:MMNN:learning:stategy}

Our learning strategy is motivated by the following basic principle: a function can be decomposed in a multi-component and multi-layer structure each component of which can be approximated and trained effectively using a one-hidden-layer network, which is a linear combination of random basis functions (e.g., of the form $\sigma(\bmW_i\cdot \bmx + \bmb_i)$, see Section~\ref{sec:decomposition}).
Therefore, optimizing the linear combination weights of the random basis functions, namely $\bmA_i$'s and $\bmc_i$'s, is both computationally efficient and sufficiently expressive.
On the other hand, optimizing the weights (orientations of the basis functions) $\bmW_i$'s and biases $\bmb_i$'s to make the basis functions more adaptive to fine-tune features of the target function, which would require capturing high-frequency information by a single layer network, leads to not only significantly more parameters to optimize but also difficulties in training as shown in~\cite{ZZZZ-23}.
Specifically, for each layer of an MMNN, we fix the activation function parameters (\(\bm{W}_i\)'s and \(\bm{b}_i\)'s) as per PyTorch's default setting during the training process. This entails initializing both weights and biases uniformly from the distribution \(\mathcal{U}(-\sqrt{k}, \sqrt{k})\), where \(k = \frac{1}{\text{in\_features}}\).\footnote{It is noteworthy that this initialization approach is similar to the widely used Xavier initialization \cite{pmlr-v9-glorot10a}, which draws weights from the distribution \(\mathcal{U}(-\sqrt{k}, \sqrt{k})\) with \(k = \frac{\sqrt{6}}{\text{in\_features} + \text{out\_features}}\) and sets the bias to $\bmzero$.}
The whole training process optimizes all \(\bmA_i\)'s and \(\bmc_i\)'s simultaneously using the Adam optimizer \cite{DBLP:journals/corr/KingmaB14}. Note that it is important to have a uniform sampling of orientations $\bmW_i$ and biases $\bmb_i$ for the random basis functions to be able to approximate an arbitrary smooth function well. Unless stated otherwise, parameter initialization adheres to the default settings provided by PyTorch in our experiments.


To demonstrate the advantages of our training approach (labeled S1), we conduct a comparison with the typical strategy in deep neural networks, denoted as Strategy S2, which uses the default PyTorch initialization and optimizes all parameters during training. In our tests, we select an oscillatory target function $f(x)=\cos(36\pi x^2)-0.6\cos(12\pi x^2)$ and use fairly compact networks. The tests are performed on a total of $1000$ uniform samples in $[-1,1]$ with a mini-batch size of $100$ and a learning rate for epoch $k$ set at $0.001 \times 0.9^{\lfloor k/400 \rfloor}$ for $k = 1, 2, \cdots, 20000$, where $\lfloor \cdot \rfloor$ denotes the floor operation. The Adam optimizer \cite{DBLP:journals/corr/KingmaB14} is applied throughout the training process.
\begin{table}[htbp!]
	\centering  
 \setlength{\tabcolsep}{0.68em} 
 \renewcommand{\arraystretch}{1.15}
\caption{
Comparison of test errors averaged over the last 100 epochs.
}
	\label{tab:stategy:error:comparison}
	\resizebox{0.95\textwidth}{!}{ 
		\begin{tabular}{ccccccccc} 
			\toprule
			    network &  (width, rank, depth) & {\#parameters (trained / all)}  &      {test error (MSE)} 
    & test error (MAX) & training time \\
			\midrule
			 \rowcolor{mygray}
			  MMNN1 (S1) & (400, 20, 6) &  40501 / 83301 & 
$2.01 \times 10^{-5}$  &  $4.36 \times 10^{-2}$
    & 23.9s / 1000 epochs \\	
   			 MMNN1 (S2) & (400, 20, 6) & 83301 / 83301 & 
$4.26 \times 10^{-5}$  &  $4.71 \times 10^{-2}$
   & 30.2s / 1000 epochs
   \\
		\rowcolor{mygray}	 MMNN2 (S1) & (590, 28, 6) &  83331 / 170061 & 
$\bm{1.39 \times 10^{-5}}$  &  $\bm{2.80 \times 10^{-2}}$
   & 25.2s / 1000 epochs\\			
			\bottomrule
		\end{tabular} 
	}
\end{table} 




\begin{figure}
    \centering	
        \begin{subfigure}[c]{0.32438\textwidth}
    \centering            \includegraphics[width=0.985\textwidth]{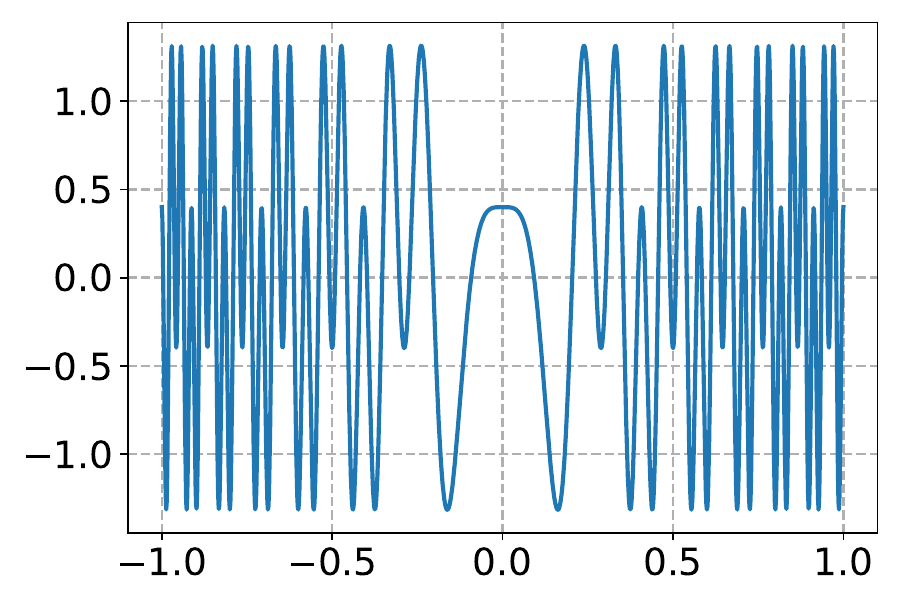}
    \end{subfigure}
    \hfill
\begin{subfigure}[c]{0.32438\textwidth}
	\centering            \includegraphics[width=0.9285\textwidth]{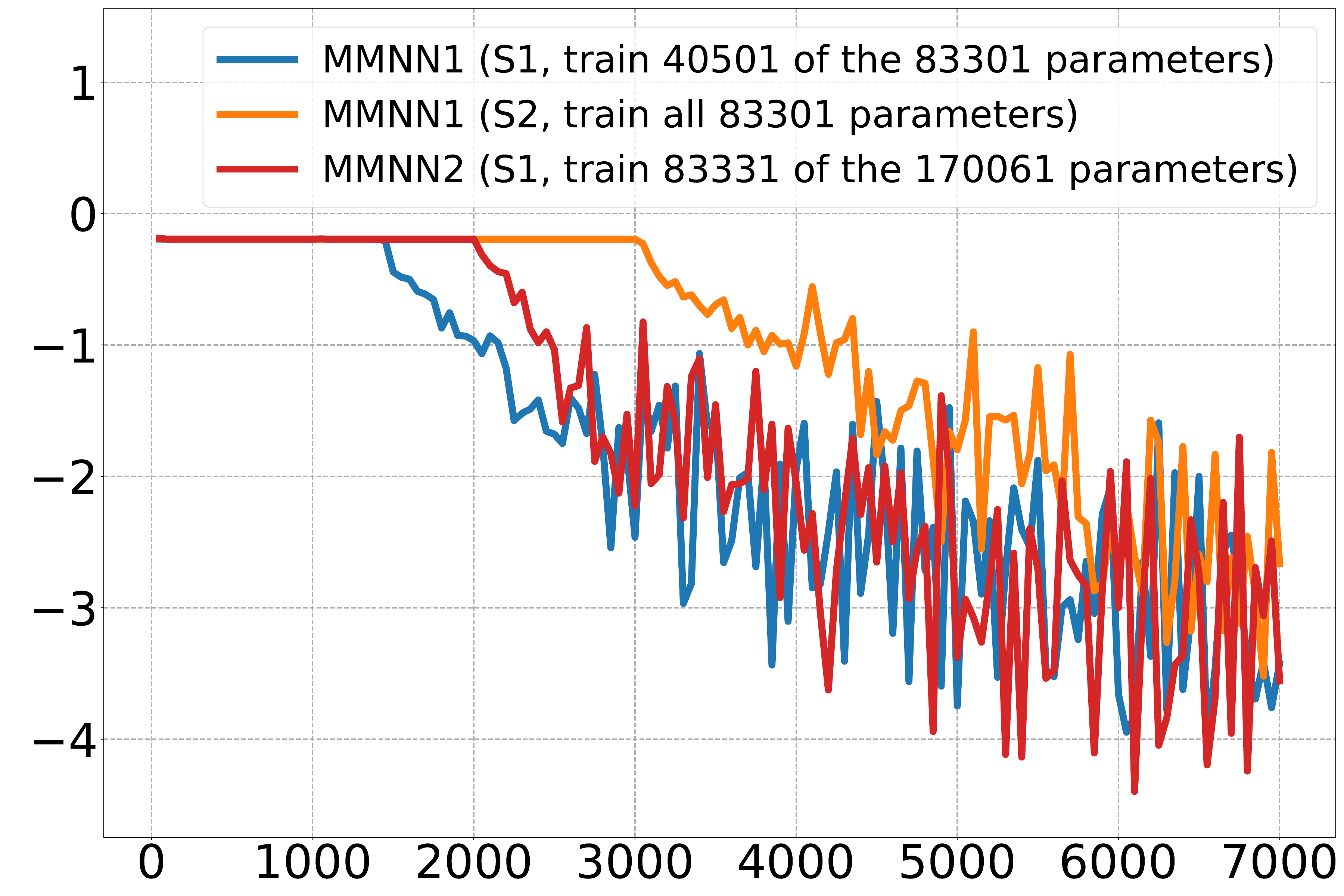}
\end{subfigure}
\hfill
\begin{subfigure}[c]{0.32438\textwidth}
    \centering            \includegraphics[width=0.9285\textwidth]{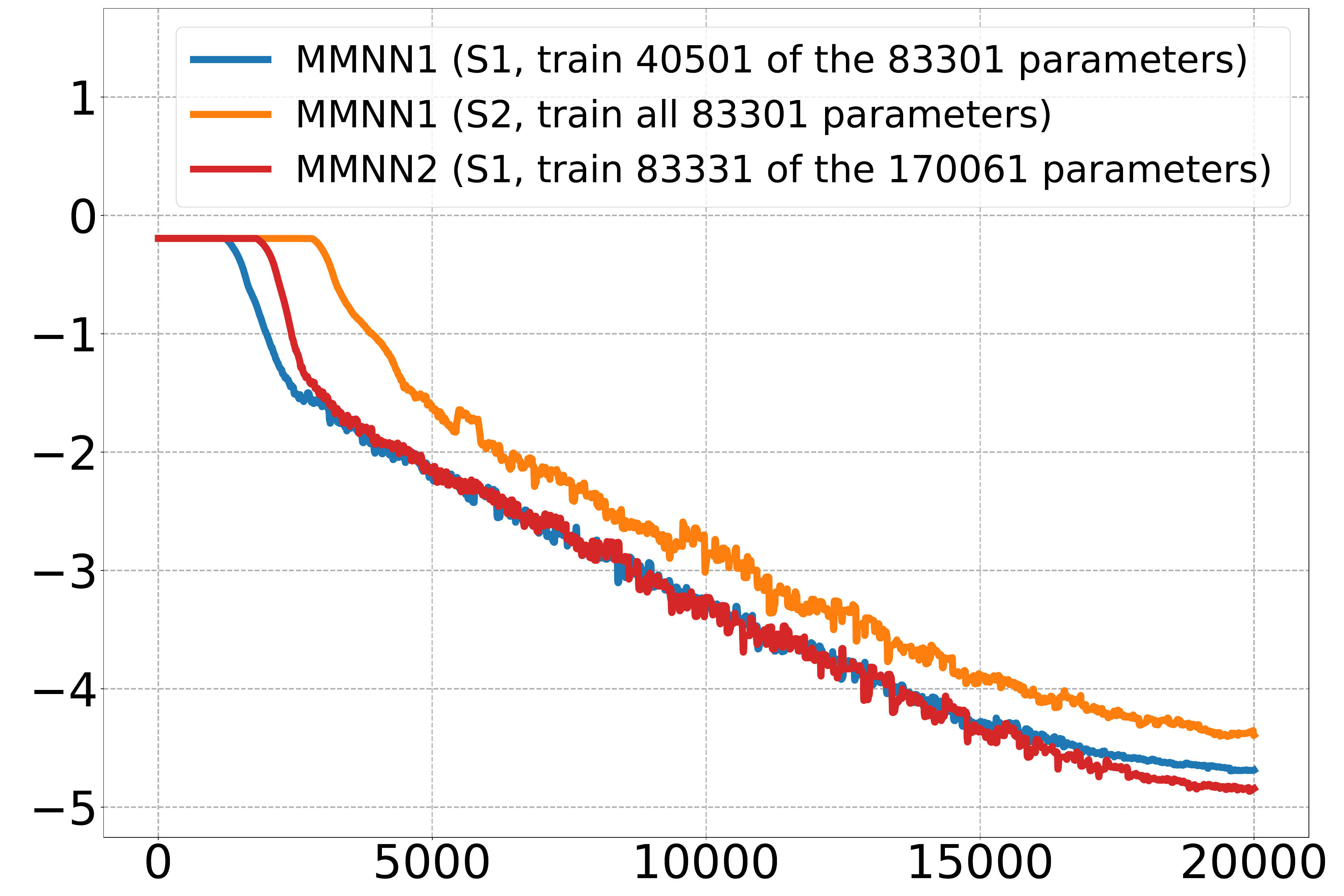}
\end{subfigure}
\caption{Left: target function $f(x)=\cos(36\pi x^2)-0.6\cos(12\pi x^2)$. Middle: base-10 logarithm of test errors vs. epoch. Right: 
base-10 logarithm of ``test-error-aver'' vs. epoch, where ``test-error-aver'' for epoch \( k \) is calculated by averaging the errors in epochs \(\max\{1, k-100\}\) to \(\min\{k+100, \#\text{epochs}\}\).
}
    \label{fig:S1:vs:S2:errors}
\end{figure}

As illustrated in Table~\ref{tab:stategy:error:comparison} and Figure~\ref{fig:S1:vs:S2:errors},  our learning strategy S\(1\) is significantly more effective than strategy S\(2\) with comparable accuracy. 
There are two main advantages of S1. First, S1 requires training only about half the number of parameters compared to S2, which results in time savings. Second, S1 converges more quickly and performs significantly better when the training is not sufficient. 
We would like to note that in certain specific cases, S2 may outperform S1, particularly when the network size is relatively small and S2 is well-trained. This is expected since S\(2\) trains all parameters, whereas S\(1\) only trains a subset.
Based on our experience, S1 is more effective in practice, particularly for sufficiently large networks.

\section{MMNNs versus FCNNs}\label{sec:comparison}

In Section~\ref{sec:MMNN}, we outlined the distinctions between MMNNs and FCNNs regarding their representation and learning approaches. Here, we evaluate their numerical performance for approximating oscillatory functions in Section~\ref{sec:1D2Dfunction} and solving PDEs in Section~\ref{sec:PDE}.
To ensure a fair comparison, we use networks with a similar number of parameters and ensure that all networks have sufficient parameters to learn the target function. Typically, when training an FCNN, all parameters are optimized. For a thorough comparison, we will employ two learning strategies for MMNNs as detailed in Section~\ref{sec:MMNN:learning:stategy}: S1 and S2. S1 involves training approximately half the number of parameters of the MMNN, while S2 involves training all parameters.

\subsection{Oscillatory function approximation}\label{sec:1D2Dfunction}

We consider  a one-dimensional function \( f_1(x) = \cos(20\pi |x|^{1.4}) + 0.5\cos(12\pi |x|^{1.6}) \) and a two-dimensional function 
\begin{equation*}
	f_{2}(x_1,x_2)= \sum_{i=1}^2\sum_{j=1}^2 a_{ij} \sin(s b_i x_i+s c_{i,j} x_i x_j) \cos (s b_j x_j + s d_{i,j} x_i^2),
\end{equation*}
 where $s=2$ and
\begin{equation*}
	(a_{i,j})=	\begin{bmatrix*}
		0.3 & 0.2 \\ 0.2 & 0.3
	\end{bmatrix*}, \qquad 
	(b_{i})=	\begin{bmatrix*}
		2\pi \\ 4\pi
	\end{bmatrix*},
	\qquad
	(c_{i,j})=	\begin{bmatrix*}
		2\pi & 4\pi \\ 8\pi & 4\pi
	\end{bmatrix*}, \quad 
	(d_{i,j})=	\begin{bmatrix*}
		4\pi & 6\pi \\ 8\pi & 6\pi
	\end{bmatrix*}.
\end{equation*}
Refer to Figures~\ref{fig:vsFCNN:f1D} and~\ref{fig:vsFCNN:f2D} for illustrations of \( f_1 \) and \( f_2 \), respectively.

\begin{figure}
\begin{minipage}[b]{0.302\linewidth}
    \centering	
    \includegraphics[width=0.925\textwidth]{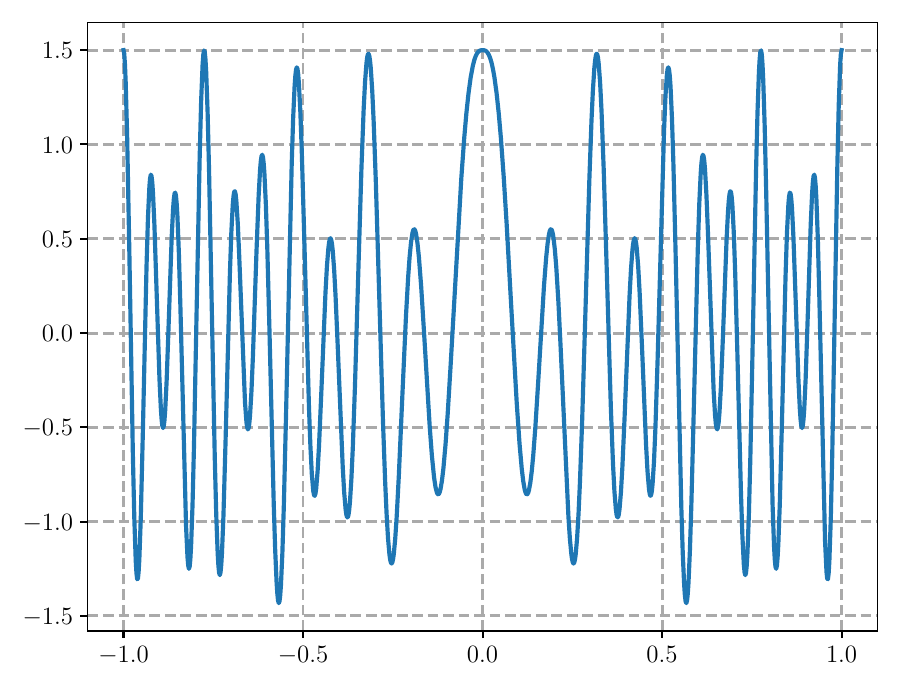}
        \caption{ Plot of $f_1$.}
    	\label{fig:vsFCNN:f1D}
    \end{minipage}
    \hfill
\begin{minipage}[b]{0.57\linewidth}
    \centering	
    \begin{subfigure}[c]{0.46\textwidth}
    \centering            \includegraphics[height=0.80\textwidth]{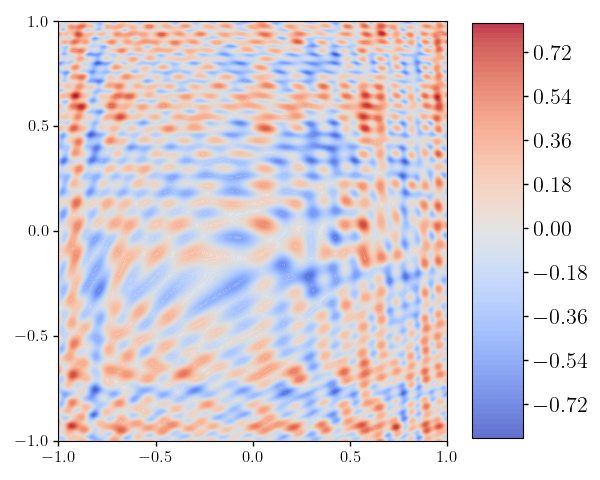}
    \end{subfigure}
    \hfill
    \begin{subfigure}[c]{0.46\textwidth}
    \centering            \includegraphics[height=0.80\textwidth]{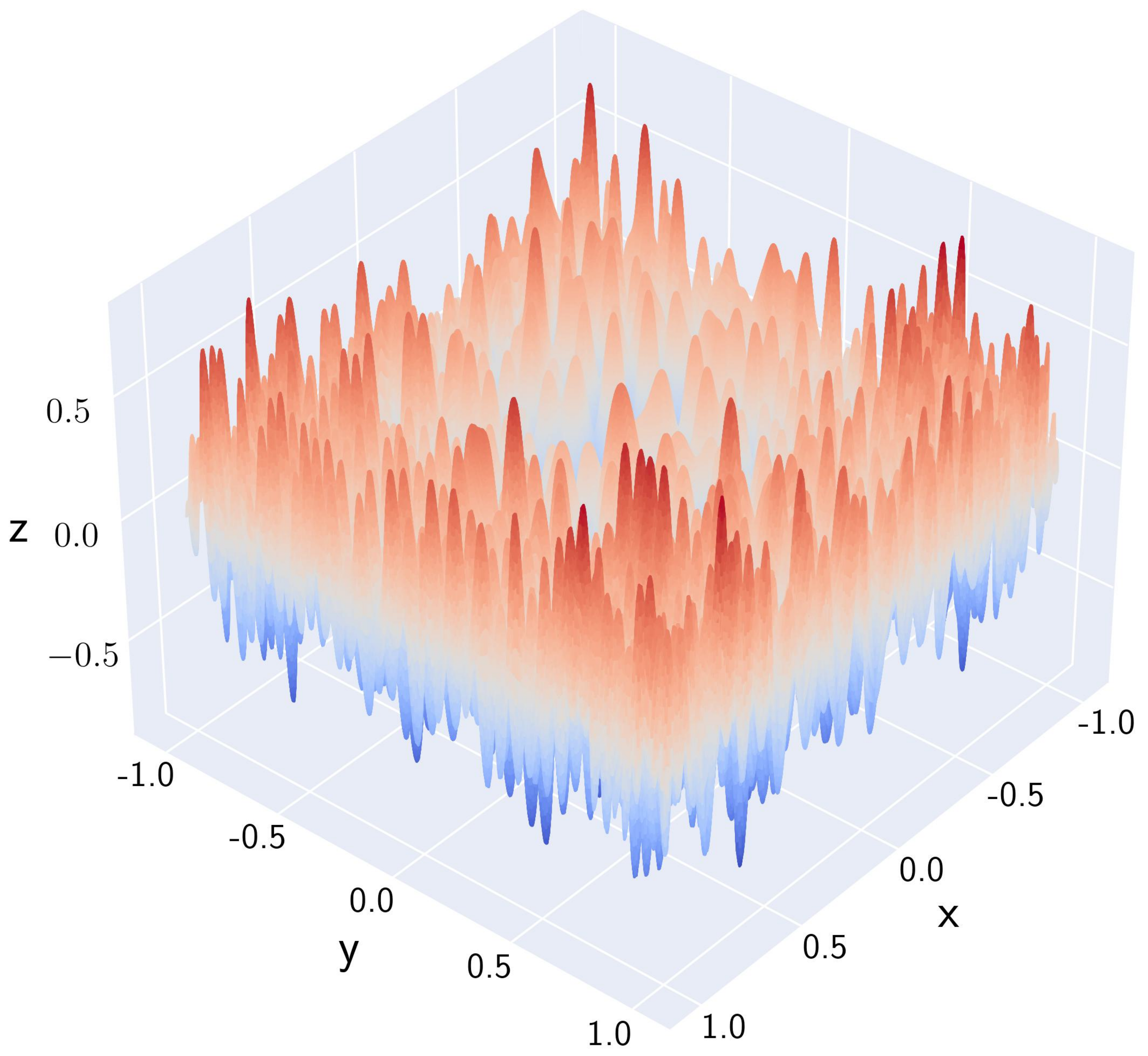}
    \end{subfigure}
    \caption{Plots of $f_2$.}
    	\label{fig:vsFCNN:f2D}
    \end{minipage}
\end{figure}

Large network sizes (see Table~\ref{tab:vsFCNN:error:comparison}) are selected to ensure that all networks possess sufficient parameters to learn the target functions.\footnote{FCNNs perform poorly if the network size is small. For a fair comparison, we choose relatively large network sizes for FCNNs and MMNNs, where both perform reasonably well.
}
For training the one-dimensional function, we sample a total of 1000 data points on a uniform grid within \([-1,1]\), using a mini-batch size of 100 and a learning rate of \(0.001 \times 0.9^{\lfloor k/400 \rfloor}\) for epochs \(k = 1, 2, \cdots, 20000\).
For training the two-dimensional function, we sample a total of \(600^2\) data points on a uniform grid within \([-1,1]^2\), using a mini-batch size of 1000 and a learning rate of \(0.001 \times 0.9^{\lfloor k/16 \rfloor}\) for epochs \(k = 1, 2, \cdots, 800\).  
The Adam optimizer is employed for both functions.

\begin{table}[htbp!]
	\centering  
 \setlength{\tabcolsep}{0.568em} 
 \renewcommand{\arraystretch}{1.15}
\caption{
Comparison of test errors averaged over the last 100 epochs.
}
	\label{tab:vsFCNN:error:comparison}
	\resizebox{0.999\textwidth}{!}{ 
		\begin{tabular}{ccccccccc} 
			\toprule
			 target function &   network &  (width, rank, depth) & {\#parameters (trained / all)}  &      {test error (MSE)} 
    & test error (MAX) & training time \\
			\midrule
			 \rowcolor{mygray}
			 $f_1$ & MMNN1 (S1) & (388, 18, 6) &  35399 / 73035 & 
$2.49 \times 10^{-6}$  &  $\bm{9.93 \times 10^{-3}}$
    & 23.3s / 1000 epochs \\	
   			$f_1$ & FCNN1-1 & (83, --,  6) & 35110 / 35110 & 
$2.43 \times 10^{-4}$  &  $1.87 \times 10^{-1}$
   & 19.5s / 1000 epochs
   \\
		\rowcolor{mygray}	$f_1$ & MMNN1 (S2) & (388, 18, 6) &  73035 / 73035 & 
$\bm{2.05 \times 10^{-6}}$  &  $1.88 \times 10^{-2}$
   & 27.4s / 1000 epochs\\			
			$f_1$ & FCNN1-2 & (120, --,  6)  & 72961 / 72961 & 
$1.73 \times 10^{-4}$  &  $1.14 \times 10^{-1}$
   & 22.3s / 1000 epochs 
   \\
			\midrule
			 \rowcolor{mygray}
			 $f_2$ & MMNN2 (S1) & (789, 36, 12) &  313630 / 637120 & 
$\bm{4.61 \times 10^{-6}}$  &  $\bm{1.55 \times 10^{-2}}$
    & 30.3s / 10 epochs \\
    			$f_2$ & FCNN2-1 & (168, --, 12) & 312985 / 312985 & 
$2.42 \times 10^{-4}$  &  $2.75 \times 10^{-1}$
   & 26.7s / 10 epochs \\
		\rowcolor{mygray}	$f_2$ & MMNN2 (S2) & (789, 36, 12)  &  637120 / 637120 & 
$6.17 \times 10^{-6}$  &  $6.05 \times 10^{-2}$
   & 35.8s / 10 epochs\\			
			$f_2$ & FCNN2-2 & (240, --, 12) & 637201 / 637201 & 
$3.28 \times 10^{-5}$  &  $1.39 \times 10^{-1}$
   & 29.3s / 10 epochs \\
			\bottomrule
		\end{tabular} 
	}
\end{table}

\begin{figure}
    \centering	
     \begin{subfigure}[b]{0.442\textwidth}
    \centering            \includegraphics[width=0.9015\textwidth]{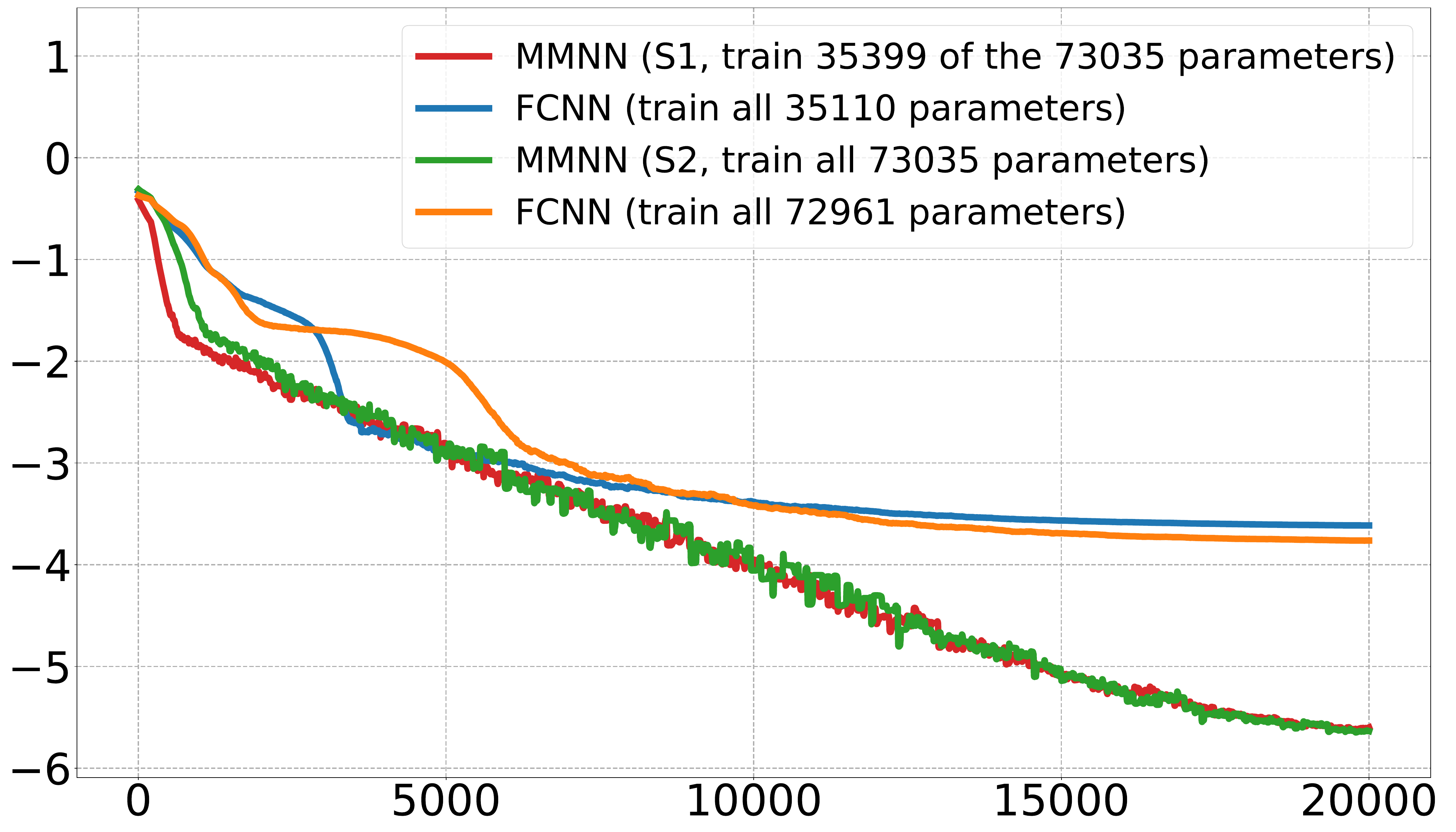}
    \subcaption{$f_1$.}
    \end{subfigure}
    \hspace{14pt}
    \begin{subfigure}[b]{0.442\textwidth}
    \centering            \includegraphics[width=0.9015\textwidth]{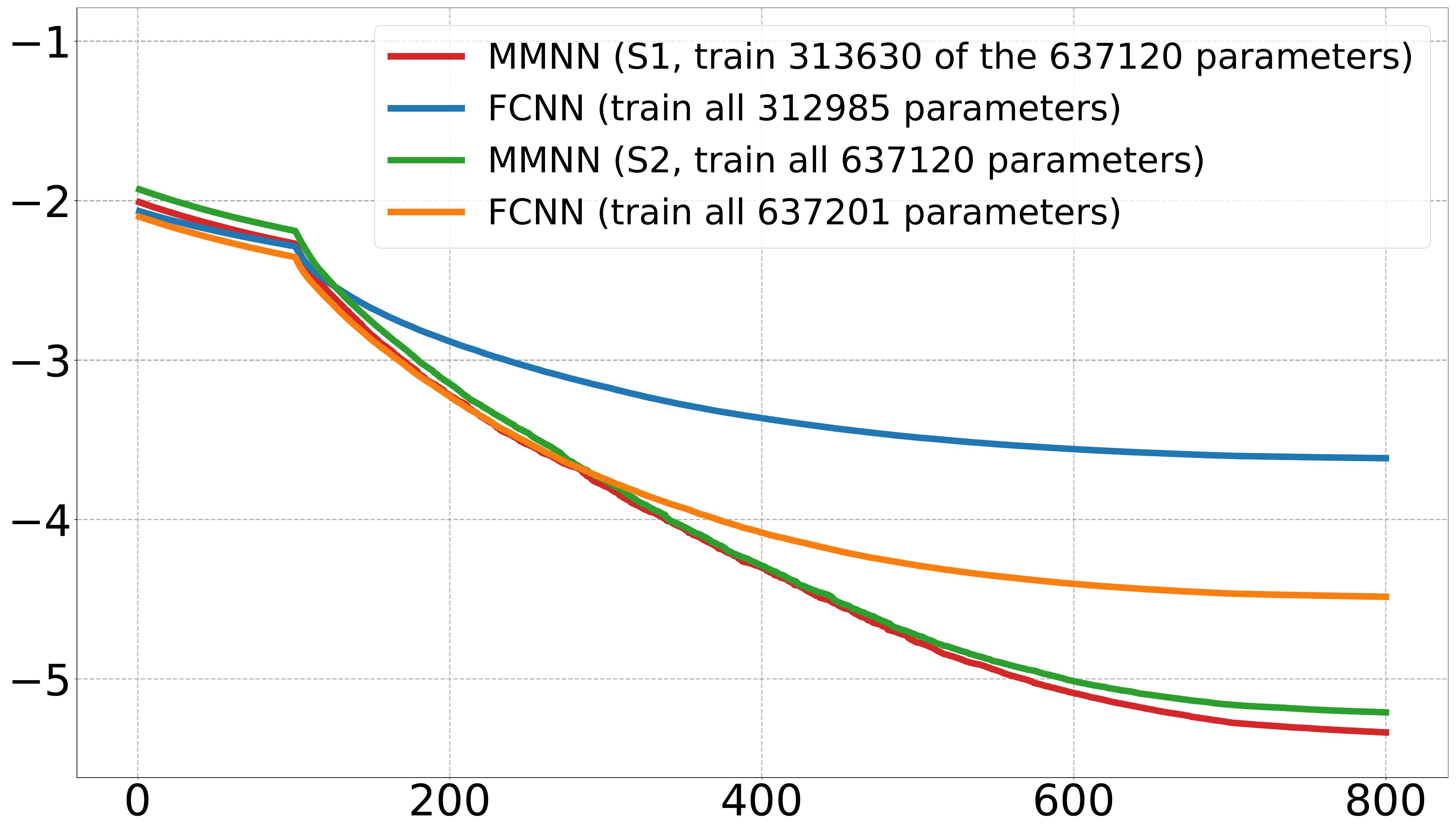}
    \subcaption{$f_2$.}
    \end{subfigure}
    \\
    \begin{subfigure}[b]{0.24311\textwidth}
    \centering            \includegraphics[width=0.985\textwidth]{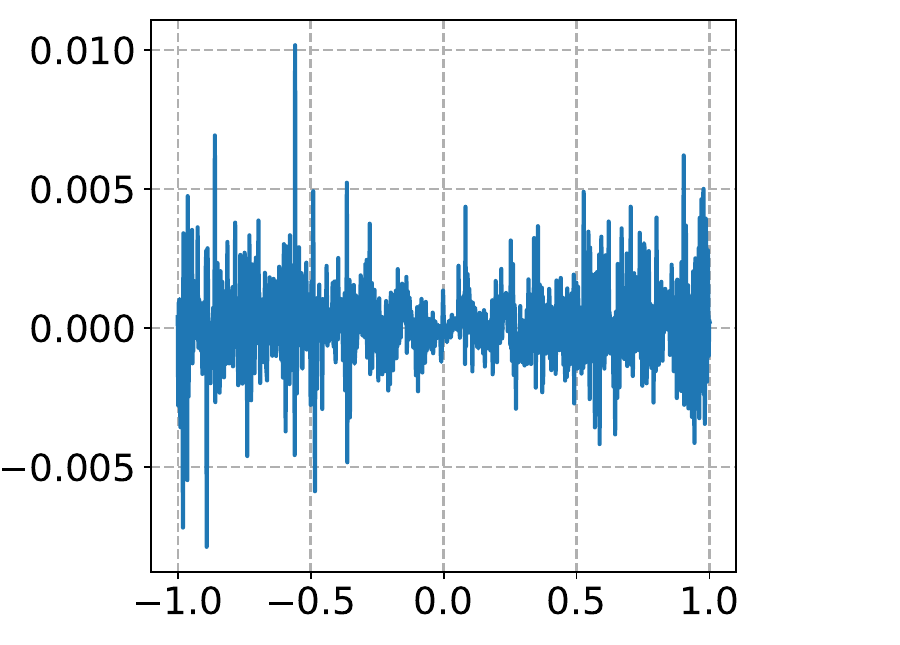}
    \subcaption{MMNN1 (S1).}
    \end{subfigure}
        \hfill
             \begin{subfigure}[b]{0.24311\textwidth}
    \centering            \includegraphics[width=0.985\textwidth]{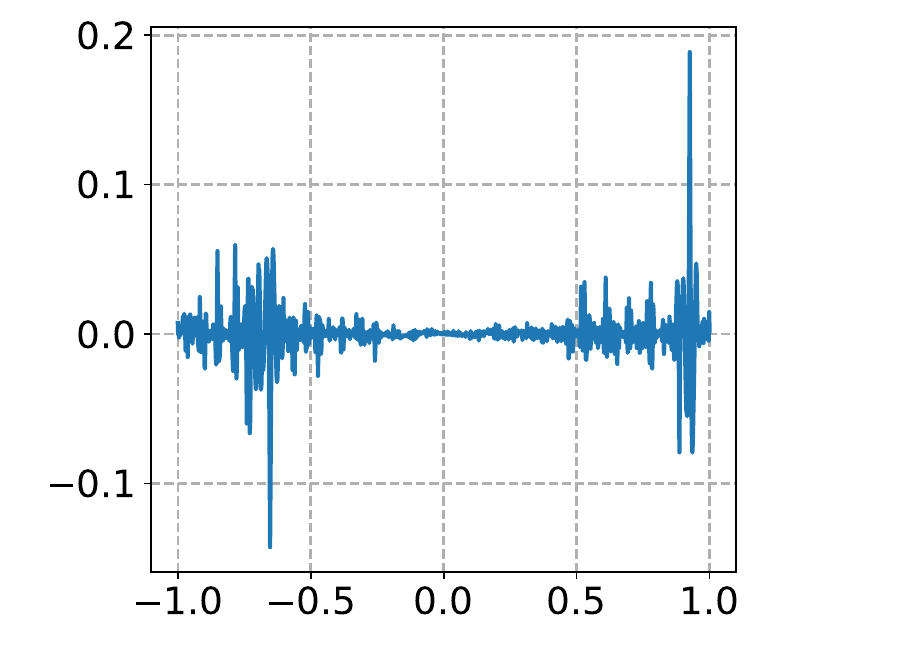}
    \subcaption{FCNN1-1.}
    \end{subfigure}
    \hfill
         \begin{subfigure}[b]{0.24311\textwidth}
    \centering            \includegraphics[width=0.985\textwidth]{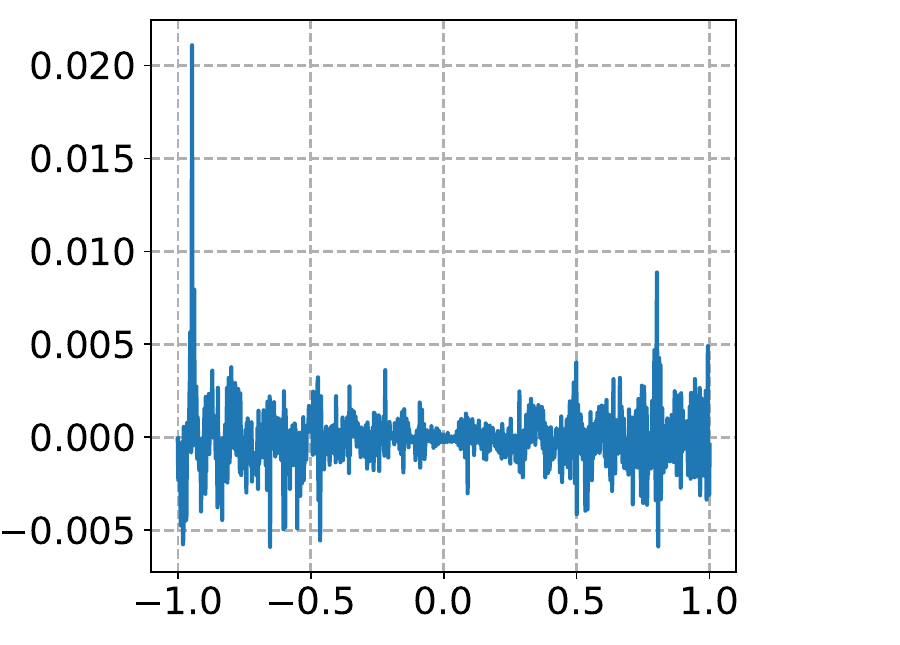}
    \subcaption{MMNN1 (S2).}
    \end{subfigure}
    \hfill
    \begin{subfigure}[b]{0.24311\textwidth}
    \centering            \includegraphics[width=0.985\textwidth]{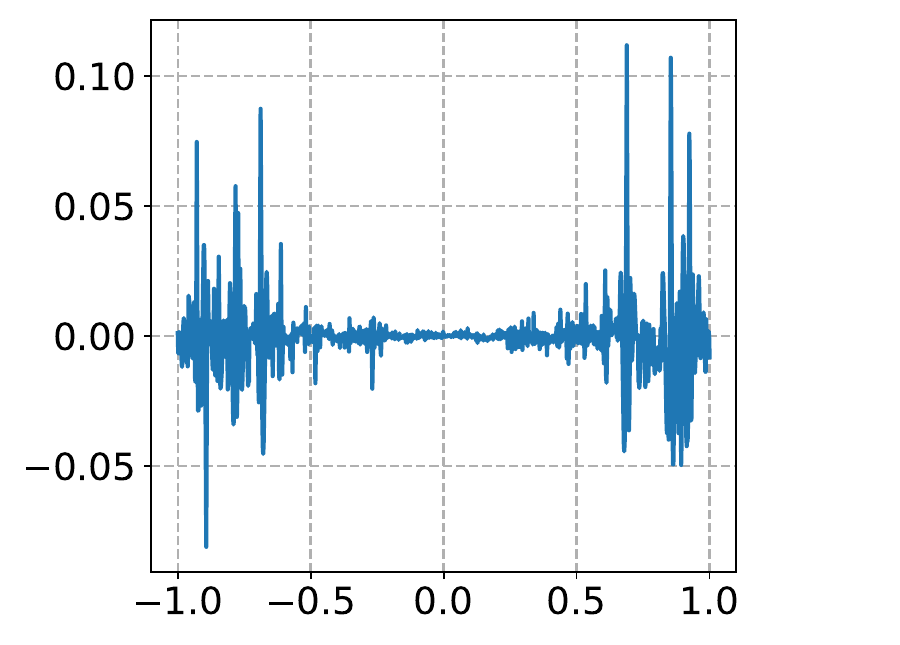}
    \subcaption{FCNN1-2.}
    \end{subfigure}\hfill
 \\[6pt]
    \begin{subfigure}[b]{0.24311\textwidth}
    \centering            \includegraphics[width=0.998055\textwidth]{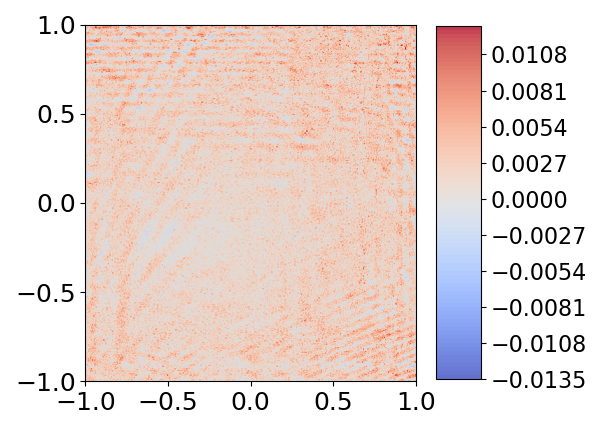}
    \subcaption{MMNN2 (S1).}
    \end{subfigure}
        \hfill
              \begin{subfigure}[b]{0.24311\textwidth}
    \centering            \includegraphics[width=0.985\textwidth]{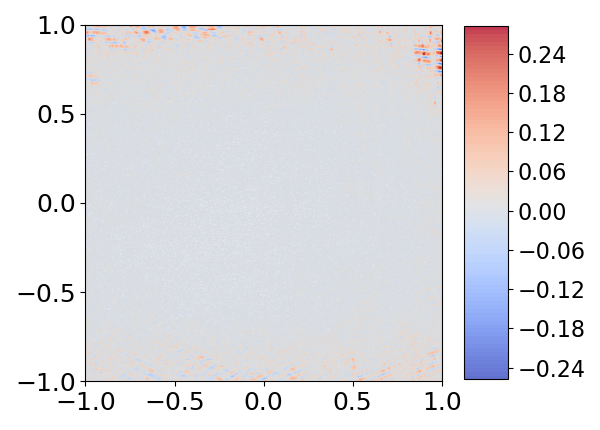}
    \subcaption{FCNN2-1.}
    \end{subfigure}
    \hfill
         \begin{subfigure}[b]{0.24311\textwidth}
    \centering            \includegraphics[width=0.998055\textwidth]{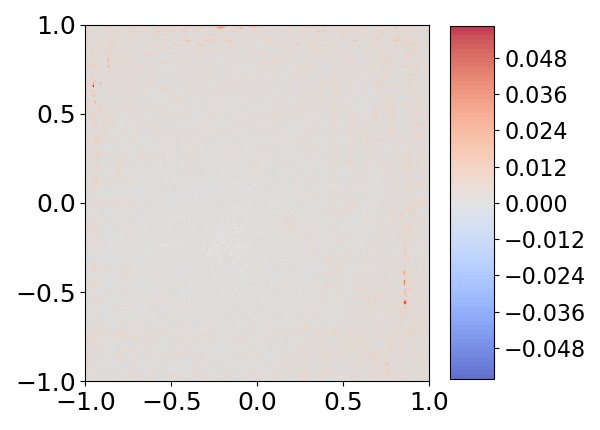}
    \subcaption{MMNN2 (S2).}
    \end{subfigure}
    \hfill
    \begin{subfigure}[b]{0.24311\textwidth}
    \centering            \includegraphics[width=0.998055\textwidth]{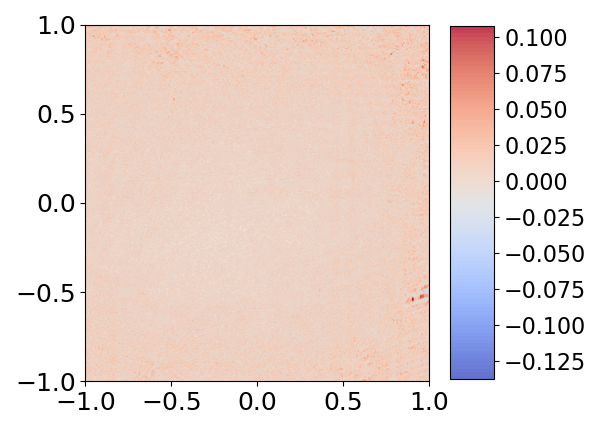}
    \subcaption{FCNN2-2.}
    \end{subfigure}
    \caption{First row: base-10 logarithm of ``test-error-aver'' vs. epoch, where ``test-error-aver'' for epoch \( k \) is calculated by averaging the errors in epochs \(\max\{1, k-100\}\) to \(\min\{k+100, \#\text{epochs}\}\).
    All errors shown on the $y$-axis are in base-10 logarithmic scale.
    Second row: differences between learned networks and $f_1$. 
    Third row: differences between learned networks and $f_2$.}
  \label{fig:vsFCNN:errors:and:NN:fun:diff}
\end{figure}

As illustrated in 
Table~\ref{tab:vsFCNN:error:comparison} and Figure~\ref{fig:vsFCNN:errors:and:NN:fun:diff}, MMNNs outperform FCNNs when both have the same depth and a comparable number of parameters, particularly for relatively oscillatory target functions. 
Moreover, as indicated in Table~\ref{tab:vsFCNN:error:comparison}, the training time for MMNN (S1) is similar to that of FCNN, while MMNN (S2) takes a bit more time.
We remark that the primary advantage of MMNNs lies in capturing high-frequency components.
As we can see from Figure~\ref{fig:vsFCNN:errors:and:NN:fun:diff}, the differences between network approximations and the corresponding target functions show that FCNNs approximate high-frequency parts of the target functions poorly. In contrast, the approximation errors for MMNNs, especially with the S1 learning strategy, are more evenly distributed across the entire domain, indicating their effectiveness in capturing high-frequency components. The Adam optimizer \cite{DBLP:journals/corr/KingmaB14} is applied throughout the training process.



\subsection{Solving partial differential equations}
\label{sec:PDE}

Next, we compare the performance of MMNNs and FCNNs for solving partial differential equations (PDEs).  
We consider a classical example: the two-dimensional Poisson equation with zero Dirichlet boundary conditions:
\[
-\Delta u(x, y) = f(x, y), \quad (x, y) \in (-1,1)^2,\quad u|_{\partial \Omega} = 0,
\]
where the source term is given by
\[
f(x,y) = -113\pi^2\sin(7\pi x)\sin(8\pi y) - 117\pi^2\sin(6\pi x)\sin(9\pi y).
\]
It is easy to verify that the exact solution is
\[
u(x,y) = \sin(7\pi x)\sin(8\pi y) + \sin(6\pi x)\sin(9\pi y).
\]

 We approximate the solution \( u(x, y) \) by a neural network \( u_\bmtheta(x, y) \), where \( \bmtheta \) denotes the network parameters and solve the Poisson equation using the Physics-Informed Neural Network (PINN) \cite{RAISSI2019686} formulation.
The PDE residual is defined as \(\mathcal{R}(x, y) = -\Delta u_\bmtheta(x, y) - f(x, y)\), where \(\Delta u_\bmtheta = \frac{\partial^2 u_\bmtheta}{\partial x^2} + \frac{\partial^2 u_\bmtheta}{\partial y^2}\), and the neural network is trained to minimize the loss function \(\mathcal{L}_{\text{total}} = \lambda_{\text{PDE}}\mathcal{L}_{\text{PDE}} + \lambda_{\text{BC}}\mathcal{L}_{\text{BC}}\), combining the PDE loss \(\mathcal{L}_{\text{PDE}} = \frac{1}{N_r} \sum_{i=1}^{N_r} \mathcal{R}(x_i, y_i)^2\) for collocation points \((x_i, y_i) \in (-1,1)^2\) and the boundary loss \(\mathcal{L}_{\text{BC}} = \frac{1}{N_b} \sum_{j=1}^{N_b} u_\bmtheta(x_j, y_j)^2\) for boundary points \((x_j, y_j) \in \partial \Omega\).

To compare FCNNs and MMNNs in terms of representation accuracy, we avoid soft boundary condition enforcement by including the penalty term \(\lambda_{\text{BC}}\mathcal{L}_{\text{BC}}\) in the loss, which requires careful tuning of \(\lambda_{\text{BC}}\). 
We adopt a hard-constraint formulation that guarantees the boundary condition is satisfied exactly. Specifically, we define the network output as
\[
u_{\bm{\theta}}(x, y) = h_{\bm{\theta}}(x, y) \cdot \cos\left(\frac{\pi x}{2}\right)\cos\left(\frac{\pi y}{2}\right),
\]
where \( h_{\bm{\theta}}(x, y) \) is modeled by either an FCNN or an MMNN. This construction ensures that \( u_{\bm{\theta}}(x, y) = 0 \) on the boundary \( \partial \Omega \) for all values of \( \bm{\theta} \).

We select an MMNN of size (301, 16, 6), another MMNN of size (503, 20, 6), and an FCNN of size (100, --, 6), denoted as MMNN1, MMNN2, and FCNN, respectively, for simplicity. 
Since \ReLU{} is not differentiable, we use the sine function as the activation function.
We sample \(100^2\) data points for \(f(x,y)\) on a uniform grid in \((-1,1)^2\), using a mini-batch size of 2000 and setting \(\lambda_{\text{PDE}} = 0.001\). 
We adopt the Adam optimizer, with the learning rate set to 
\(\lfloor k/200 \rfloor / 800\) for \(k < 16000\), and to 
\(0.001 \times 0.9^{\lfloor (k - 16000)/1600 \rfloor}\) for 
\(k \geq 16000\), where \(k = 1, 2, \dots, 160000\) denotes the training epoch.
We note that for \( k < 16000 \), we use an increasing learning rate to facilitate warm-up (see, e.g., \cite{NEURIPS2020_288cd256,NEURIPS2024_ca98452d}), which enhances training performance.

Our experiments reveal that initial parameters significantly impact training, particularly for FCNNs. To ensure experimental reliability, we repeated the experiments with 16 different seeds. As demonstrated in Figures~\ref{fig:PDE:error:comparison}, \ref{fig:PDE:fun:diff}, and Table~\ref{tab:PDE:error:comparison}, both MMNNs surpass the FCNN in solving PDEs with PINNs, even with comparable depth and total (or training) parameters.
We note that using the same seed may yield different outcomes across various code environments. Our tests consistently show that MMNNs always succeed, whereas FCNNs fail with high probability (increasing FCNN width may improve the likelihood of success).

\begin{figure}
    \centering	
            \begin{subfigure}[c]{0.3253012\textwidth}
    \centering            \includegraphics[width=0.975859598055\textwidth]{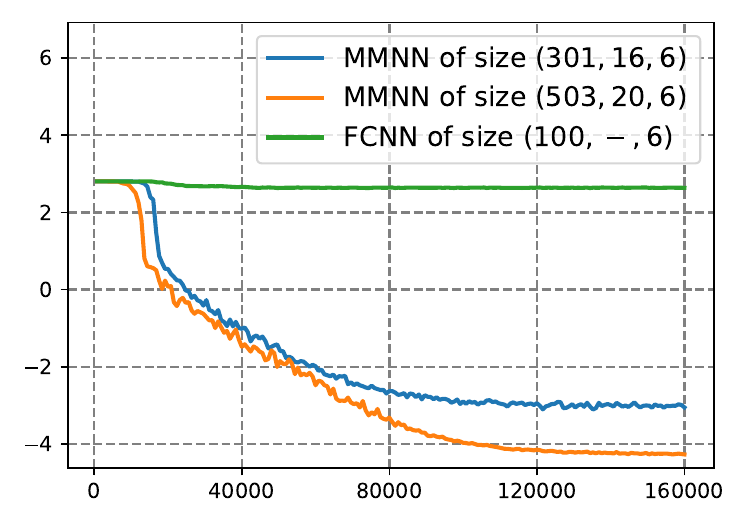}
    \subcaption{Training error.}
    \end{subfigure}
    \hfill
    \begin{subfigure}[c]{0.3253012\textwidth}
    \centering            \includegraphics[width=0.975859598055\textwidth]{figures/PDE/PDE_error_train.pdf}
    \subcaption{Test error (MSE).}
    \end{subfigure}
    \hfill 
 \begin{subfigure}[c]{0.3253012\textwidth}
    \centering            \includegraphics[width=0.975859598055\textwidth]{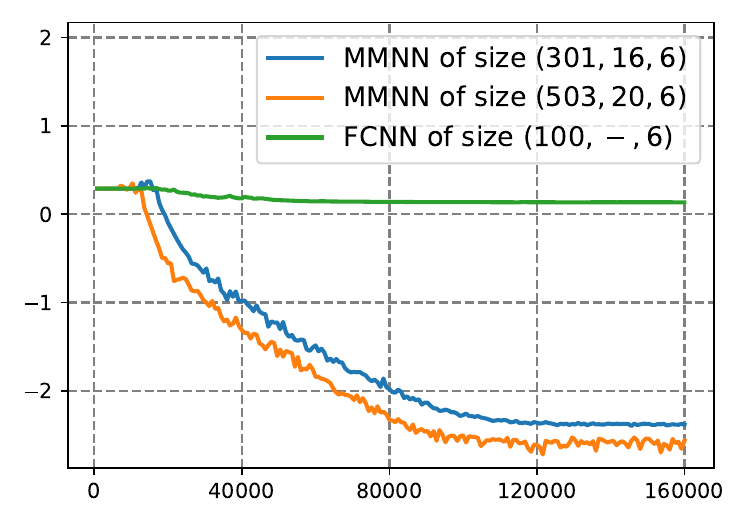}
    \subcaption{Test error (MAX).}
    \end{subfigure}
\caption{Average errors across 16 seeds versus epoch. The training error corresponds to the PDE loss, while the test error (MSE or MAX) quantifies the difference between the learned network and the true solution. All errors shown on the $y$-axis are in base-10 logarithmic scale.}
\label{fig:PDE:error:comparison}
\end{figure}

\begin{figure}
    \centering	
    \begin{subfigure}[b]{0.24311\textwidth}
    \centering            \includegraphics[width=0.998055\textwidth]{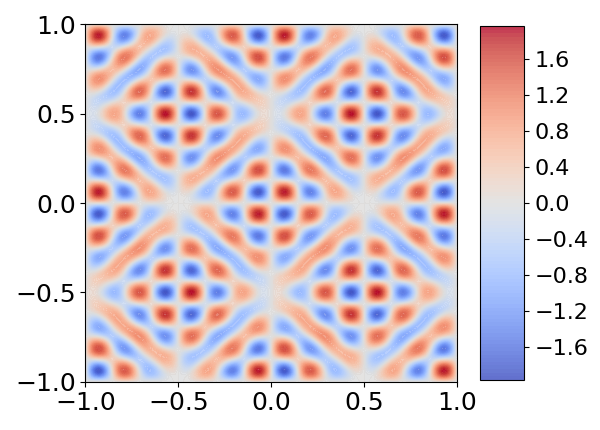}
    \subcaption{Truth.}
    \end{subfigure}
        \hfill
              \begin{subfigure}[b]{0.24311\textwidth}
    \centering            \includegraphics[width=0.985\textwidth]{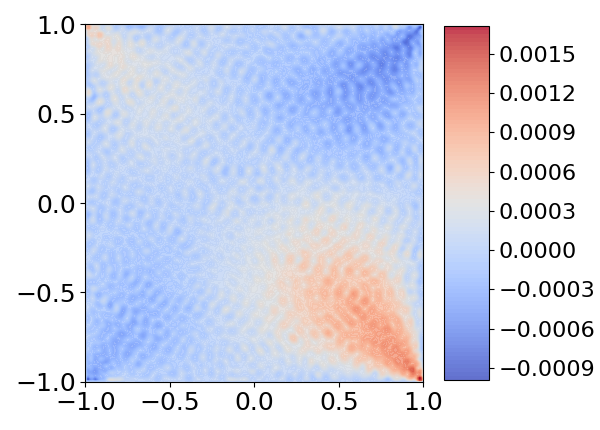}
    \subcaption{MMNN1.}
    \end{subfigure}
    \hfill
         \begin{subfigure}[b]{0.24311\textwidth}
    \centering            \includegraphics[width=0.998055\textwidth]{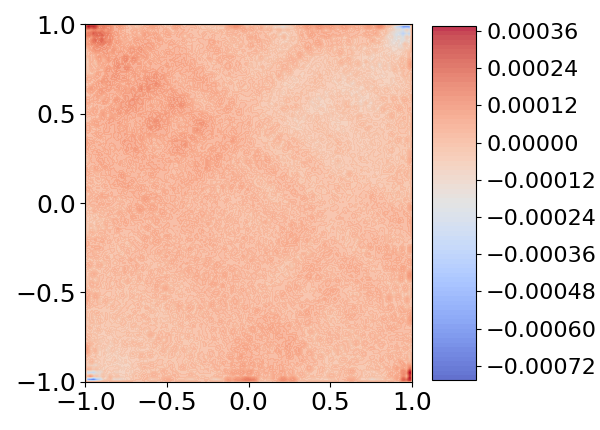}
    \subcaption{MMNN2.}
    \end{subfigure}
    \hfill
    \begin{subfigure}[b]{0.24311\textwidth}
    \centering            \includegraphics[width=0.998055\textwidth]{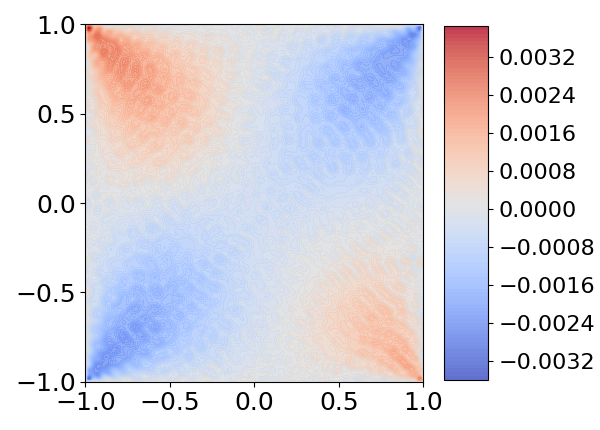}
    \subcaption{FCNN.}
    \end{subfigure}
\caption{Comparison of three networks: (a) true solution; (b, c, d) differences between the true solution and predictions from learned networks. For each network, we select the best trained model from 16 seeds.}
  \label{fig:PDE:fun:diff}
\end{figure}

\begin{table}[h]
	\centering  
 \setlength{\tabcolsep}{0.868em} 
 \renewcommand{\arraystretch}{1.15}
\caption{Comparison of test errors. }
	\label{tab:PDE:error:comparison}
	\resizebox{0.95\textwidth}{!}{ 
		\begin{tabular}{ccccccccc} 
			\toprule
            {\#parameters (trained / all)}
            &\multicolumn{2}{c}{24462 / \textbf{50950}} &\multicolumn{2}{c}{\textbf{50904} / 105228} &\multicolumn{2}{c}{\textbf{50901} / \textbf{50901}}   \\
            \cmidrule(lr){2-3}
            \cmidrule(lr){4-5}
            \cmidrule(lr){6-7}
        \rowcolor{mygray}   
        &\multicolumn{2}{c}{MMNN1 of size (301, 16, 6)} &\multicolumn{2}{c}{MMNN2 of size (503, 20, 6)} &\multicolumn{2}{c}{FCNN of size (100,  --, 6)}
        
        \\
            \cmidrule(lr){2-3}
            \cmidrule(lr){4-5}
            \cmidrule(lr){6-7}
        {seed for randomness} &         MSE 
    &  MAX &    MSE 
    &  MAX &    MSE 
    &  MAX 
    \\
  \midrule

$0$   &  $ 2.60 \times 10^{-6} $  &  $ 6.22 \times 10^{-3} $  &  $ 1.77 \times 10^{-8} $  &  $ 3.38 \times 10^{-3} $  &  $ 4.95 \times 10^{-1} $  &  $ 1.94 \times 10^{0} $ 
 \\ 

\rowcolor{mygray}$1$   &  $ 2.18 \times 10^{-7} $  &  $ 2.08 \times 10^{-3} $  &  $ 9.32 \times 10^{-9} $  &  $ 6.76 \times 10^{-4} $  &  $ \bm{7.95 \times 10^{-7}} $  &  $ 3.35 \times 10^{-3} $ 
 \\ 

 $2$   &  $ 1.18 \times 10^{-6} $  &  $ 4.18 \times 10^{-3} $  &  $ 2.38 \times 10^{-8} $  &  $ 1.55 \times 10^{-3} $  &  $ 5.01 \times 10^{-1} $  &  $ 1.94 \times 10^{0} $ 
 \\ 

\rowcolor{mygray}$3$   &  $ 3.19 \times 10^{-6} $  &  $ 5.57 \times 10^{-3} $  &  $ 6.68 \times 10^{-9} $  &  $ 1.19 \times 10^{-3} $  &  $ 5.03 \times 10^{-1} $  &  $ 1.94 \times 10^{0} $ 
 \\ 

 $4$   &  $ 5.34 \times 10^{-7} $  &  $ 3.12 \times 10^{-3} $  &  $ 2.35 \times 10^{-8} $  &  $ 2.89 \times 10^{-3} $  &  $ 4.99 \times 10^{-1} $  &  $ 1.94 \times 10^{0} $ 
 \\ 

\rowcolor{mygray}$5$   &  $ 1.76 \times 10^{-6} $  &  $ 5.15 \times 10^{-3} $  &  $ 1.78 \times 10^{-8} $  &  $ 2.51 \times 10^{-3} $  &  $ 4.96 \times 10^{-1} $  &  $ 1.94 \times 10^{0} $ 
 \\ 

 $6$   &  $ 4.92 \times 10^{-7} $  &  $ 2.61 \times 10^{-3} $  &  $ 9.68 \times 10^{-9} $  &  $ 4.65 \times 10^{-3} $  &  $ 5.04 \times 10^{-1} $  &  $ 1.94 \times 10^{0} $ 
 \\ 

\rowcolor{mygray}$7$   &  $ 1.97 \times 10^{-6} $  &  $ 5.34 \times 10^{-3} $  &  $ 3.87 \times 10^{-8} $  &  $ 3.37 \times 10^{-3} $  &  $ 5.01 \times 10^{-1} $  &  $ 1.94 \times 10^{0} $ 
 \\ 

 $8$   &  $ 4.98 \times 10^{-7} $  &  $ 2.44 \times 10^{-3} $  &  $ 4.04 \times 10^{-8} $  &  $ 9.83 \times 10^{-4} $  &  $ 8.20 \times 10^{-7} $  &  $ \bm{3.22 \times 10^{-3}} $ 
 \\ 

\rowcolor{mygray}$9$   &  $ 2.57 \times 10^{-6} $  &  $ 4.10 \times 10^{-3} $  &  $ \bm{1.69 \times 10^{-9}} $  &  $ \bm{6.66 \times 10^{-4}} $  &  $ 5.00 \times 10^{-1} $  &  $ 1.94 \times 10^{0} $ 
 \\ 

 $10$   &  $ 2.47 \times 10^{-6} $  &  $ 6.15 \times 10^{-3} $  &  $ 2.72 \times 10^{-8} $  &  $ 3.88 \times 10^{-3} $  &  $ 5.02 \times 10^{-1} $  &  $ 1.94 \times 10^{0} $ 
 \\ 

\rowcolor{mygray}$11$   &  $ 1.03 \times 10^{-6} $  &  $ 3.86 \times 10^{-3} $  &  $ 2.02 \times 10^{-8} $  &  $ 1.31 \times 10^{-3} $  &  $ 5.03 \times 10^{-1} $  &  $ 1.94 \times 10^{0} $ 
 \\ 

 $12$   &  $ 1.61 \times 10^{-6} $  &  $ 5.13 \times 10^{-3} $  &  $ 4.38 \times 10^{-9} $  &  $ 1.16 \times 10^{-3} $  &  $ 1.60 \times 10^{-5} $  &  $ 1.12 \times 10^{-2} $ 
 \\ 

\rowcolor{mygray}$13$   &  $ \bm{9.35 \times 10^{-8}} $  &  $ \bm{1.52 \times 10^{-3}} $  &  $ 3.85 \times 10^{-8} $  &  $ 1.07 \times 10^{-2} $  &  $ 2.86 \times 10^{-2} $  &  $ 4.39 \times 10^{-1} $ 
 \\ 

 $14$   &  $ 7.42 \times 10^{-7} $  &  $ 3.41 \times 10^{-3} $  &  $ 1.51 \times 10^{-8} $  &  $ 1.77 \times 10^{-3} $  &  $ 5.01 \times 10^{-1} $  &  $ 1.95 \times 10^{0} $ 
 \\ 

\rowcolor{mygray}$15$   &  $ 6.58 \times 10^{-6} $  &  $ 6.25 \times 10^{-3} $  &  $ 2.74 \times 10^{-8} $  &  $ 1.23 \times 10^{-3} $  &  $ 6.85 \times 10^{-6} $  &  $ 7.97 \times 10^{-3} $ 
 \\

 \bottomrule
		\end{tabular} 
	}
\end{table}

\section{Multi-component and multi-layer decomposition}
\label{sec:decomposition}

It has been shown in \cite{ZZZZ-23} that a one-hidden-layer neural network acts as a low-pass filter and cannot effectively represent or learn high-frequency features. Using mathematical construction, we demonstrate that MMNNs, which are composed of one-hidden-layer neural networks, can overcome this difficulty by decomposing complexity through their components and/or depth.
We emphasize that the decomposition is highly non-unique. Our construction is ``man-made" which can be different from the one by computer through an optimization (learning) process.
Our discussion begins with one-dimensional construction in Section~\ref{sec:decomposition:1D} and later extends to higher dimensions in Section~\ref{sec:decomposition:highD}.



\subsection{One dimensional construction}
\label{sec:decomposition:1D}

We begin with a two-component decomposition in the one-dimensional as both an illustration and an example in Section~\ref{sec:1D:two:component:decom}. Later in Section~\ref{sec:1D:general:multi:component:decom}, we introduce the general multi-component decomposition. Finally in Section~\ref{sec:1D:decom:examples}, we use concrete examples for demonstration.

\subsubsection{Two-component decomposition}
\label{sec:1D:two:component:decom}

We demonstrate a simple ``divide and conquer'' strategy for an example function 
$f(x)=\cos(2n\pi x)$, a high frequency Fourier mode when $n$ is large. 
Define
\begin{equation*}
     f_2:(u,v)\in [-1,1]^2 \mapsto 
    \cos\big(n\pi  (u+1)\big)+ \cos\big(n\pi  (v-1)\big)\in \R.
\end{equation*}
and 
$\bmf_1=\begin{bsmallmatrix}
    f_{1,1}\\  f_{1,2}\end{bsmallmatrix}:[-1,1]\mapsto[-1,1]^2$, where the components $f_{1,1}$ and $f_{1,2}$ are given by
\begin{equation}\label{eq:component:f11}
	f_{1,1}(x)=\ReLU(2x)-1=\begin{cases}
		-1 & \tn{for}\  x\in [-1,0),\\
		2x-1 & \tn{for}\  x\in [0,1],
	\end{cases}
\end{equation}
and
\begin{equation}\label{eq:component:f12}
	f_{1,2}(x)=-\ReLU(-2x)+1=\begin{cases}
		2x+1 & \tn{for}\  x\in [-1,0),\\
		1  & \tn{for}\  x\in [0,1].
	\end{cases}
\end{equation}
Then for any $x\in [-1,1]$ we have 
\begin{equation*}
	\begin{split}
	    f(x)&=\cos\Big(n\pi  \cdot\ReLU(2x)\Big)+\cos\Big(-n\pi  \cdot\ReLU(-2x)\Big)
     \\& =\cos\Big(n\pi  \big(f_{1,1}(x)+1\big)\Big)+\cos\Big(n\pi  \big(f_{1,2}(x)-1\big)\Big)
 =
 f_2\circ \bmf_1(x).
	\end{split}
\end{equation*}
Through this decomposition and piecewise linear transformation, which can be approximated easily by a single layer of {\ReLU} network, one only needs to approximate a function that is smoother than the original $f$: $\bm{f}_1$ is simplified, while $f_2$ is reduced to half of the frequency of the original target function $f$. 

We observe that this decomposition approach is universally applicable for any function $f:[-1,1]\mapsto\mathbb{R}$. Specifically, the decomposition is defined as
\begin{equation*}
    f_2:(u,v)\in[-1,1]^2\mapsto 
    f\big(\tfrac{u+1}{2}\big)+ f\big(\tfrac{v-1}{2}\big)-f(0)\in\R.
\end{equation*}
and 
$\bmf_1=\begin{bsmallmatrix}
    f_{1,1}\\  f_{1,2}\end{bsmallmatrix}:[-1,1]\mapsto[-1,1]^2$,
where $f_{1,1}$ and $f_{1,2}$ are given in \eqref{eq:component:f11} and \eqref{eq:component:f12}.
Hence, for any $x\in [-1,1]$, we achieve the following reconstruction of $f(x)$:
\begin{equation*}
	\begin{split}
	    f(x)&=f\Big(\tfrac{\ReLU(2x)}{2}\Big)+ f\Big(\tfrac{-\ReLU(-2x)}{2}\Big)-f(0)
     \\& =f\Big(\tfrac{f_{1,1}(x)+1}{2}\Big)+ f\Big(\tfrac{f_{1,2}(x)-1}{2}\Big)-f(0)
 =
 f_2\circ \bmf_1(x)
	\end{split}
\end{equation*}
demonstrating a structured decomposition that allows the function to be expressed through the composition of a smoother function with a piecewise (component-wise) transformation and rescaling.

\subsubsection{General multi-component decomposition}
\label{sec:1D:general:multi:component:decom}

Now we extend to a general multi-component adaptive decomposition, a ``divide and conquer" strategy, that can distribute the complexity of a target function evenly to multiple components.

Given a sequence $x_0<x_1<\cdots<x_n$ where the target function is defined on the interval $[x_0,x_n]$, we will demonstrate how our new architecture allows us to partition the complexities of the function $f$ into smaller intervals $[x_{i-1},x_i]$. By rescaling each subinterval, one only needs to deal with a much smoother function in each interval. This approach enables us to effectively approximate the target function over the entire interval $[x_0,x_n]$. 

Let $\calL_i:[a_i,b_i]\to [x_{i-1},x_i]$ be the linear map with 
\begin{equation}\label{eq:1D:Li}
    \calL_i(a_i)=x_{i-1} \quad \tn{and}\quad \calL_i(b_i)=x_i.
\end{equation}
Define
\begin{equation}\label{eq:1D:fi}
    f_i=f\circ \calL_i: [a_i,b_i]\to \R. 
\end{equation}

To decompose the target function into smoother pieces, we define a piecewise linear transformation 
$\psi_i$ using a linear combination of two \texttt{ReLU} functions (or a simple single layer network),
\begin{equation}\label{eq:1D:psii}
    \psi_i(x)=s_i\cdot \ReLU\left( {x-x_{i-1}} \right)
    -s_i\cdot \ReLU\left( {x-x_{i}} \right)+a_i.
\end{equation}
Here $s_i=\frac{b_i-a_i}{x_i-x_{i-1}}$ is the ``slope'' of $\calL_i^{-1}$, which is a local rescaling. For example, $f_i$ becomes a smoother function than $f$ after stretching $[x_{i-1},x_i]$ to a larger domain $[a_i,b_i]$. See an illustration of $\psi_i(x)$ in Figure~\ref{fig:psi:i}. 

\begin{figure}
    \centering
    \includegraphics[width=0.615\textwidth]{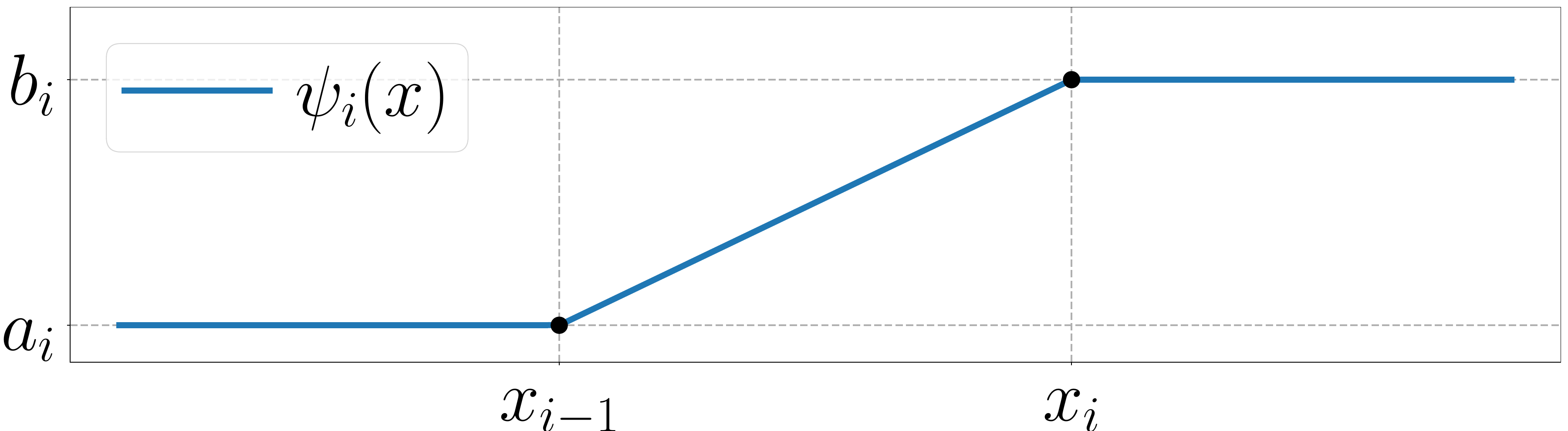}
    \caption{An illustration of $\psi_i(x)$.}
    \label{fig:psi:i}
\end{figure}


\begin{figure}
    \centering	
    \begin{subfigure}[c]{0.998\textwidth}
    \centering            \includegraphics[width=0.827998055\textwidth]{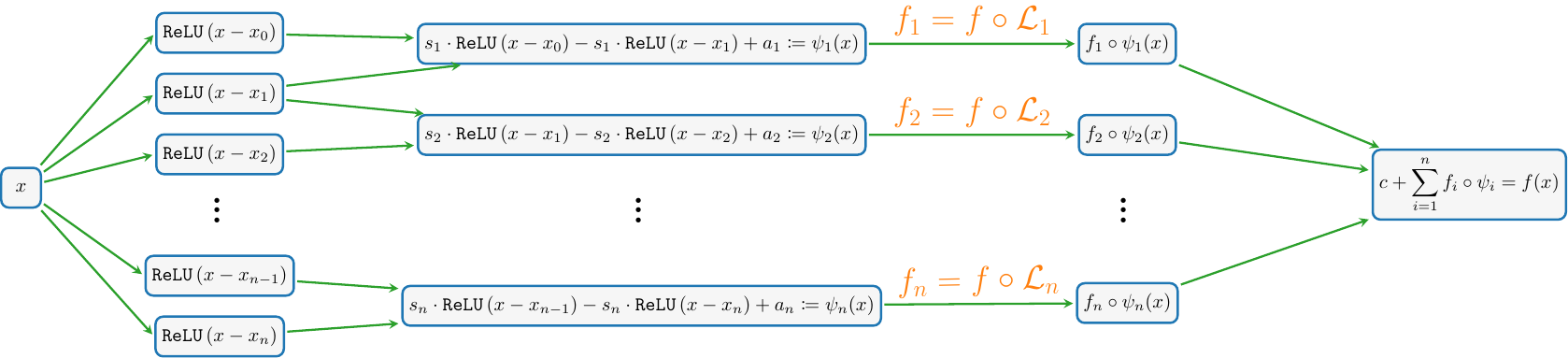}
    \subcaption{Decomposition of target function $f=c+ \sum_{i=1}^n f_i\circ \psi_i$:
    oscillatory $f$ to smooth $f_i$'s.}
    \end{subfigure}\\
    \vspace*{12pt}
        \begin{subfigure}[c]{0.998\textwidth}
    \centering            \includegraphics[width=0.827998055\textwidth]{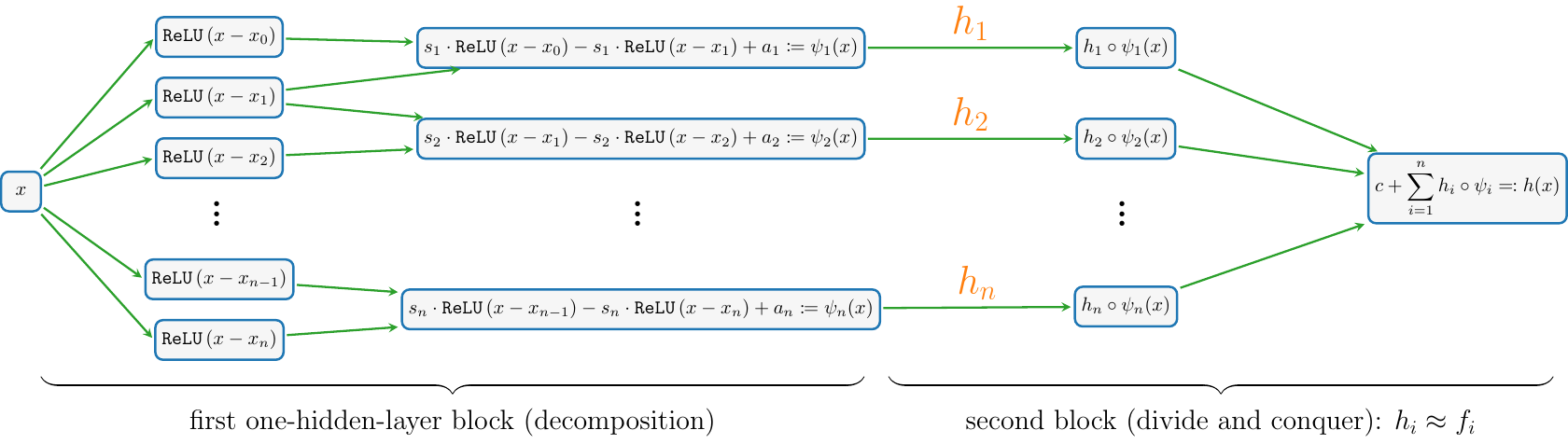}
    \subcaption{Neural network architecture of $h=c+ \sum_{i=1}^n h_i\circ \psi_i$ by using $h_i\approx f_i$.}
    \end{subfigure}
    \caption{
Visual representations of the decompositions of $f$ and $h$ are provided with $c=\sum_{i=0}^{n-1}f(x_i)$ being a constant and $s_i$ being the slope. Here, the function $f$ is dissected into several simpler functions, labeled as $f_i$. Each $f_i$ represents a simplified and more manageable segment of $f$, allowing for the straightforward application of sub-network $h_i$ to closely approximate $f_i$, even with the use of shallow networks.
    }
    	\label{fig:1D:decomposition:approx}
\end{figure}

\begin{theorem}
\label{thm:decomposition:1D}
Given \(x_0<x_1<\cdots<x_n\),
    suppose $\calL_{i}$ and $\psi_i$ are given in Equations~\eqref{eq:1D:Li} and \eqref{eq:1D:psii}, respectively. Then the target function $f:[x_0,x_n]\to\R$ has the following (smoother) decomposition ($f_i$) with a piecewise linear transformation ($\psi_i$),
    
\begin{equation*}
    \begin{split}
        f(x)=\sum_{i=1}^n f_i\circ \psi_i(x)
    -\underbrace{ \sum_{i=1}^{n-1}f(x_i)}_{\tn{constant}}\quad \tn{for any $x\in [x_{0},x_n]$,}
    \end{split}
\end{equation*}
where $f_{i}$ is given in Equation~\eqref{eq:1D:fi}.
\end{theorem}




\begin{proof}
By definition of $\psi_i$ in Equation~\eqref{eq:1D:psii}, 
it is easy to check
\begin{equation*}
    \psi_i(x)=\begin{cases}
        b_i &\tn{if  } x>x_{i},\\
        \calL_i^{-1}(x) &\tn{if  }x\in [x_{i-1},x_i],\\
        a_i &\tn{if  }x<x_{i-1},
    \end{cases}
    \quad \Longrightarrow \quad 
        \psi_i(x)=\begin{cases}
        b_i &\tn{if  }i\le j-1,\\
        \calL_j^{-1}(x) &\tn{if  }i= j ,\\
        a_i &\tn{if  }i\ge j+1,
    \end{cases}
\end{equation*}
for a fixed $j\in \{1,2,\cdots,n\}$ and any $x\in [x_{j-1}, x_j]$.
It follows that
\begin{equation*}
    \begin{split}
        \sum_{i=1}^n f_i\circ \psi_i(x)
        &=\sum_{i=1}^n f\circ \calL_i\circ \psi_i(x)
    \\&=\sum_{i=1}^{j-1} f\circ \calL_i\circ \psi_i(x)
    +f\circ \calL_j\circ \psi_j(x)
    +\sum_{i=j+1}^n f\circ \calL_i\circ \psi_i(x)
    \\& = \sum_{i=1}^{j-1} f\circ \calL_i(b_i)
    +f\circ \calL_j\circ\calL_j^{-1}(x) 
    +\sum_{i=j+1}^n f\circ \calL_i(a_i)
    \\& = \sum_{i=1}^{j-1} f (x_i)
    +f(x)
    +\sum_{i=j+1}^n f (x_{i-1})
    =f(x)+\underbrace{\sum_{i=1}^{n-1}f(x_i)}_{\tn{constant}}.
    \end{split}
\end{equation*}
It follows that
\begin{equation*}
    \begin{split}
        f(x)=\sum_{i=1}^n f_i\circ \psi_i(x)
    -\underbrace{\sum_{i=1}^{n-1}f(x_i)}_{\tn{constant}}\quad \tn{for any $x\in [x_{j-1},x_j]$.}
    \end{split}
\end{equation*}
Since $j$ is arbitrary, the above equation holds for all
$x=\cup_{j=1}^n[x_{j-1}, x_j]= [x_0,x_n]$. 
\end{proof}

For each smoother $f_i$, one can use a shallow network component $\phi_i$, a linear combination of random basis functions, to approximate $f_i$ well on $[a_i,b_i]$. Then
\begin{equation*}
    f(x)=\sum_{i=1}^n f_i\circ \psi_i(x)-\underbrace{\sum_{i=1}^{n-1}f(x_i)}_{\tn{constant}}
    \approx \sum_{i=1}^n \phi_i\circ \psi_i(x)-\underbrace{\sum_{i=1}^{n-1}f(x_i)}_{\tn{constant}}
   \eqqcolon h(x),
\end{equation*}
$h(x)$ is a one-hidden-layer neural network approximation of the target function $f(x)$ that can approximate a complex function better than a single layer.
See Figure~\ref{fig:1D:decomposition:approx} for an illustration. In practice, one can choose repeated decomposition using a multi-component and multi-layer network structure which is the motivation for MMNN. It is well-known that neural networks can approximate smooth functions well. For
localized rapid change/oscillation, our construction shows that a small network in terms of the width as well as the number of components and layers can achieve adaptive decomposition and deal with it rather easily. Hence, MMNN is effective in approximating a function with localized fine features. This is an important advantage in dealing with low-dimensional structures embedded in high dimensions. The most difficult situation is approximating global highly oscillatory functions, especially with diverse frequency modes, for which wider networks with more components and layers are needed to deal with both the complexity and curse of dimensions.

\subsubsection{Examples}
\label{sec:1D:decom:examples}


Here we use two examples to demonstrate the complexity decomposition strategy presented in the previous section. We start with the Runge function $f(x) = \frac{1}{25x^2+1}$ and modify it to $f(x) = \frac{1}{1000x^2+1}$, which has a localized rapid change near $0$. As an example, we use four components $n=4$, choose points $x_0, x_1, x_2, x_3, x_4$ at $-1, -0.2, 0, 0.2, 1$, and let $a_i = -1$ and $b_i = 1$ for all 
 $i$. In practice, each component is approximated by a single-layer network - a linear combination of basis functions, and trained by an optimization method, e.g., Adam. Our examples here are just a proof of concept for the decomposition of a target function into smoother components using MMNN structure in the form
\[
f(x) = \sum_{i=1}^4 f_i\circ \psi_i(x) - \underbrace{\sum_{i=1}^{3} f(x_i)}_{\text{constant}},
\]
where $f_i$ and $\psi_i$ (piecewise tranformation/rescaling) are defined as in~\eqref{eq:1D:fi} and \eqref{eq:1D:psii}, respectively. These components are illustrated in Figure~\ref{fig:Runge:F:decom}.
Each component is relatively smooth, making it easier for approximation and learning through shallow networks. This approach essentially utilizes a divide-and-conquer principle.

\begin{figure}
    \centering	
    \begin{subfigure}[c]{0.1948\textwidth}
    \centering            \includegraphics[width=0.998055\textwidth]{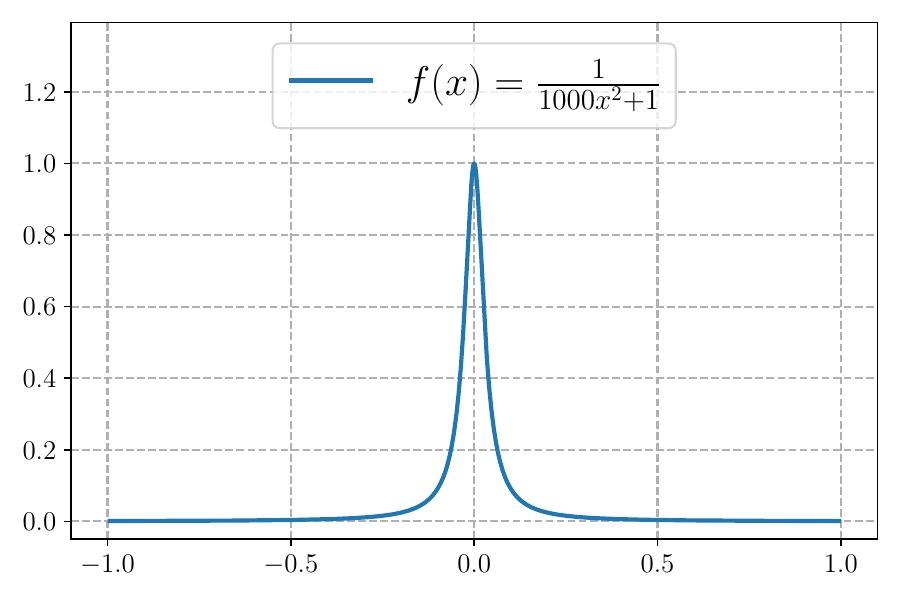}
    \end{subfigure}\hfill
    \begin{subfigure}[c]{0.1948\textwidth}
    \centering            \includegraphics[width=0.998055\textwidth]{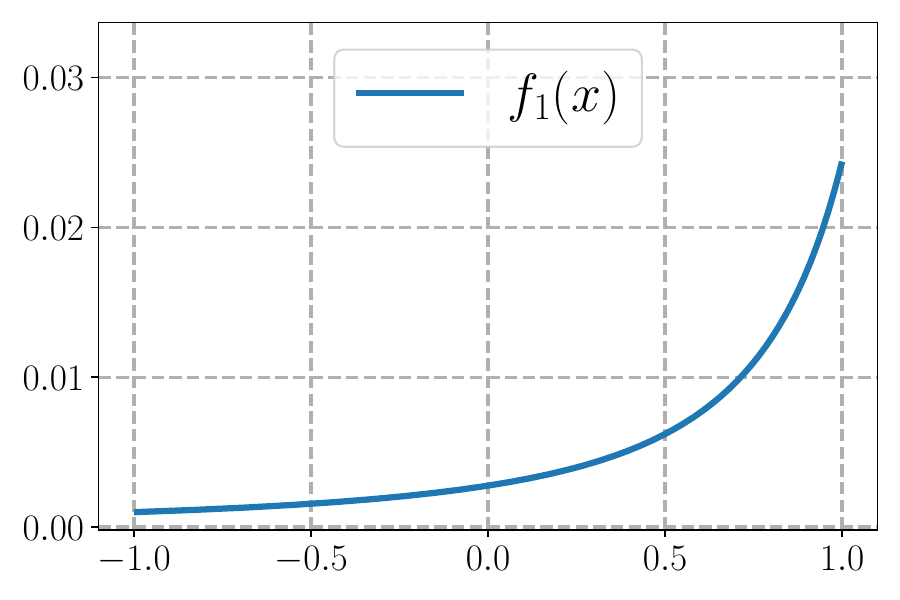}
    \end{subfigure}\hfill
        \begin{subfigure}[c]{0.1948\textwidth}
    \centering            \includegraphics[width=0.998055\textwidth]{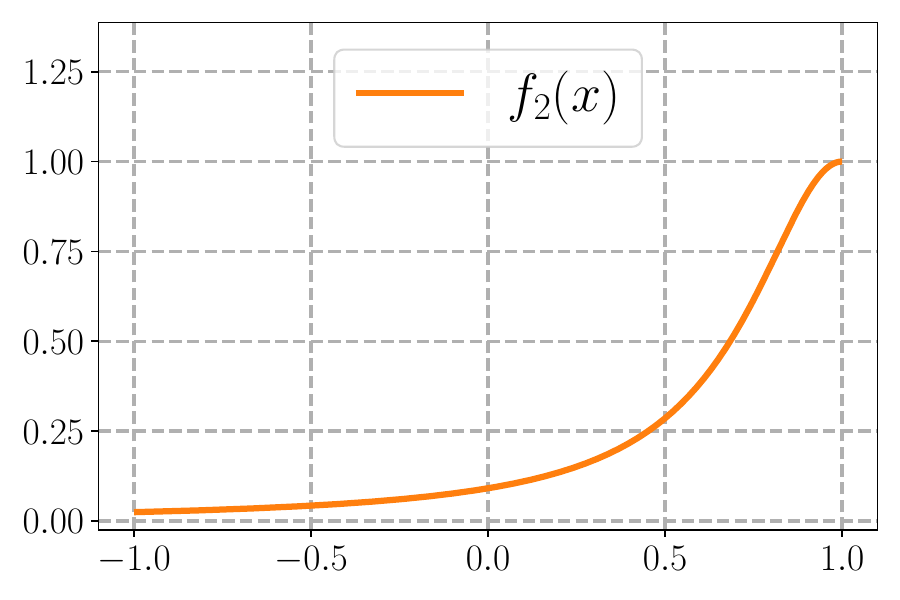}
    \end{subfigure}\hfill
    \begin{subfigure}[c]{0.1948\textwidth}
    \centering            \includegraphics[width=0.998055\textwidth]{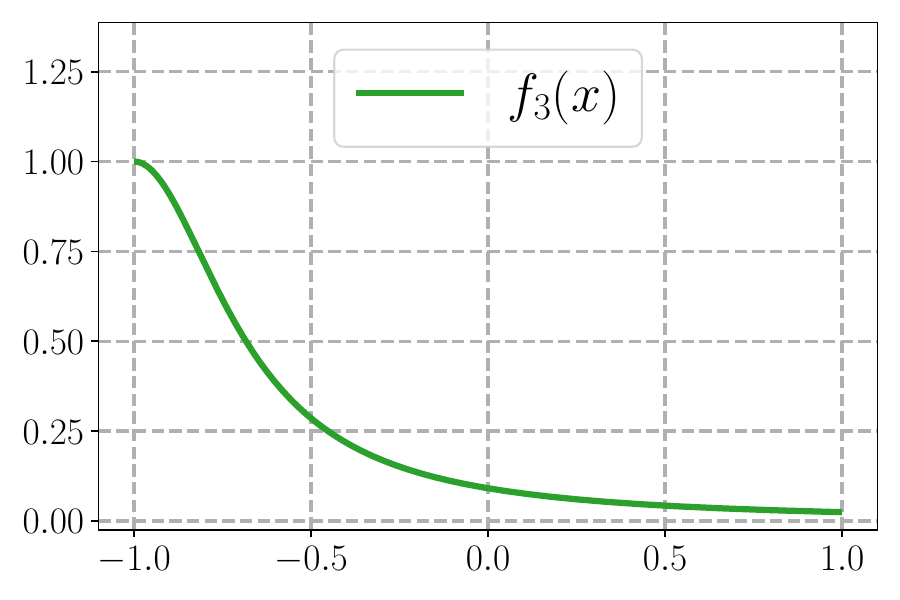}
    \end{subfigure}\hfill
        \begin{subfigure}[c]{0.1948\textwidth}
    \centering            \includegraphics[width=0.998055\textwidth]{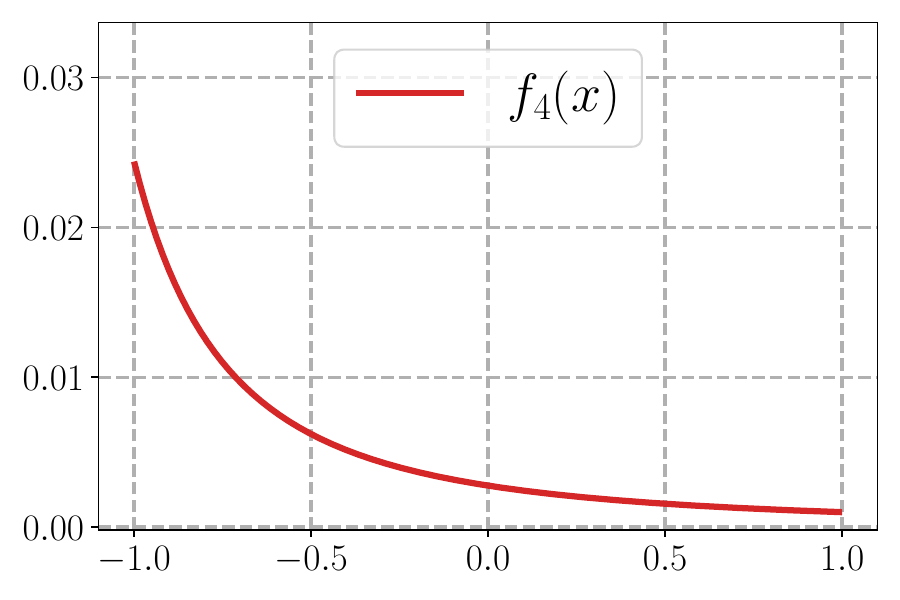}
    \end{subfigure}\\
    \begin{subfigure}[c]{0.1948\textwidth}
    \centering            \includegraphics[width=0.998055\textwidth]{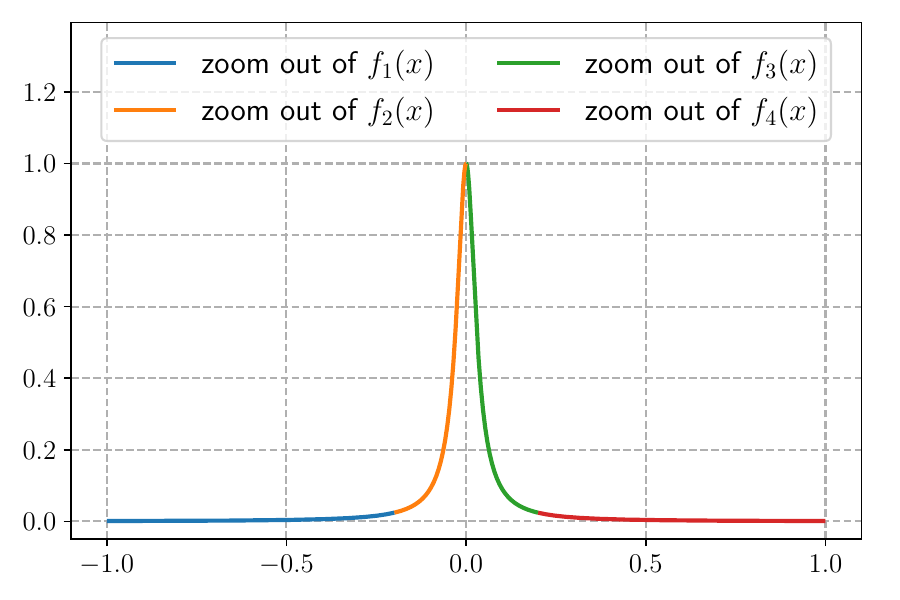}
    \end{subfigure}
        \begin{subfigure}[c]{0.1948\textwidth}
    \centering            \includegraphics[width=0.998055\textwidth]{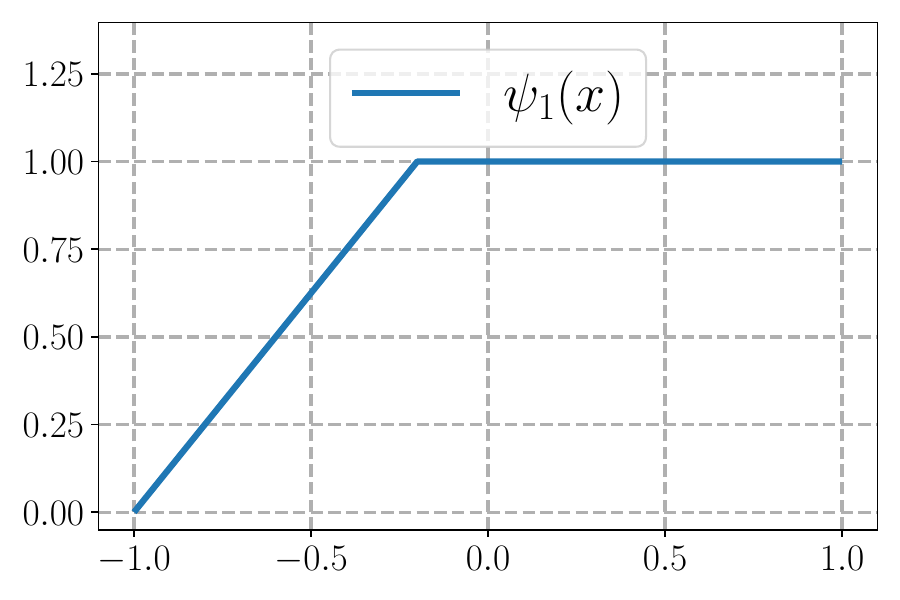}
    \end{subfigure}\hfill
        \begin{subfigure}[c]{0.1948\textwidth}
    \centering            \includegraphics[width=0.998055\textwidth]{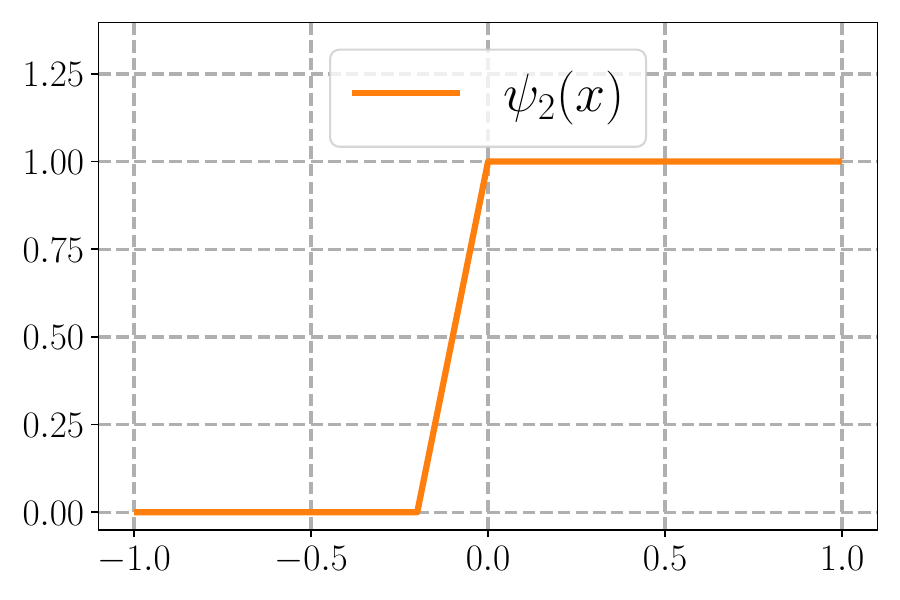}
    \end{subfigure}\hfill
    \begin{subfigure}[c]{0.1948\textwidth}
    \centering            \includegraphics[width=0.998055\textwidth]{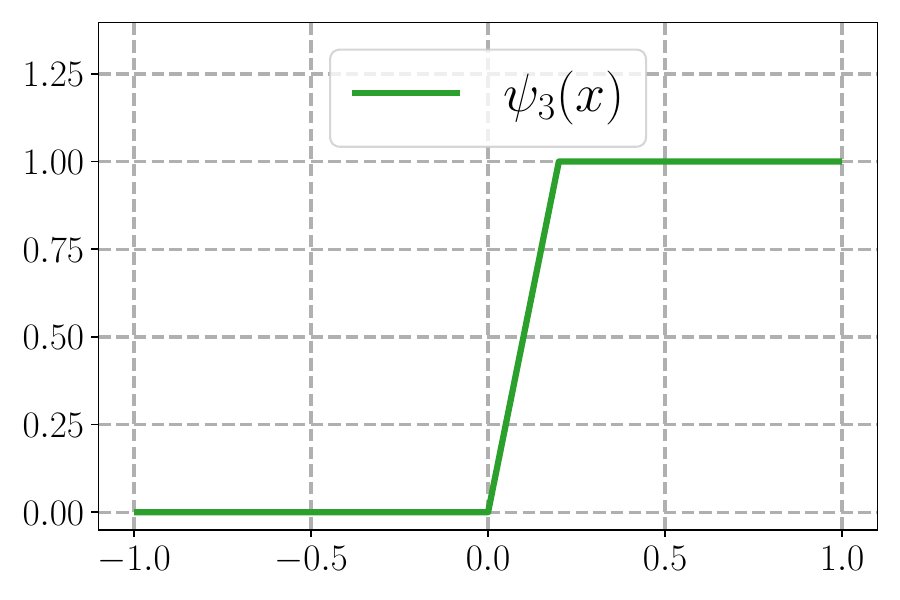}
    \end{subfigure}\hfill
        \begin{subfigure}[c]{0.1948\textwidth}
    \centering            \includegraphics[width=0.998055\textwidth]{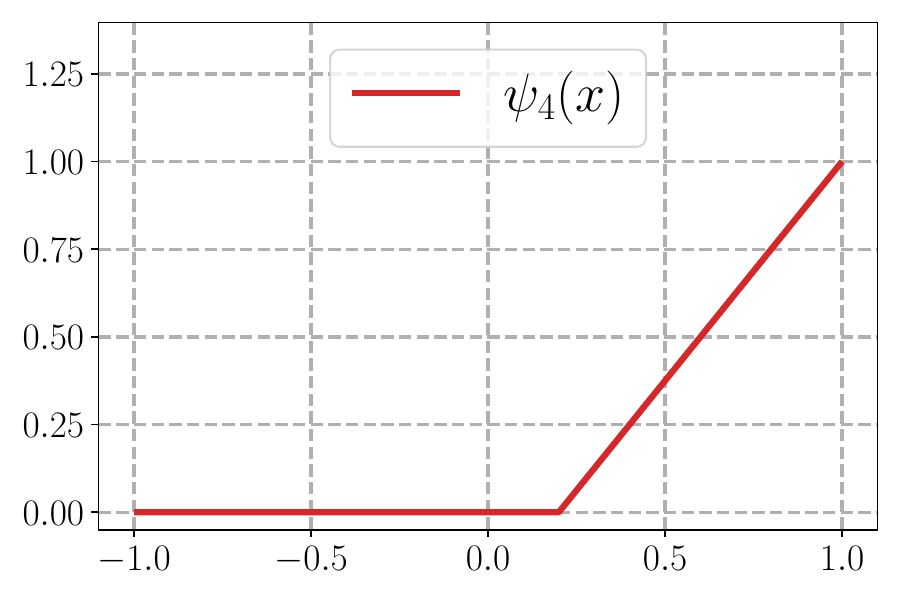}
    \end{subfigure}
    \caption{
Illustrations of $f(x)=\frac{1}{1000x^2+1}$ and its multi-component decomposition through $f_i$ and $\psi_i$ $i=1,2,3,4$, where
$f(x) = \sum_{i=1}^4 f_i\circ \psi_i(x) -  {\sum_{i=1}^{3} f(x_i)}$.
    }
    	\label{fig:Runge:F:decom}
\end{figure}

The second example is a globally oscillatory function of the form
\[
f(x) = \cos^2(6\pi x) + \sin(10\pi x^2).
\]
Again we illustrate using four components $n=4$, selecting points $x_0, x_1, x_2, x_3, x_4$ at $-1, -0.7, 0, 0.7, 1$, and setting $a_i = -1$ and $b_i = 1$ for all $i$. As shown in Figure~\ref{fig:CosSin:F:decom}, the target function $f(x)$ is decomposed into components that are less oscillatory again facilitating their approximation and learning through shallow networks.

\begin{figure}
    \centering	
    \begin{subfigure}[c]{0.1948\textwidth}
    \centering            \includegraphics[width=0.998055\textwidth]{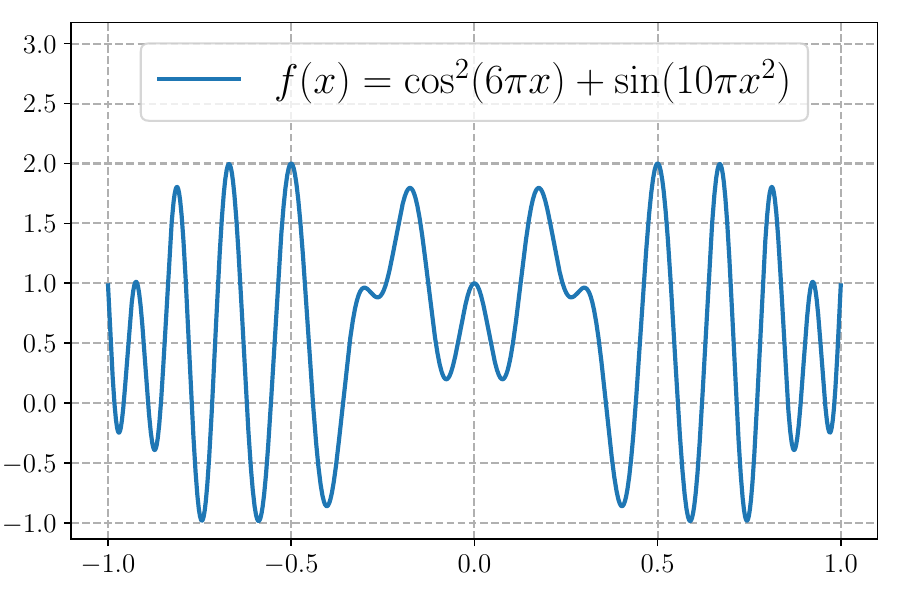}
    \end{subfigure}\hfill
        \begin{subfigure}[c]{0.1948\textwidth}
    \centering            \includegraphics[width=0.998055\textwidth]{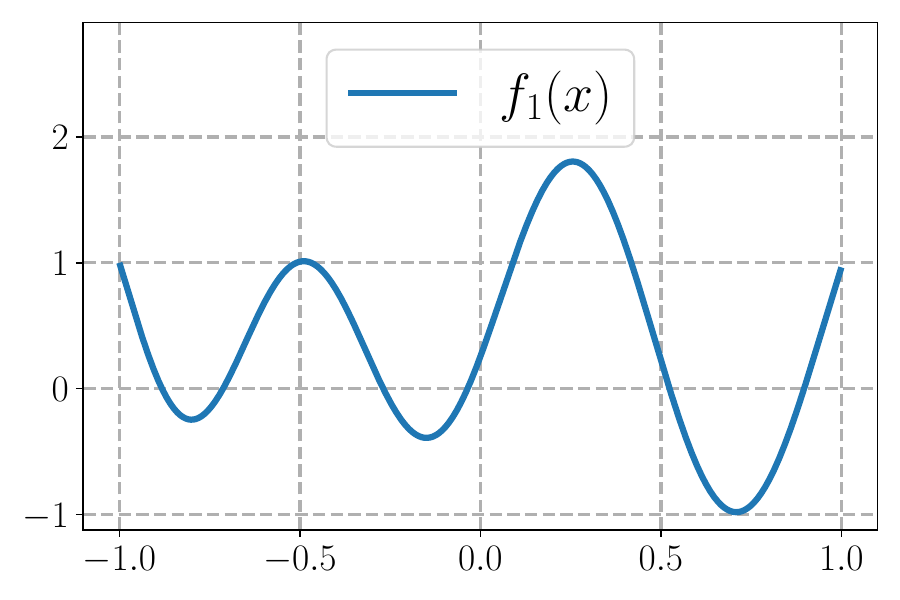}
    \end{subfigure}\hfill
        \begin{subfigure}[c]{0.1948\textwidth}
    \centering            \includegraphics[width=0.998055\textwidth]{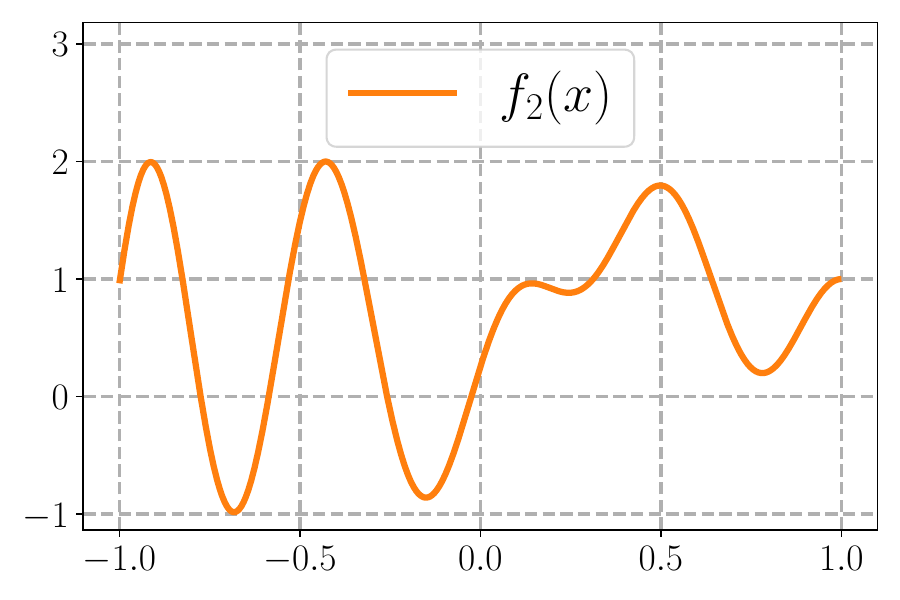}
    \end{subfigure}\hfill
    \begin{subfigure}[c]{0.1948\textwidth}
    \centering            \includegraphics[width=0.998055\textwidth]{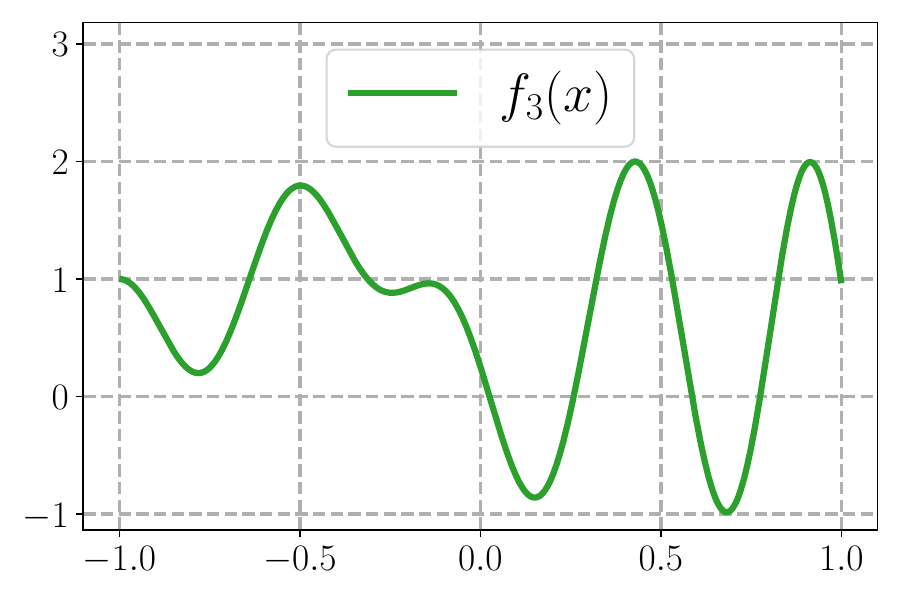}
    \end{subfigure}\hfill
        \begin{subfigure}[c]{0.1948\textwidth}
    \centering            \includegraphics[width=0.998055\textwidth]{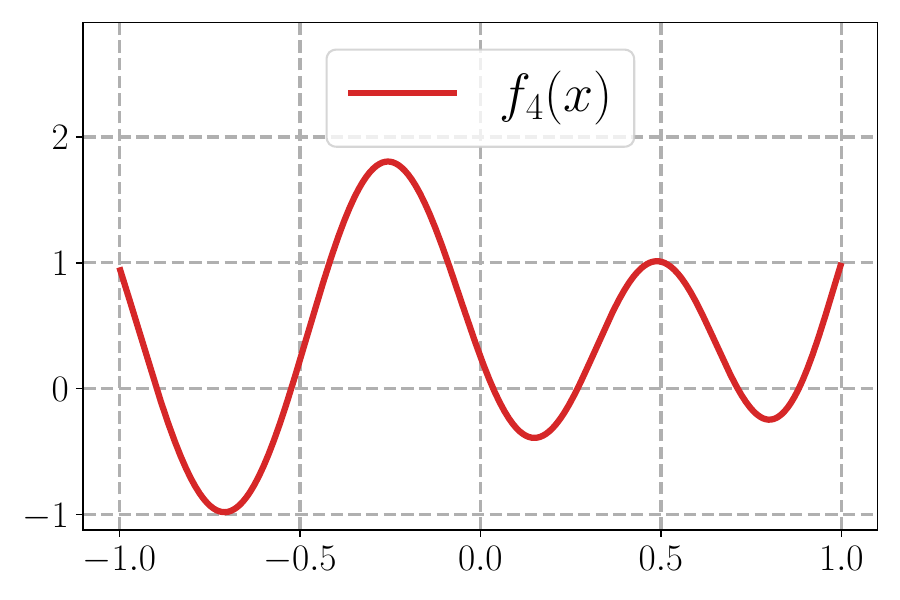}
    \end{subfigure}\\
            \begin{subfigure}[c]{0.1948\textwidth}
    \centering            \includegraphics[width=0.998055\textwidth]{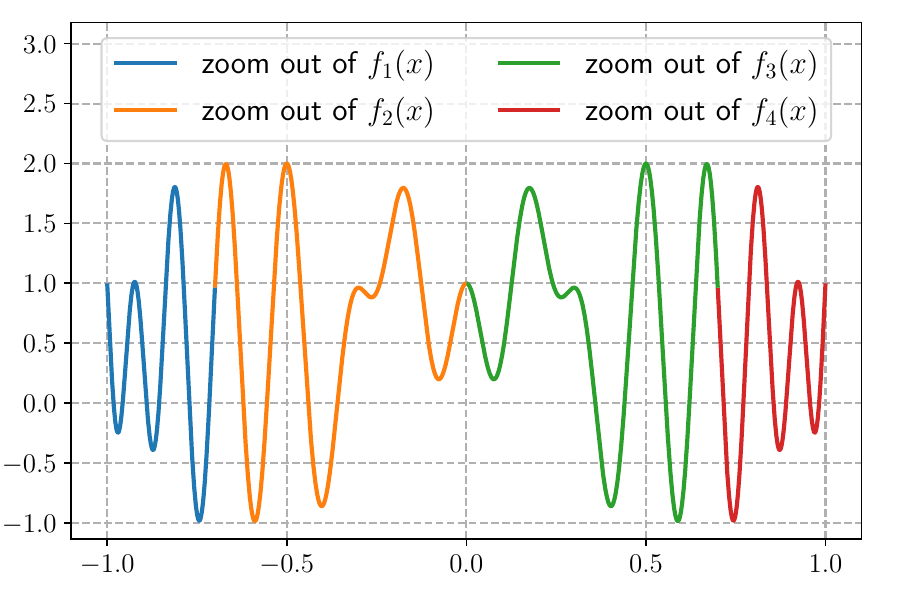}
    \end{subfigure}\hfill
        \begin{subfigure}[c]{0.1948\textwidth}
    \centering            \includegraphics[width=0.998055\textwidth]{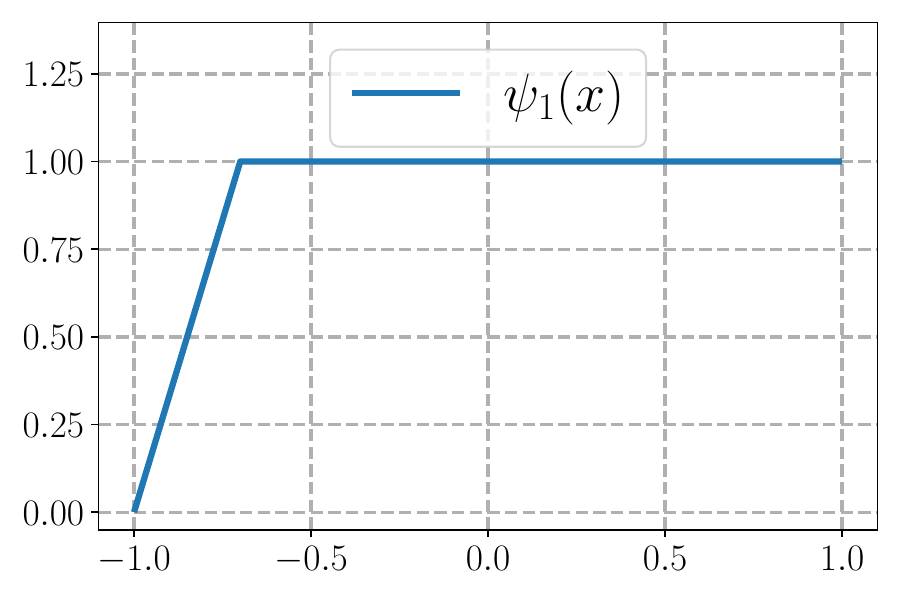}
    \end{subfigure}\hfill
        \begin{subfigure}[c]{0.1948\textwidth}
    \centering            \includegraphics[width=0.998055\textwidth]{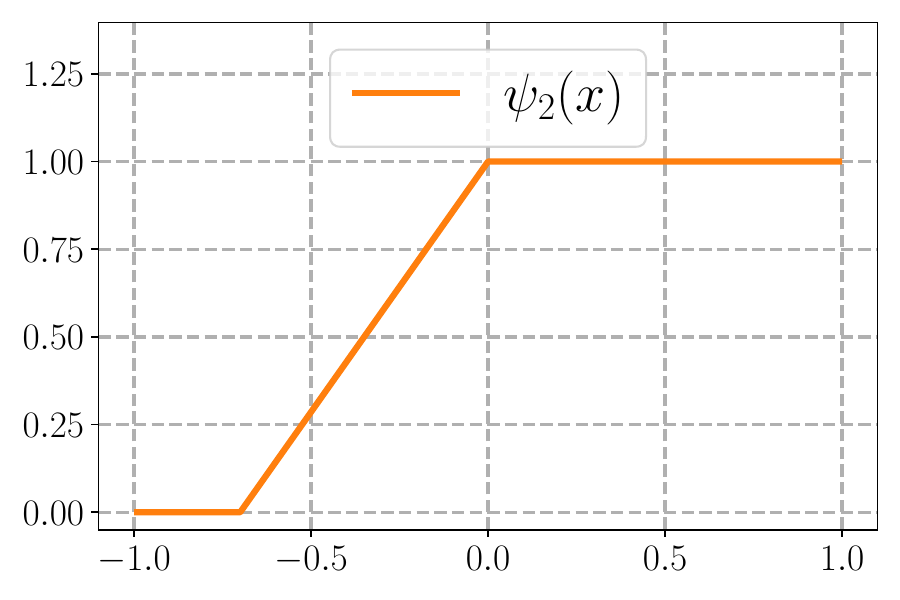}
    \end{subfigure}\hfill
    \begin{subfigure}[c]{0.1948\textwidth}
    \centering            \includegraphics[width=0.998055\textwidth]{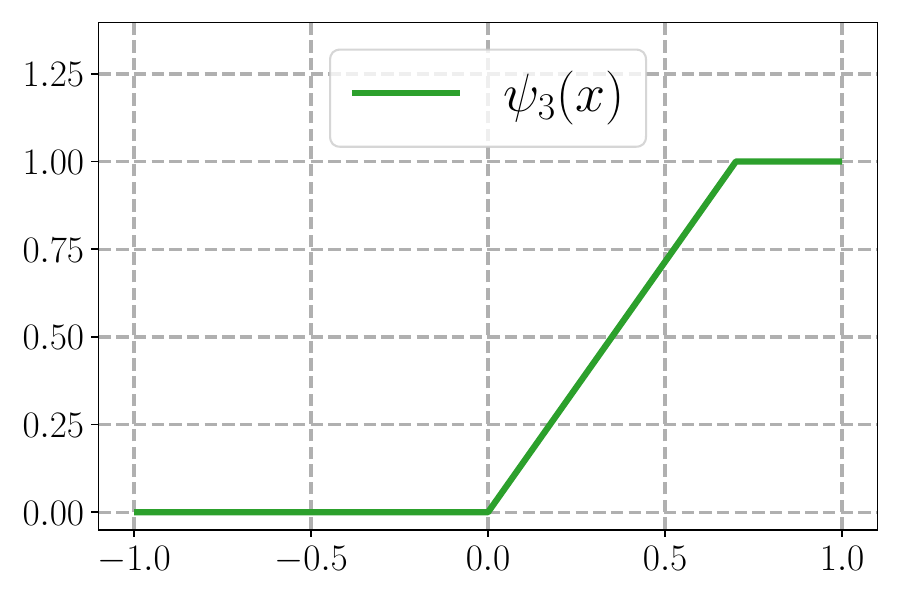}
    \end{subfigure}\hfill
        \begin{subfigure}[c]{0.1948\textwidth}
    \centering            \includegraphics[width=0.998055\textwidth]{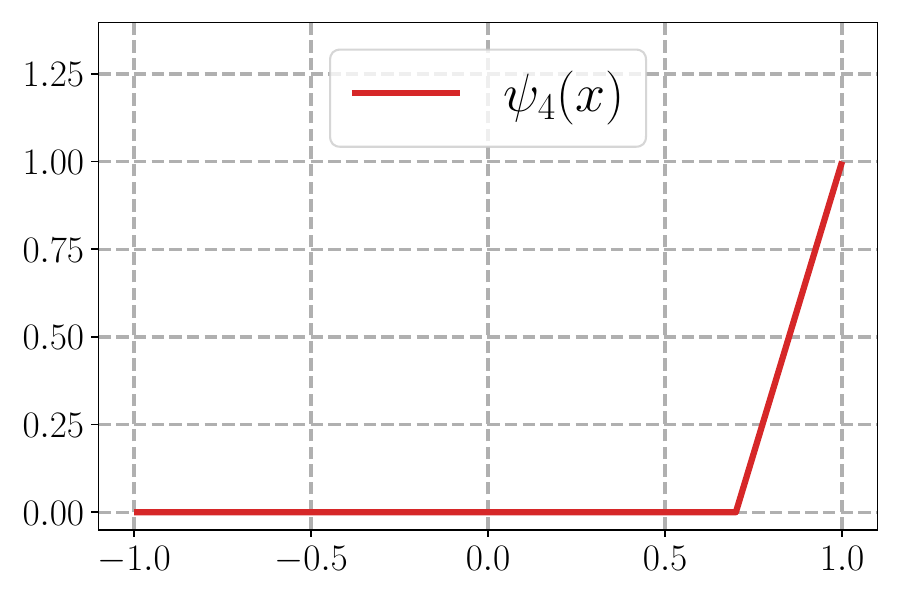}
    \end{subfigure}
    \caption{
Illustrations of $f(x)=\cos^2(6\pi x)+\sin(10\pi x^2)$ and its decomposition components $f_i$ and $\psi_i$ such that
$f(x) = \sum_{i=1}^4 f_i\circ \psi_i(x) -  {\sum_{i=1}^{3} f(x_i)}$.
    }
    	\label{fig:CosSin:F:decom}
\end{figure}

\subsection{High dimensional cases}
\label{sec:decomposition:highD}

Let us now consider the extension to multiple dimensions, using the case of two dimensions as an example, since the straightforward dimension-by-dimension strategy can be applied to any number of dimensions.

Given \(x_0<x_1<\cdots<x_n\) and \(y_0<y_1<\cdots<y_m\), dividing the domain of the function \(f(x,y)\) into small Cartesian rectangles \([x_{i-1},x_i]\times [y_{j-1}, y_j]\).
Let $\calL_{1,i}:[a_i,b_i]\to [x_{i-1},x_i]$ and $\calL_{2,j}:[c_i,d_i]\to [y_{j-1},y_j]$ be the linear maps with 
\begin{equation}
\label{eq:L:1i:2j}
\begin{cases}
        \calL_{1,i}(a_i)=x_{i-1},\\
    \calL_{1,i}(b_i)=x_i
\end{cases}
     \quad \tn{and}\quad\,\   
     \begin{cases}
        \calL_{2,j}(c_i)=y_{j-1},\\
    \calL_{2,j}(d_i)=y_j.
\end{cases}
\end{equation}
For $i=1,2,\cdots,n$ and $j=1,2,\cdots,m$, we define
\begin{equation}
\label{eq:fij}
\left\{\begin{array}{l}
    f_{i,0}(x,y)\coloneqq  f\Big(\calL_{1,i}(x),\, y\Big),
\\
    f_{0,j}(x,y)\coloneqq  f\Big(x,\, \calL_{2,j}(y)\Big),
    \\
    f_{i,j}(x,y)\coloneqq  f\Big(\calL_{1,i}(x),\, \calL_{2,j}(y)\Big)=f_{0,j}\Big(\calL_{1,i}(x),\, y\Big)=f_{i,0}\Big(x,\, \calL_{2,j}(y)\Big).
\end{array}\right.
\end{equation}
It is evident that with appropriate transformation and rescaling, \(f_{i,0}(x,y)\) is smooth in \(x\) when \(y\) is held constant, \(f_{0,j}(x,y)\) is smooth in \(y\) when \(x\) is fixed, and \(f_{i,j}(x,y)\) is smooth in both \(x\) and \(y\). 
Define
\begin{equation}\label{eq:psii:phij}
    \psi_i(x)=\begin{cases}
        b_i &\tn{if  } x>x_{i},\\
        \calL_{1,i}^{-1}(x) &\tn{if  }x\in [x_{i-1},x_i],\\
        a_i &\tn{if  }x<x_{i-1}
    \end{cases}
    \quad \tn{and}\quad 
        \phi_j(y)=\begin{cases}
        d_j &\tn{if  } y>y_{j},\\
        \calL_{2,j}^{-1}(y) &\tn{if  }y\in [y_{j-1},y_j],\\
        c_j &\tn{if  }y<y_{j-1}.
    \end{cases}
\end{equation}
The theorem below provides a decomposition of $f$ that fits into the MMNN structure.

\begin{theorem}
\label{thm:decomposition:2D}
Given \(x_0<x_1<\cdots<x_n\) and \(y_0<y_1<\cdots<y_m\),
    suppose $\calL_{1,i}, \calL_{2,j}$ and $\psi_i,\phi_j$ are given in Equations~\eqref{eq:L:1i:2j} and \eqref{eq:psii:phij}, respectively. Then the function $f:[x_0,x_n]\times [y_0, y_m]\to\R$ can be expressed as 
\begin{equation}
\label{eq:highD:decomposition}
    \begin{split}          
        f( x, \, y)
        & =\sum_{i=1}^n \sum_{j=1}^m f_{i,j}\Big( \psi_i(x), \, \phi_j(y)\Big)-
        \sum_{i=1}^n\sum_{j=1}^{m-1} f_{i,0}\Big(\psi_i(x),\, y_{j}\Big)
        \\ & \quad\   -  \sum_{i=1}^{n-1} \sum_{j=1}^m f_{0,j}\Big( x_i, \, \phi_j(y)\Big)+
        \sum_{i=1}^{n-1}\sum_{j=1}^{m-1} f (x_i,\, y_{j})
    \end{split}
\end{equation}
for all $(x,y)\in [x_0,x_n]\times [y_0,y_m]$, where $f_{i,j}$ are given in 
Equation~\eqref{eq:fij}.
\end{theorem}


\begin{proof}
Fixing $(k,j)$, for any $(x,y)\in [x_{k-1}, x_k]\times [y_{\ell-1}, y_\ell]$, we have
\begin{equation*}
    \psi_i(x)=\begin{cases}
        b_i &\tn{if  }i\le k-1,\\
        \calL_{1,k}^{-1}(x) &\tn{if  }i= k ,\\
        a_i &\tn{if  }i\ge k+1 
    \end{cases}
        \quad \tn{and}\quad 
        \phi_j(y)=\begin{cases}
        d_j &\tn{if  } j \le \ell-1,\\
        \calL_{2,\ell}^{-1}(y) &\tn{if  } j=\ell,\\
        c_j &\tn{if  } j\ge \ell+1.
    \end{cases}
\end{equation*}
It follows that 
\begin{equation*}
    \begin{split}
          &\phantom{=\;}\sum_{i=1}^n f_{i,0}\Big( \psi_i(x), \, y\Big)
        = \sum_{i=1}^n f\Big( \calL_{1,i}\circ \psi_i(x), \, y\Big)
        \\
        &=
        \sum_{i=1}^{k-1} f\Big( \calL_{1,i}\circ \psi_i(x), \, y\Big)
        +f\Big( \calL_{1,k}\circ \psi_k(x), \, y\Big)
        +
        \sum_{i=k+1}^n f\Big( \calL_{1,i}\circ \psi_i(x), \, y \Big)
        \\
        &=
        \sum_{i=1}^{k-1} f\Big( \calL_{1,i}(b_i), \, y\Big)
        +f\Big( \calL_{1,k}\circ \calL_{1,k}^{-1}(x), \, y\Big)
        +
        \sum_{i=k+1}^n f\Big( \calL_{1,i}(a_i), \, y\Big)
        \\
        &=
        \sum_{i=1}^{k-1} f( x_i, \, y)
        +f( x, \, y)
        +
        \sum_{i=k+1}^n f( x_{i-1}, \, y)
=
        f( x, \, y)
        +\sum_{i=1}^{n-1} f( x_i, \, y),
    \end{split}
\end{equation*}
implying 
\begin{equation*}
        f( x, \, y)=\sum_{i=1}^n f_{i,0}\Big( \psi_i(x), \, y\Big)-  \sum_{i=1}^{n-1} f( x_i, \, y).
\end{equation*}
For each \(i\), using the one-dimensional decomposition technique described in Section~\ref{sec:decomposition:1D}, we find the decompositions for 
\(f_{i,0}\big( \psi_i(x), \, y\big)\) and \(f( x_i, \, y)\).
We have 
{\small\begin{equation*}
    \begin{split}
          &\phantom{=\;}\sum_{j=1}^m f_{i,j}\Big( \psi_i(x), \, \phi_j(y)\Big)
        = \sum_{j=1}^m f_{i,0}\Big(\psi_i(x),\, \calL_{2,j}\circ \phi_j(y)\Big)\\
        &=
        \sum_{j=1}^{\ell-1} f_{i,0}\Big(\psi_i(x),\, \calL_{2,j}\circ \phi_j(y)\Big)
        +f_{i,0}\Big(\psi_i(x),\, \calL_{2,\ell }\circ \phi_\ell(y)\Big)
        +      \sum_{j=\ell+1}^m f_{i,0}\Big(\psi_i(x),\, \calL_{2,j}\circ \phi_j(y)\Big)
        \\
        &=
        \sum_{j=1}^{\ell-1} f_{i,0}\Big(\psi_i(x),\, \calL_{2,j}(d_j)\Big)
        +f_{i,0}\Big(\psi_i(x),\, \calL_{2,\ell }\circ \calL_{2,\ell }^{-1}(y)\Big)
        +
        \sum_{j=\ell +1}^m f_{i,0}\Big(\psi_i(x),\, \calL_{2,j}(c_j)\Big)
        \\
        &=
        \sum_{j=1}^{\ell-1} f_{i,0}\Big(\psi_i(x),\, y_j\Big)
        +f_{i,0}\Big(\psi_i(x),\, y\Big)
        +
         \sum_{j=\ell+1}^m f_{i,0}\Big(\psi_i(x),\, y_{j-1}\Big)
         \\&=f_{i,0}\Big(\psi_i(x),\, y\Big)
        +
        \sum_{j=1}^{m-1} f_{i,0}\Big(\psi_i(x),\, y_{j-1}\Big),
    \end{split}
\end{equation*}}implying 
\begin{equation}\label{eq:fi0:psiixy}
    \begin{split}
          f_{i,0}\Big(\psi_i(x),\, y\Big)
        =\sum_{j=1}^m f_{i,j}\Big( \psi_i(x), \, \phi_j(y)\Big)-
        \sum_{j=1}^{m-1} f_{i,0}\Big(\psi_i(x),\, y_{j}\Big).
    \end{split}
\end{equation}
Moreover, 
 \begin{equation*}
    \begin{split}
          &\phantom{=\;}\sum_{j=1}^m f_{0,j}\Big( x_i, \, \phi_j(y)\Big)
        = \sum_{j=1}^m f \Big(x_i,\, \calL_{2,j}\circ \phi_j(y)\Big)\\
        &=
        \sum_{j=1}^{\ell-1} f \Big(x_i,\, \calL_{2,j}\circ \phi_j(y)\Big)
        +f \Big(x_i,\, \calL_{2,\ell}\circ \phi_\ell(y)\Big)
        +
         \sum_{j=\ell+1}^mf \Big(x_i,\, \calL_{2,j}\circ \phi_j(y)\Big)
        \\
        &=
        \sum_{j=1}^{\ell-1} f \Big(x_i,\, \calL_{2,j}(d_j)\Big)
        +f \Big(x_i,\, \calL_{2,\ell }\circ \calL_{2,\ell }^{-1}(y)\Big)
        +
        \sum_{j=\ell+1}^m f \Big(x_i,\, \calL_{2,j}(c_j)\Big)
        \\
        &=
        \sum_{j=1}^{\ell-1} f (x_i,\, y_j)
        +f (x_i,\, y)
        +
        \sum_{j=\ell+1}^m f (x_i,\, y_{j-1})        =f (x_i,\, y)
        +
        \sum_{j=1}^{m-1} f (x_i,\, y_{j}),
    \end{split}
\end{equation*}
implying 
\begin{equation}\label{eq:fxiy}
    \begin{split}
          f (x_i,\, y)
        =\sum_{j=1}^m f_{0,j}\Big( x_i, \, \phi_j(y)\Big)-
        \sum_{j=1}^{m-1} f (x_i,\, y_{j}).
    \end{split}
\end{equation}
Therefore, for any $(x,y)\in [x_{k-1}, x_k]\times [y_{\ell-1}, y_\ell]$, by \eqref{eq:fi0:psiixy}  and \eqref{eq:fxiy}, we have
\begin{equation*}
    \begin{split}          
        f( x, \, y)
        &=\sum_{i=1}^n f_{i,0}\Big( \psi_i(x), \, y\Big)-  \sum_{i=1}^{n-1} f( x_i, \, y)
        \\& =\sum_{i=1}^n \sum_{j=1}^m f_{i,j}\Big( \psi_i(x), \, \phi_j(y)\Big)-
        \sum_{i=1}^n\sum_{j=1}^{m-1} f_{i,0}\Big(\psi_i(x),\, y_{j}\Big)
           \\& \quad\  -  \sum_{i=1}^{n-1} \sum_{j=1}^m f_{0,j}\Big( x_i, \, \phi_j(y)\Big)+
        \sum_{i=1}^{n-1}\sum_{j=1}^{m-1} f (x_i,\, y_{j}).
    \end{split}
\end{equation*}
Since $k$ and $\ell$ are arbitrary, the above equation holds for all
$(x,y)=\cup_{k=1}^n\cup_{\ell=1}^m[x_{k-1}, x_k]\times [y_{\ell-1}, y_\ell]= [x_0,x_n]\times [y_0,y_m]$.
\end{proof}

\subsection{Related work}
\label{sec:related:work}
Several lines of research are closely related to this work, including approximation theory, low-rank methods, random feature models, and architectures inspired by the Kolmogorov–Arnold representation.
\paragraph{Approximation}
Extensive research has examined the approximation capabilities of neural networks, focusing on various architectures to approximate diverse target functions. Early studies concentrated on the universal approximation power of single-hidden-layer networks~\cite{Cybenko1989ApproximationBS,HORNIK1991251,HORNIK1989359}, which demonstrated that sufficiently large neural networks could approximate specific functions with arbitrary precision mathematically, without quantifying the error relative to network size. Subsequent research, such as~\cite{yarotsky18a,yarotsky2017,doi:10.1137/18M118709X,ZHOU2019,10.3389/fams.2018.00014,2019arXiv190501208G,2019arXiv190207896G,MO,shijun:NonlineArpprox,shijun:Characterized:by:Numer:Neurons,shijun:smooth:functions,shijun:arbitrary:error:with:fixed:size,shijun:thesis,shijun:intrinsic:parameters,shijun:RCNet,JMLR:v25:23-0912,shijun:net:arc:beyond:width:depth,shijun:optimal:rate:in:width:and:depth}, analyzed the approximation error for different networks in terms of size characterized by width, depth, or the number of parameters. 
Those studies have primarily concentrated on the mathematical theory that supports the existence theory for such neural networks. However, there has been limited focus on determining the parameters within these networks computationally and the numerical errors, particularly those arising from finite precision in computer simulations. This gap motivated our current investigation, which considers practical training processes and numerical errors. Specifically, the balanced structure of MMNN, the choice of training parameters, and the associated learning strategy discussed here are intended to facilitate a smooth decomposition of the function, thereby promoting an efficient training process.

\paragraph{Low-rank methods}
Low-rank structures in the weight matrix $\bmW$ of a fully connected neural network have been investigated by various groups. For example, the methods proposed in \cite{2022arXiv220913569R,9157223,6638949} focus on accelerating training and reducing memory requirements while maintaining final performance. The concept of low-rank structures is further extended to tensor train decomposition in~\cite{novikov2015tensorizing}.
The MMNN proposed here differs in two key aspects. First, each layer contains two matrices: $\bmA$ outside and $\bmW$ inside the activation functions. Each row of $\bmA$ represents the weights for a linear combination of a set of random basis functions, forming a component in each layer. The number of rows in $\bmA$, which equals the number of components, is selected based on the complexity of the function and is typically much smaller than the number of columns, corresponding to the number of basis functions. Each row of $(\bmW, \bmb)$ represents a random parameterization of a basis function, with the number of rows in $\bmW$ corresponding to the number of basis functions, usually much larger than the number of columns in $\bmW$, which is the input dimension.
Secondly, in our MMNN, only $\bmA$ is trained while $\bmW$ remains fixed with randomly initialized values. Theoretical studies and numerical experiments demonstrate that the architecture of MMNN, combined with the learning strategy, is effective in approximating complex functions.


\paragraph{Random features} 
Fixing $(\bmW, \bmb)$ of each layer and use of random basis functions in the MMNNs is inspired by a previous approach known as random features \cite{NIPS2007_013a006f,NIPS2016_e7061188,9495136,NIPS2017_61b1fb3f,peng2021random}. In typical random feature methods, only the linear combination parameters at the output layer are trained which also leads to the issue of ill-conditioning of the representation. While in MMNNS matrix $\bmA$ and vector $\bmc$ of each layer are trained. 
Our MMNN employs a composition architecture and learning mechanism that enhances the approximation capabilities compared to random feature methods while achieving a more effective training process than a standard fully connected network of equivalent size. Extensive experiments demonstrate that our approach can strike a satisfactory balance between approximation accuracy and training cost.

\paragraph{Komogolrov-Arnold (KA) representation} 
The KA representation theorem \cite{kolmogorov1957} states that any multivariate continuous function on a hypercube can be expressed as a finite composition of continuous univariate functions and the binary operation of addition. However, this elegant mathematical representation may result in compositions of non-smooth or even fractal univariate functions in general, a computational challenge one has to address in practice. 
KA representation has been explored in several studies \cite{2024arXiv240419756L,MAIOROV199981,shijun:arbitrary:error:with:fixed:size,ISMAYILOVA2024106333}. A recently proposed network known as the KA network (KAN) utilizes spline functions to approximate the univariate functions in the KA representation. The proposed MMNN is motivated by a multi-component and multi-layer smooth decomposition, or a ``divide and conquer" approach, employing distinct network architectures, activation functions, and training strategies.

\section{Numerical experiments}
\label{sec:experiments}

We perform extensive experiments to validate our analysis and demonstrate the effectiveness of MMNNs through multi-component and multi-layer decomposition studied in Section~\ref{sec:decomposition}. In particular, our tests show its ability in (1) adaptively capturing localized high-frequency features in Section~\ref{sec:local}, (2) approximating highly oscillatory functions in Section~\ref{sec:oscillatory}, 
(3) approximating discontinuous functions with porous structures in Section~\ref{sec:discontinuous:func:porous}, (4) some interesting learning dynamics in Section~\ref{sec:dynamics}, and 
(5) solving problems in three and higher dimensions in Section~\ref{sec:higher}.
All our experiments involve target functions that include high-frequency components in various ways and are difficult to handle by shallow networks (no matter how wide) as shown in our previous work \cite{ZZZZ-23}. Moreover, our experience on these tests shows that using a fully connected deep neural network would require many more parameters and is much harder (if possible) to train to get a comparable result. This is mainly due to a balanced and structured network design of MMNN in terms of (1) the network width $w$, which is the number of hidden neurons or random basis functions in each component, (2) the rank $r$, which is the number of components in each layer, and (3) the network depth $l$, which is the number of layers in the network. The use of a controllable number of collective components (through $\bmA$) in each layer instead of a large number of individual neurons and the use of fixed and randomly chosen weights $(\bmW,\bmb)$ make the training process more effective.

In all tests, (1) data are sampled enough to resolve fine features in the target function, (2) the Adam optimizer is used in training, (3) the mean squared error (MSE) is the loss function, (4) the default activation function used is \ReLU, (5)
all computation and training use single precision in PyTorch,
(6) 
all parameters are initialized according to the PyTorch default initialization (Section~\ref{sec:MMNN:learning:stategy}) unless otherwise specified,
(7) $\bmW$'s and $\bmb$'s (the parameters inside the activation functions, see Section~\ref{sec:MMNN:architecture}) are fixed and only $\bmA$'s and $\bmc$'s (the parameters outside the activation functions) are trained,
(8)
computations are conducted on a \textit{NVIDIA RTX 3500 Ada Generation Laptop GPU (power cap 130W)}, with most experiments concluding within a range from a few dozen to several thousand seconds.
All our MMNN setups are specified by three parameters $(w,r,l)$ which depend on the function complexity. Another tuning parameter is the learning rate which is guided by the following criteria: (1) not too large initially due to stability, (2) a decreasing rate with iterations such that the learning rate becomes small near the equilibrium to achieve a good accuracy while not decreasing too fast (especially during a long training process for more difficult target functions) so that the training is stalled. 

\subsection{Localized rapid changes}\label{sec:local}

We begin with two one-dimensional examples. The first is $f(x) = \arctan(100x+20)$, which is smooth but features a rapid transition at zero. While demonstrated in our previous work \cite{ZZZZ-23}, a shallow network struggles to capture such a simple local fast transition which contains high-frequencies, we show that this function can be approximated easily by a composition of a smooth function on top of a (repeated) spatial decomposition and local rescaling using MMNN structure in Section~\ref{sec:MMNN:architecture}.  Our test indeed verifies that our new architecture can effectively capture a localized fast transition rather easily using a very small network of size $(16,4,3)$ as shown in Figure~\ref{fig:arctan:100x}. 
For this test, a total of 1000 data points are uniformly sampled in the range $[-1, 1]$, with a mini-batch size of 100, a learning rate of $10^{-3}$, and the number of epochs set to 2000. Figure~\ref{fig:arctan:100x:error} gives the error plot.

\begin{figure}
\begin{minipage}[b]{0.70\linewidth}
    \centering
    \begin{subfigure}[c]{0.32431\textwidth}
        \centering
        \includegraphics[width=0.998055\textwidth]{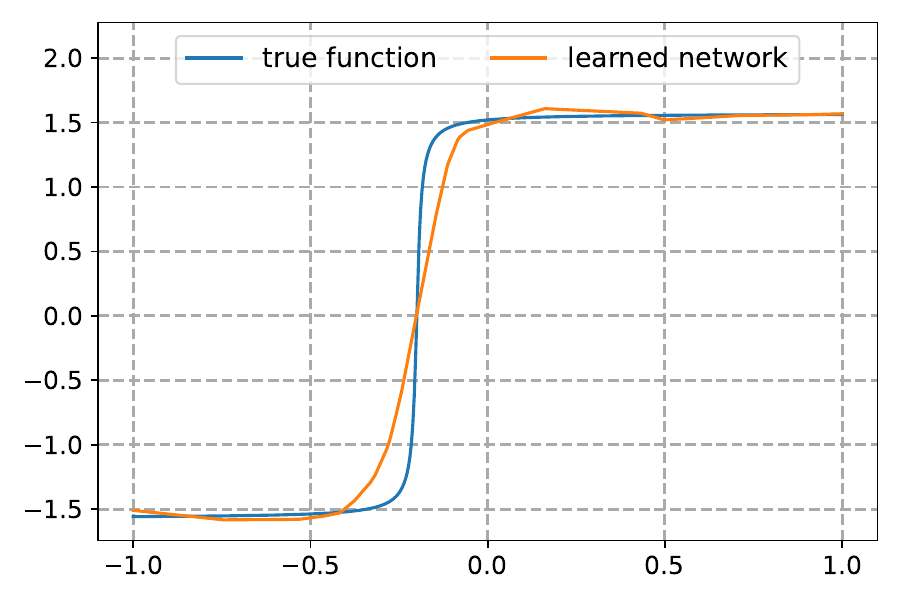}
        \subcaption{Epoch 100.}
    \end{subfigure}
    \hfill
    \begin{subfigure}[c]{0.32431\textwidth}
        \centering
        \includegraphics[width=0.998055\textwidth]{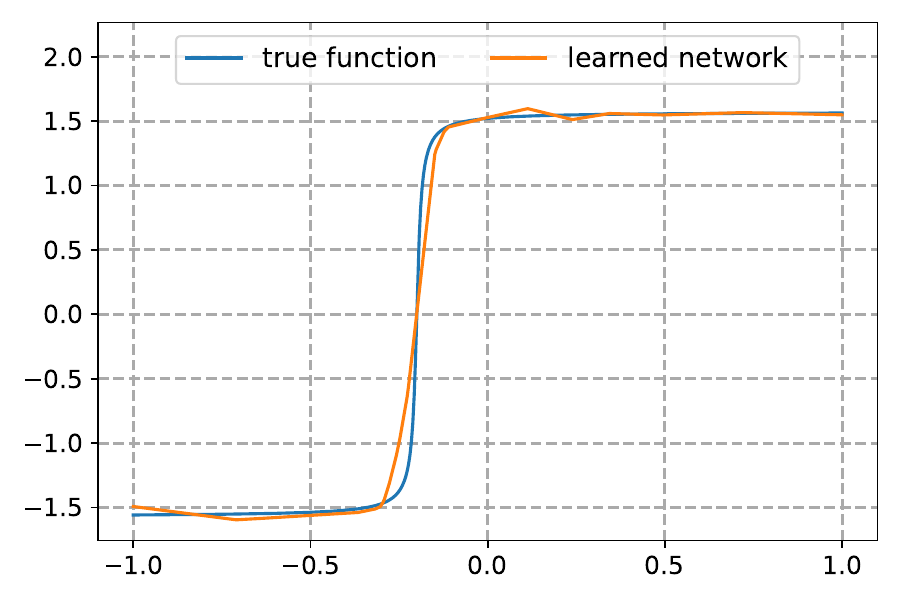}
        \subcaption{Epoch 200.}
    \end{subfigure}
    \hfill
    \begin{subfigure}[c]{0.32431\textwidth}
        \centering
        \includegraphics[width=0.998055\textwidth]{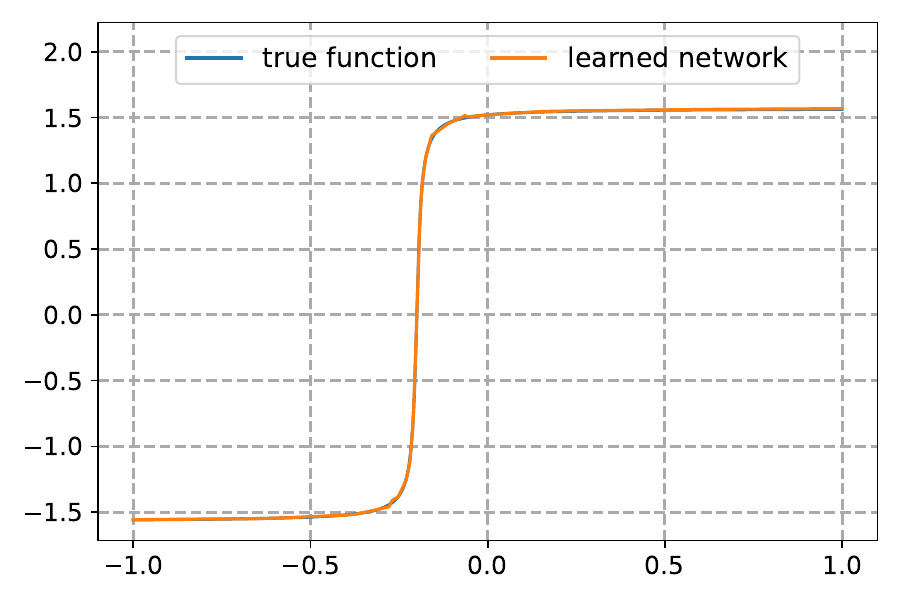}
        \subcaption{Epoch 2000.}
    \end{subfigure}
    \caption{Illustrations of the training process.}
    \label{fig:arctan:100x}
    \end{minipage}
    \hfill
    \begin{minipage}[b]{0.27\linewidth}
   \centering       \includegraphics[width=0.998055\textwidth]{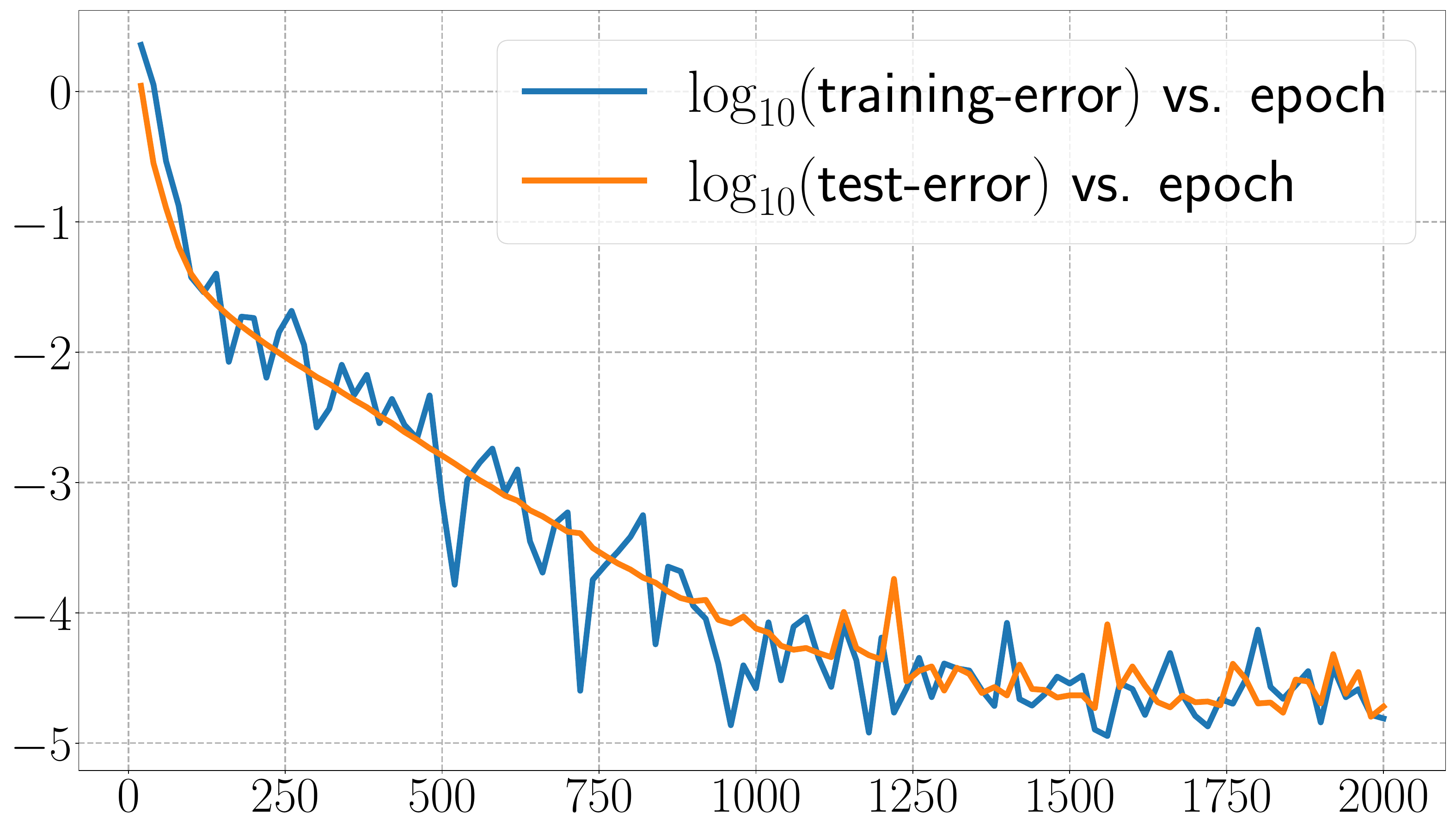}
    \caption{Training and test errors (MSE).}
    \label{fig:arctan:100x:error} 
    \end{minipage}
\end{figure}

Next, we consider a more complicated target function, $f(x) = \one_{\{|x+0.2| < 0.02\}} \cdot \sin(50\pi x)$, which represents a localized fast oscillation.
 For this example, we will conduct two tests. The first one is to show the flexibility of MMNN to automatically adapt to local features. The network has a small size as above $(16, 4, 3)$. Each layer has a network width of 16. In other words, each component is a linear combination of 16 \texttt{ReLU} functions which has no way to approximate such a target function well. However, with a multi-layer and multi-component decomposition with parameters appropriately trained by Adam, MMNN can adapt to the behavior of the target function as shown in Figure~\ref{fig:sin:50pi:one:period:sample1}. Figure~\ref{fig:sin:50pi:one:period:error} gives the error plot. Also, the test shows that this example is more difficult to train.  For this test, there are a total of 1000 uniformly sampled points in $[-1,1]$ with a mini-batch size of 100 and a learning rate of $0.002\times 0.95^{\lfloor k/1000 \rfloor}$, where $\lfloor \cdot \rfloor$ denotes floor operation and $k=1,2,\cdots,20000$ is the epoch number. 
It should be noted that in this test, we initialize the biases $\bmb$'s to $\bm{0}$ and use the PyTorch default initialization method for the weights $\bmW$. This approach, inspired by Xavier initialization, is chosen because the target function is locally oscillatory 
and the MMNN size is quite small, necessitating a setup adaptive to the target function to facilitate the training. For other experiments, both the biases and weights use the PyTorch default initialization.
We then compare with least square approximation using uniform finite element method (FEM) basis with the same degrees of freedom. As shown in Figure~\ref{fig:sin:50pi:FEM:vs:low:rank:NN}, MMNN renders a better approximation due to automatic adaptation through the training process. 
We remark that when training an MMNN with an extreme compact size with respect to the target function complexity, due to the lack of flexibility/redundancy, the training may become more subtle and need more careful calibration, such as initialization, learning rate, min-batch size, and etc. 
However, introducing slight redundancy into an MMNN, such as by increasing its size marginally, enhances its flexibility and makes training more tractable.
On the other hand, when the network becomes too large, then training a large number of parameters and over-redundancy will lead to potential difficulties for optimization. This also shows that there is a trade-off between representation and optimization one needs to balance in practice.


\begin{figure}
    \centering
    \begin{subfigure}[c]{0.231\textwidth}
        \centering
        \includegraphics[width=0.998055\textwidth]{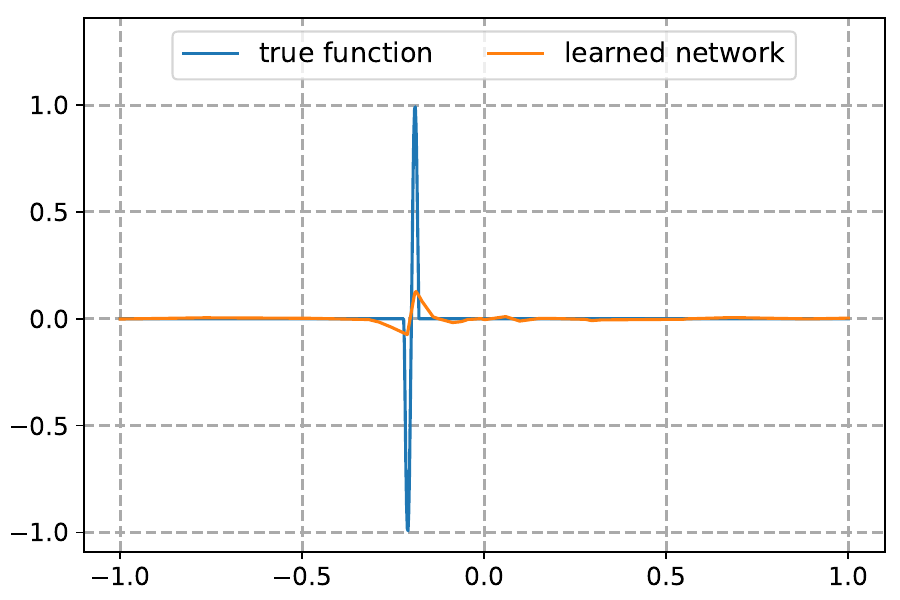}
        \subcaption{Epoch 500.}
    \end{subfigure}
    \hfill
    \begin{subfigure}[c]{0.231\textwidth}
        \centering
        \includegraphics[width=0.998055\textwidth]{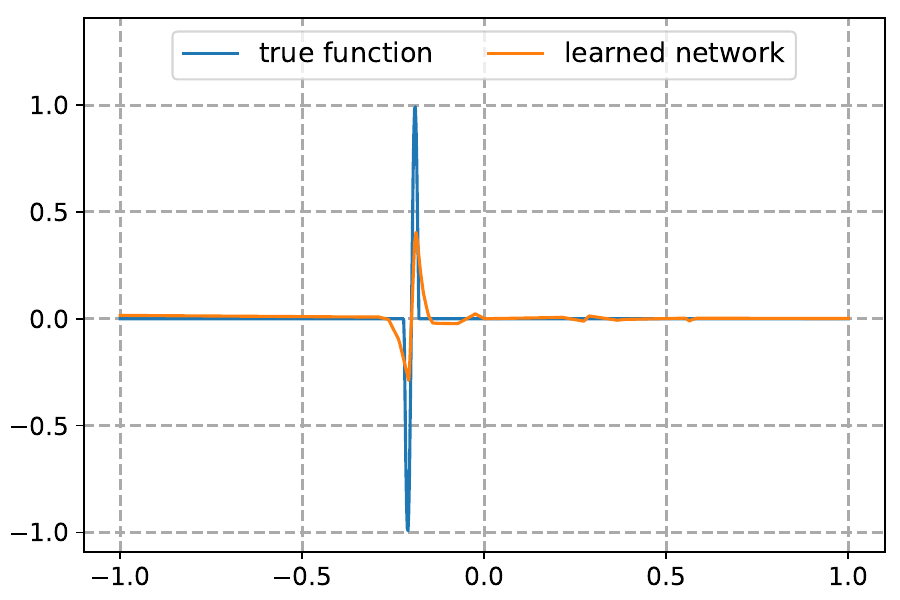}
        \subcaption{Epoch 1000.}
    \end{subfigure}
    \hfill
    \begin{subfigure}[c]{0.231\textwidth}
        \centering
        \includegraphics[width=0.998055\textwidth]{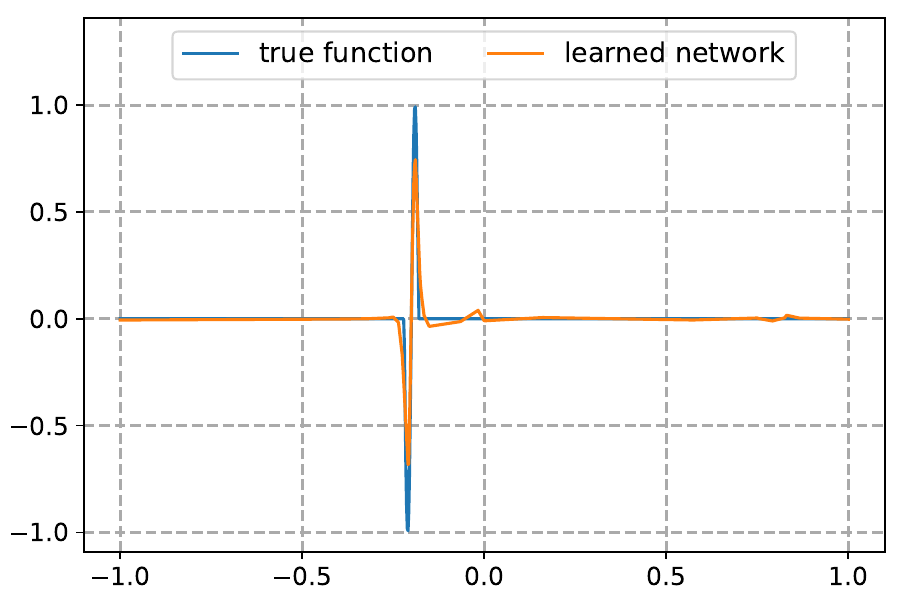}
        \subcaption{Epoch 2000.}
    \end{subfigure}
    \hfill
    \begin{subfigure}[c]{0.231\textwidth}
        \centering
        \includegraphics[width=0.998055\textwidth]{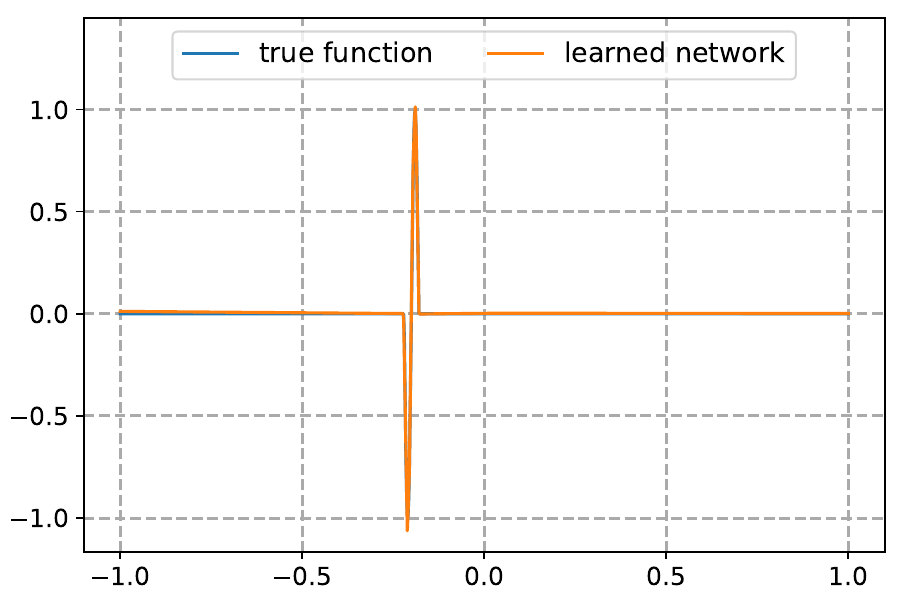}
        \subcaption{Epoch 20000.}
    \end{subfigure}
    \caption{Illustrations of the training process.}
    \label{fig:sin:50pi:one:period:sample1}
\end{figure}


\begin{figure}
    \begin{minipage}[b]{0.32\linewidth}
   \centering       \includegraphics[width=0.75\textwidth]{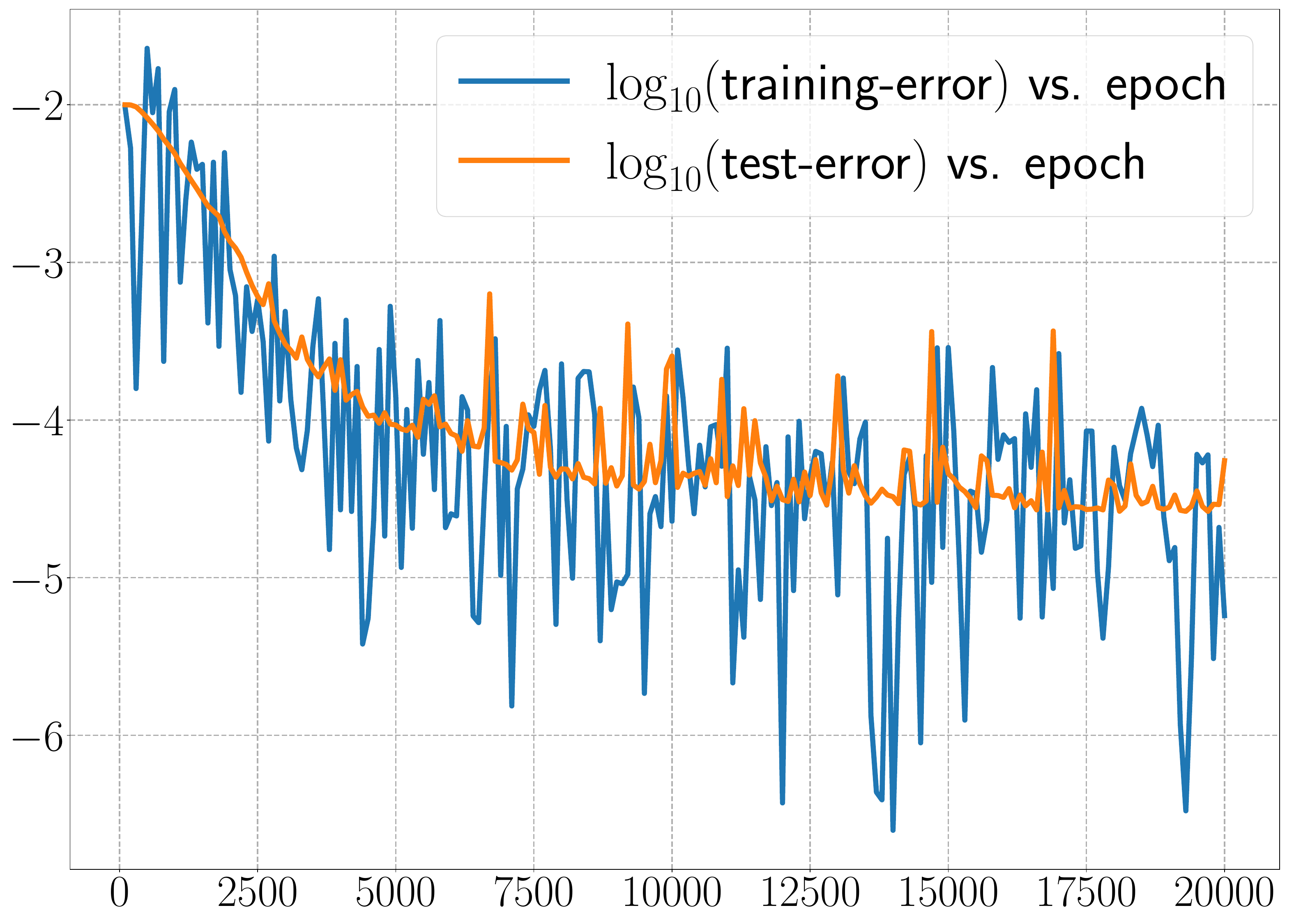}
    \caption{Training and test errors (in MSE) vs. epoch.}    \label{fig:sin:50pi:one:period:error} 
    \end{minipage}
    \hfill
\begin{minipage}[b]{0.64\linewidth}
    \centering	
            \begin{subfigure}[c]{0.458\textwidth}
    \centering            \includegraphics[width=0.8998055\textwidth]{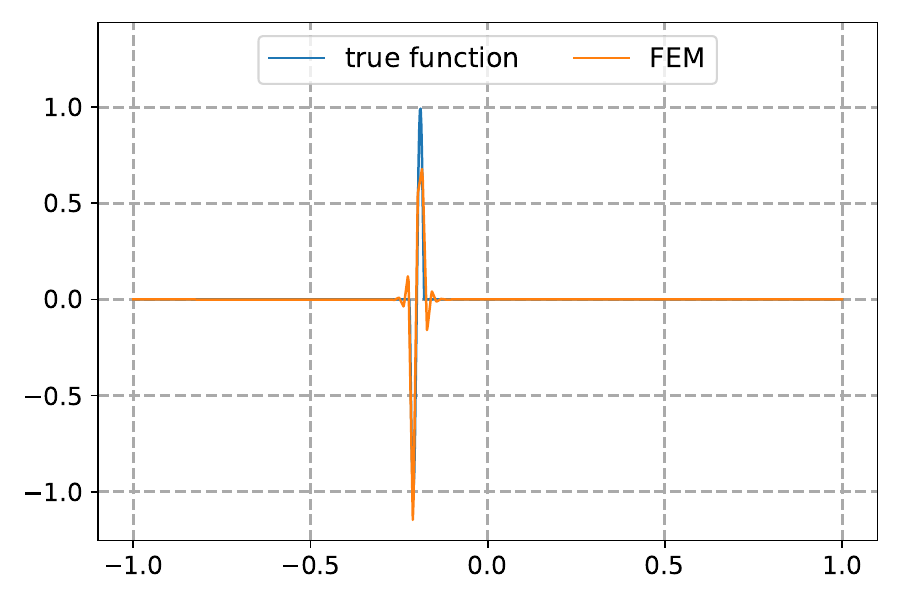}
    \end{subfigure}
    \hspace{5pt}
    \begin{subfigure}[c]{0.458\textwidth}
    \centering            \includegraphics[width=0.8998055\textwidth]{figures/Local1D/sin1P/epoch20000_idx0_w16r4d3.pdf}
    \end{subfigure}
\caption{Left: Least square using equally spaced 153 FEM bases. Right: MMNN with $153$ free parameters.} 	\label{fig:sin:50pi:FEM:vs:low:rank:NN}
    \end{minipage}
\end{figure}

Finally, we show an example in the two-dimensional  case shown in~Figure~\ref{fig:f3D3:local:HF} and defined in polar coordinates by
\begin{equation*}
    f(r,\theta)=\begin{cases}
        0 & \tn{if  } 0.5+25\rho-25r\le 0,\\
        1 & \tn{if  } 0.5+25\rho-25r\ge 1,\\
        0.5+25\rho-25r & \tn{otherwise},\\
    \end{cases}
    \quad \tn{where}\ \rho=0.1+0.02\cos(8\pi  \theta).
\end{equation*}
Again a rather compact MMNN of size $(100,10,6)$ can produce a good approximation. 
Figure~\ref{fig:2D-3} shows the $\log$ plot of training and testing errors in MSE. For this test there are a total of $400^2$ uniformly sampled points in $[-1,1]^2$ with mini-batch size of $1000$ and a learning rate of $10^{-3}\times 0.9^{\lfloor k/25 \rfloor}$, where $k=1,2,\cdots,1000$ is the epoch number. We compare the result with piecewise linear interpolation and least square approximation using FEM basis on a uniform grid with the same number of degrees of freedom in Figure~\ref{fig:comparison}. As observed before, MMNN renders the best result due to its adaptivity through training.
    When adaptive finite element is applicable, it is hard to beat. However, adaptive FEMs are more humanly involved, while MMNNs based on training are more automatic. The key message here is to demonstrate that MMNN has adaptive features through training/optimization, which will be useful when an adaptive finite element method becomes difficult or impossible in applications. 

\begin{figure}
\centering
\begin{minipage}[b]{0.5427161\linewidth}
   \centering	
    \begin{subfigure}[c]{0.48538\textwidth}
    \centering            \includegraphics[height=0.75\textwidth]{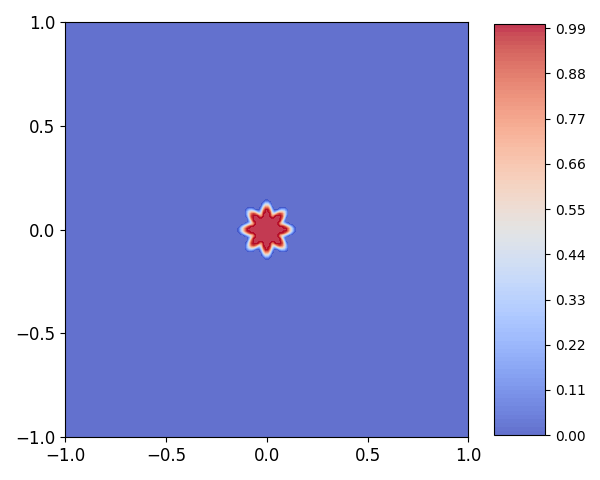}
    \end{subfigure}
    \hfill
    \begin{subfigure}[c]{0.48538\textwidth}
    \centering            \includegraphics[height=0.75\textwidth]{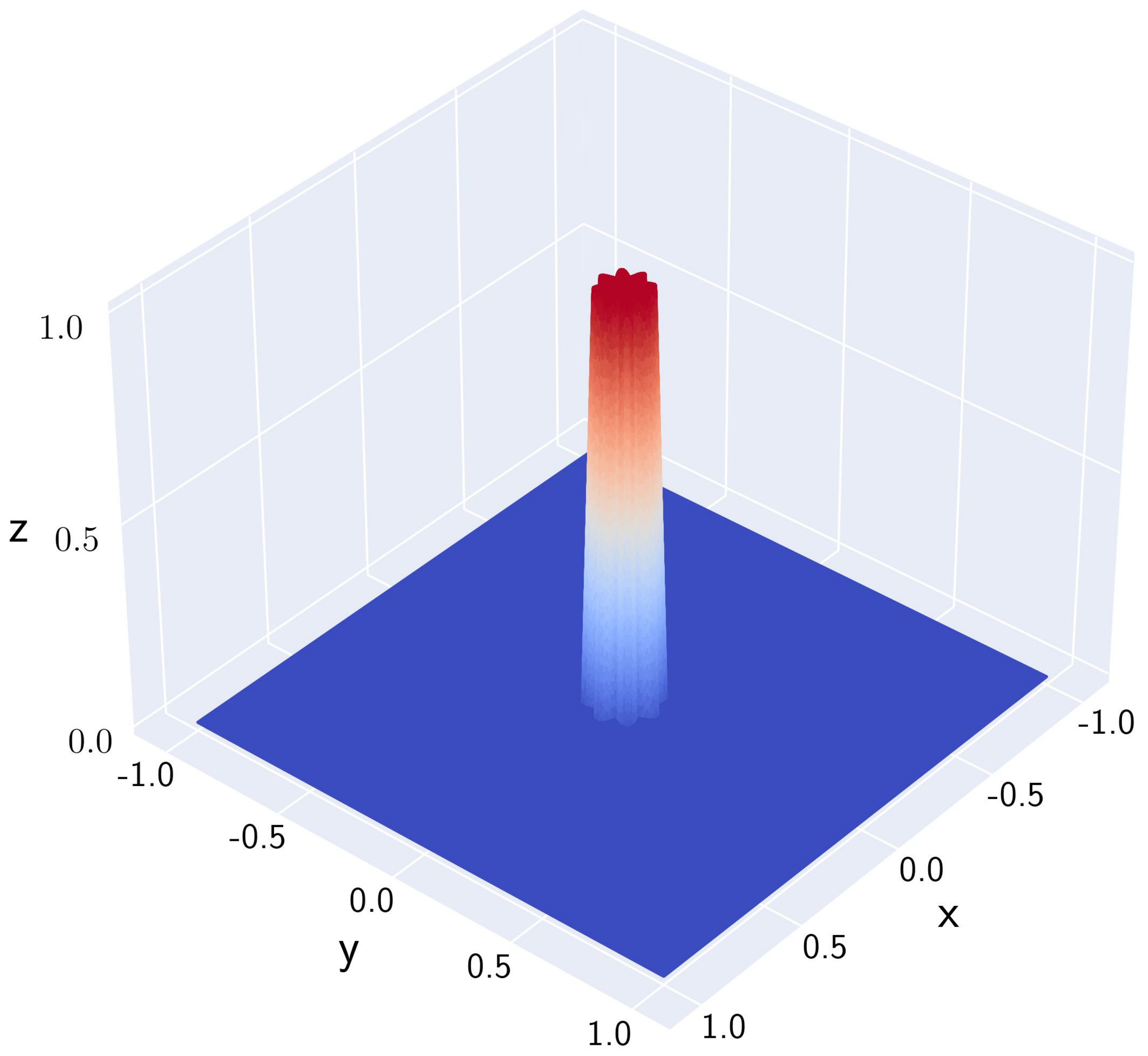}
    \end{subfigure}
    \caption{Target function.}
    	\label{fig:f3D3:local:HF}
\end{minipage}
\hspace{0.1\linewidth}
\begin{minipage}[b]{0.2963\linewidth}
    \centering	    \includegraphics[width=0.95\textwidth]{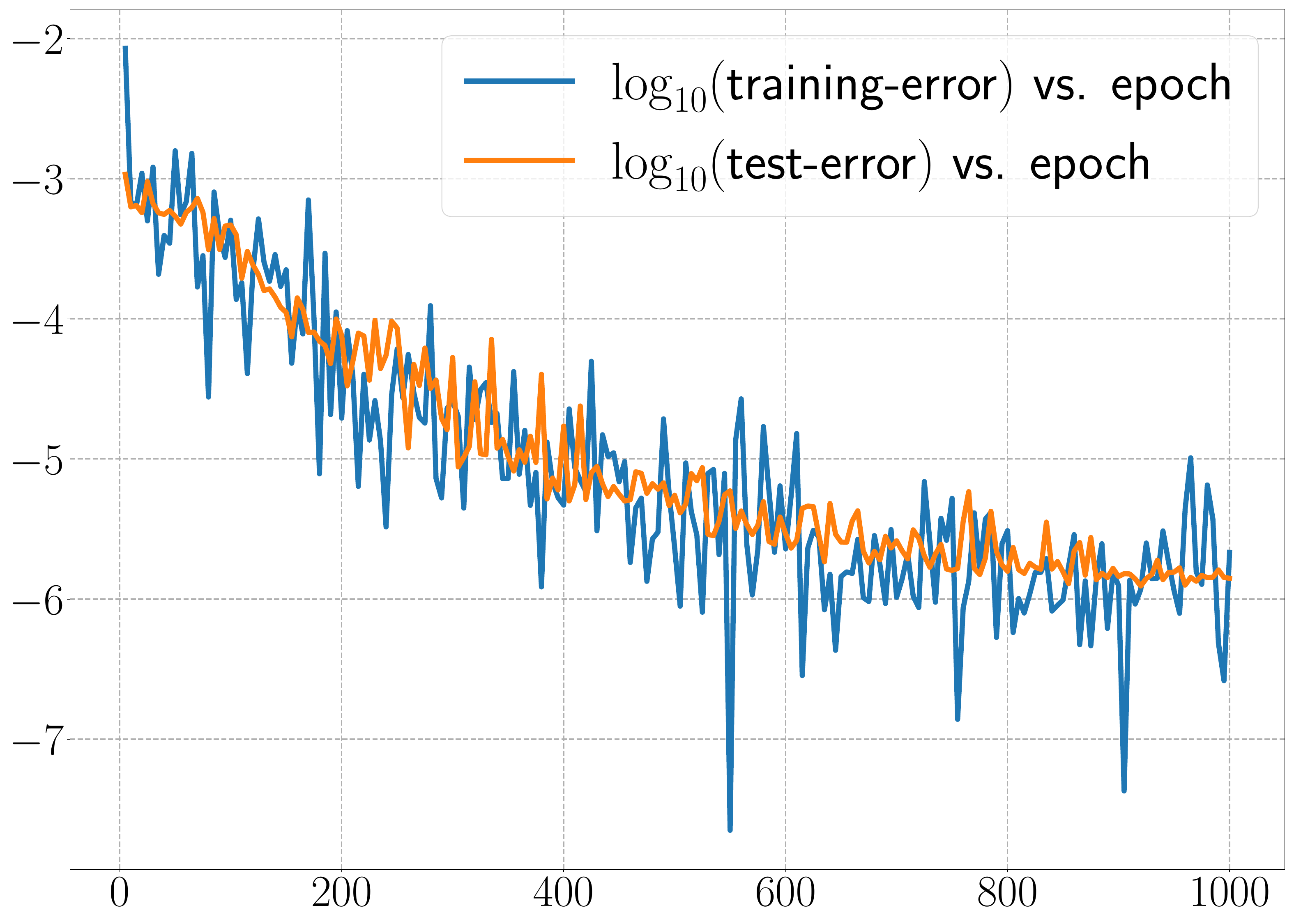}
    \caption{Errors (in MSE).}
    	\label{fig:2D-3}
\end{minipage}
\end{figure}

\begin{figure}
    \centering	
            \begin{subfigure}[c]{0.312\textwidth}
    \centering            \includegraphics[width=0.75859598055\textwidth]{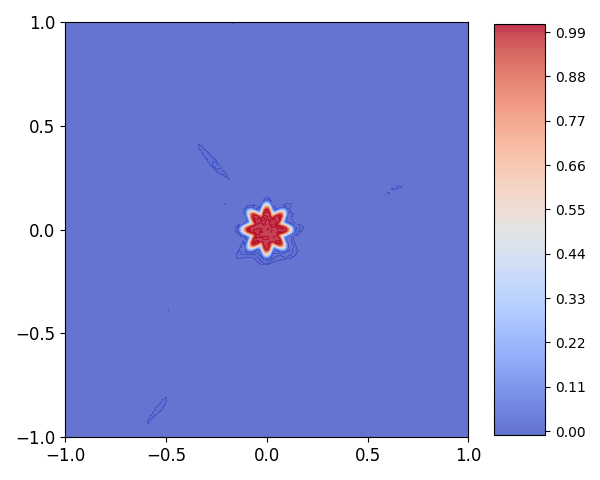}
    \subcaption{Network.}
    \end{subfigure}
    \hfill
    \begin{subfigure}[c]{0.312\textwidth}
    \centering            \includegraphics[width=0.75859598055\textwidth]{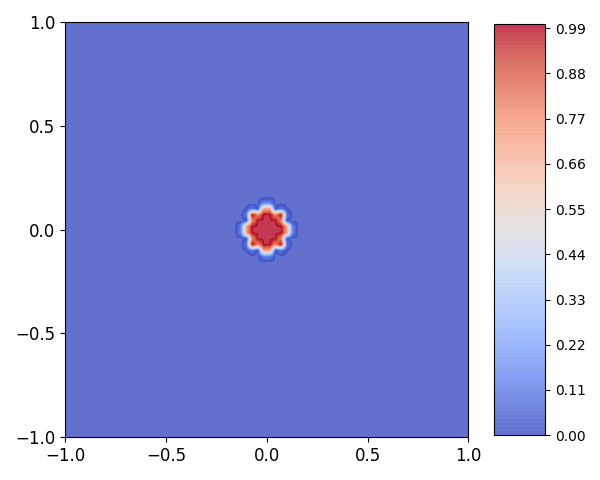}
    \subcaption{Interpolation.}
    \end{subfigure}
    \hfill 
 \begin{subfigure}[c]{0.312\textwidth}
    \centering            \includegraphics[width=0.75859598055\textwidth]{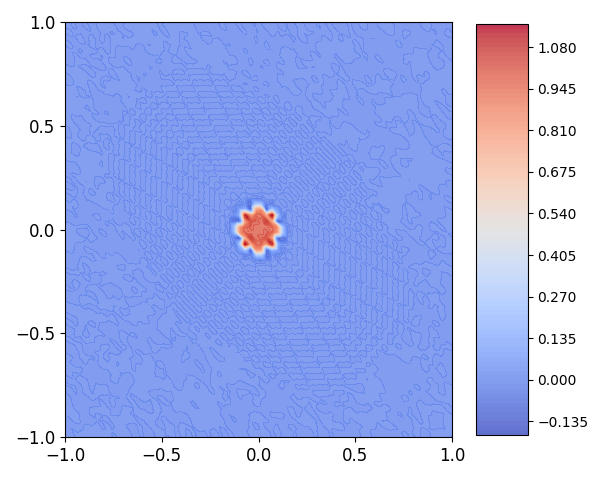}
    \subcaption{FEM.}
    \end{subfigure}\hfill
    \begin{subfigure}[c]{0.32\textwidth}
    \centering            \includegraphics[width=0.75859598055\textwidth]{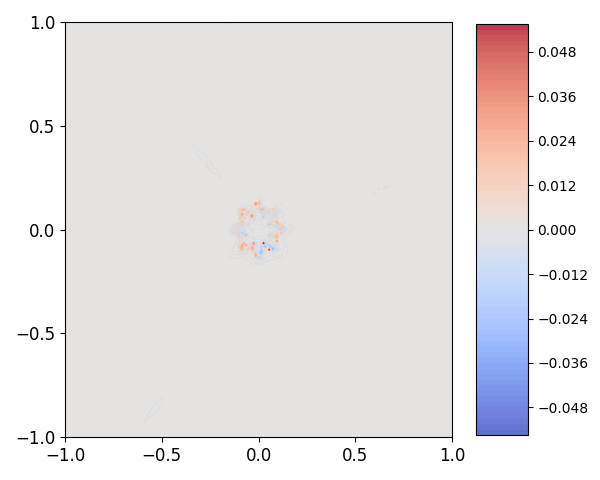}
    \subcaption{Network (error).}
    \end{subfigure}                
    \hfill
    \begin{subfigure}[c]{0.32\textwidth}
    \centering            \includegraphics[width=0.75859598055\textwidth]{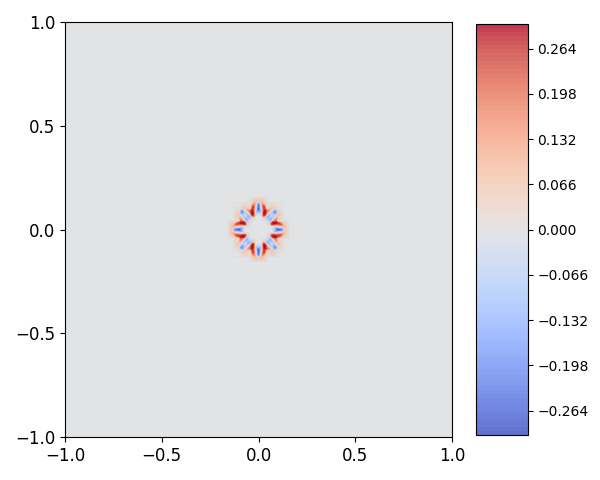}
    \subcaption{Interpolation (error).}
    
    \end{subfigure}     
    \hfill
    \begin{subfigure}[c]{0.32\textwidth}
    \centering            \includegraphics[width=0.75859598055\textwidth]{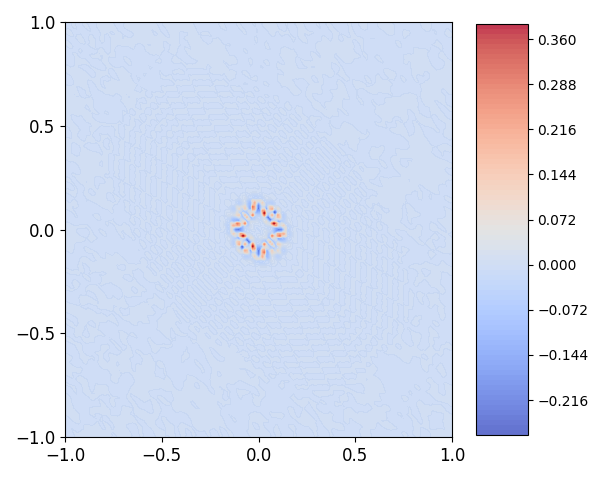}
    \subcaption{FEM (error).}    
    \end{subfigure}
    \caption{Comparison among different approximations using MMNN, interpolation, and least square FEM. The interpolation and FEM are all based on a $72\times 72=5184$ uniform grid. MMNN has $(100+1) \times 10 \times (6-1) + (100+1) = 5151$ free parameters. The maximum error is approximately $0.05$ for MMNN, $0.31$ for interpolation, and $0.38$ for FEM. The corresponding MSE errors are $0.85 \times 10^{-6}$, $1.95 \times 10^{-4}$, and $1.45 \times 10^{-4}$, respectively.}
\label{fig:comparison}
\end{figure}

\subsection{Highly oscillatory functions}\label{sec:oscillatory}

Globally oscillatory functions with significant high-frequency components can not be approximated well by a shallow network when a global bounded activation function of the form $\sigma(\bmW\cdot \bmx-\bmb)$, such as \texttt{ReLU}, is used. Due to almost orthogonality or high decorrelation (in terms of the inner product) between $\sigma(\bmW\cdot \bmx-\bmb)$ and oscillatory functions with high likelihood (in terms of a random choice of $(\bmW, \bmb)$), the set of parameters that can render a good approximation, namely the \emph{Rashomon set}~\cite{semenova2022existence}, becomes smaller and smaller (in terms of relative measure) and hence harder and harder to find as the target function becomes more and more oscillatory (see~\cite{ZZZZ-23}). Although this difficulty can be alleviated by complexity decomposition using MMNN as shown in Section~\ref{sec:MMNN}, it still requires a larger network in terms of width, rank, and layers and more training. Here we limit our tests to oscillatory functions in one-dimensional and two-dimensional cases due to the dramatic increase of complexity with dimensions, or the curse of dimensions, in general. 

We again start with a one-dimensional example, $f(x)=\sin(50\pi x), x\in[-1,1]$. An MMNN of size $(800,40,15)$ produces a good approximation of this highly oscillatory function, as illustrated by the error plot in Figure~\ref{fig:50pi-error}, with a smaller learning rate and a longer training process compared to previous examples with localized fine features. Due to the significant depth, we consider using ResMMNN as discussed in Section~\ref{sec:MMNN:architecture}. For this test, a total of $1000$ uniformly sampled points in $[-1,1]$ are used with a mini-batch size of $100$ and a learning rate of $10^{-4}\times 0.9^{\lfloor k/800 \rfloor}$, where $k=1,2,\cdots,40000$ is the epoch number. Also, an interesting learning dynamics for Adam is observed from Figure~\ref{fig:sin:50pi:NN}. In the beginning, nothing seems to happen until about epoch 3600 when learning starts from the boundary. Then more and more features are captured from the boundary to the inside gradually. Eventually, all features are captured and then fine-tuned together to improve the overall approximation.

\begin{figure}
\begin{minipage}[b]{0.59\linewidth}
    \centering	
    \begin{subfigure}[c]{0.32\textwidth}
        \centering       \includegraphics[width=0.99\textwidth]{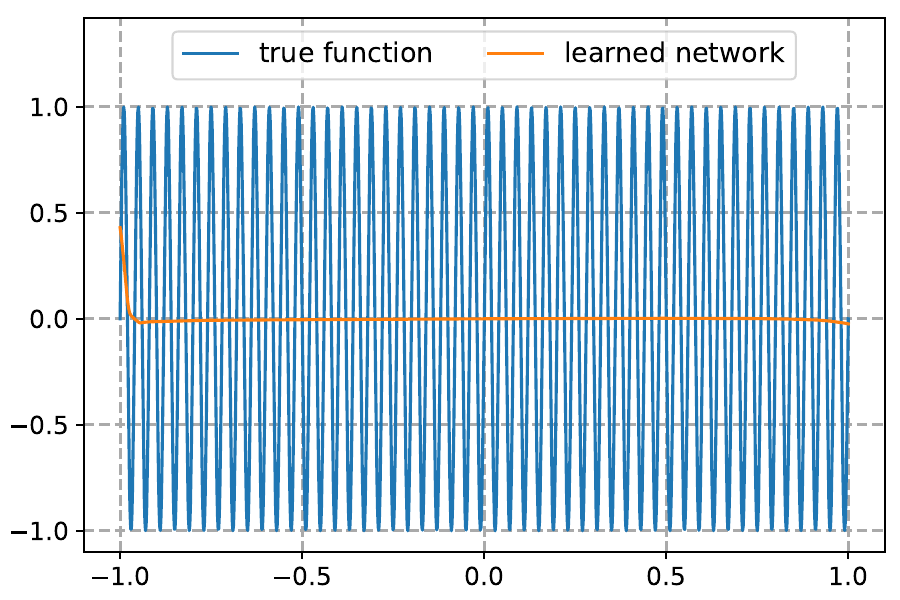}
        \subcaption{Epoch 3600.}
    \end{subfigure}
    \hfill
    \begin{subfigure}[c]{0.32\textwidth}
        \centering
        \includegraphics[width=0.99\textwidth]{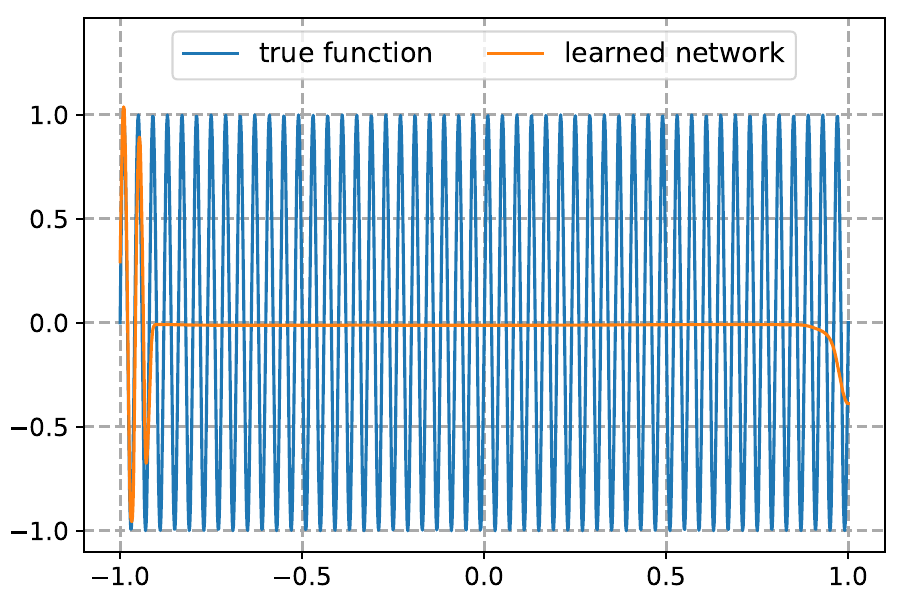}
        \subcaption{Epoch 3800.}
    \end{subfigure}
    \hfill
    \begin{subfigure}[c]{0.32\textwidth}
        \centering
        \includegraphics[width=0.99\textwidth]{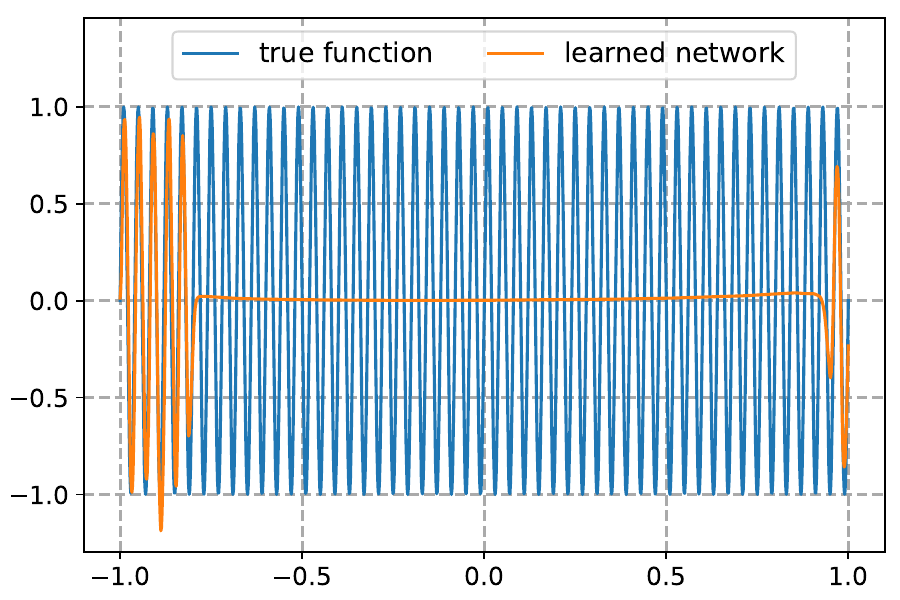}
        \subcaption{Epoch 4200.}
    \end{subfigure}
\\  
    \begin{subfigure}[c]{0.32\textwidth}
        \centering
        \includegraphics[width=0.99\textwidth]{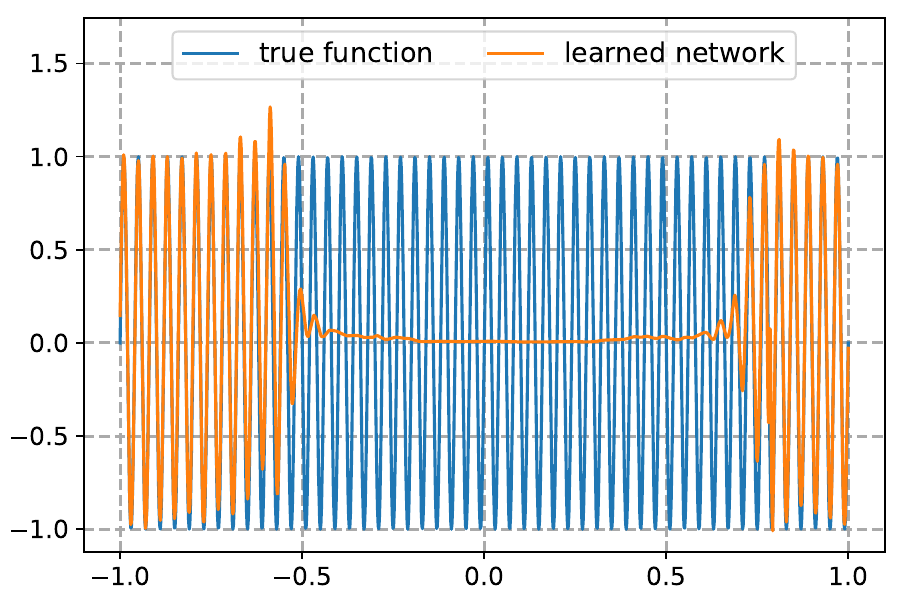}
        \subcaption{Epoch 5000.}
    \end{subfigure}
    \hfill
    \begin{subfigure}[c]{0.32\textwidth}
        \centering
        \includegraphics[width=0.99\textwidth]{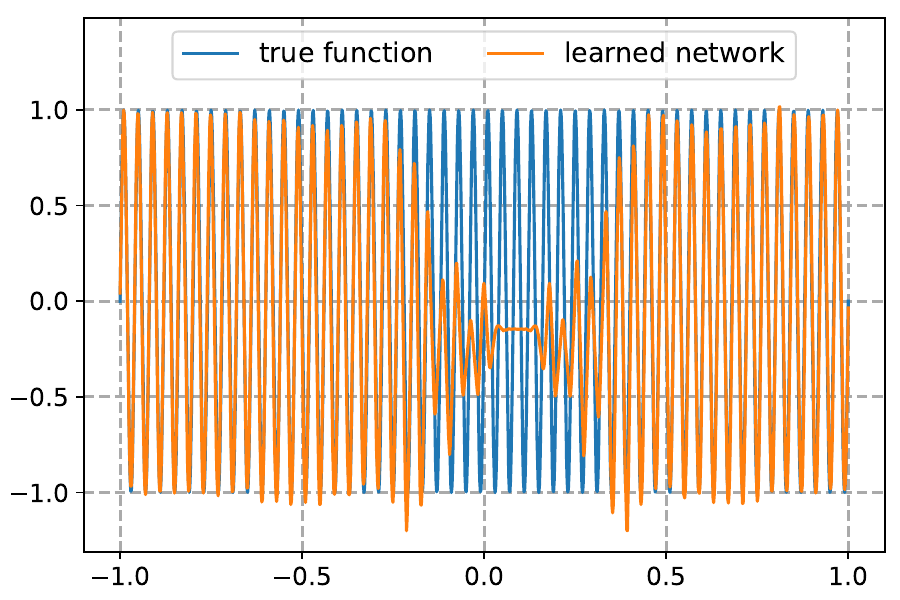}
        \subcaption{Epoch 10000.}
    \end{subfigure}
    \hfill
    \begin{subfigure}[c]{0.32\textwidth}
        \centering
        \includegraphics[width=0.99\textwidth]{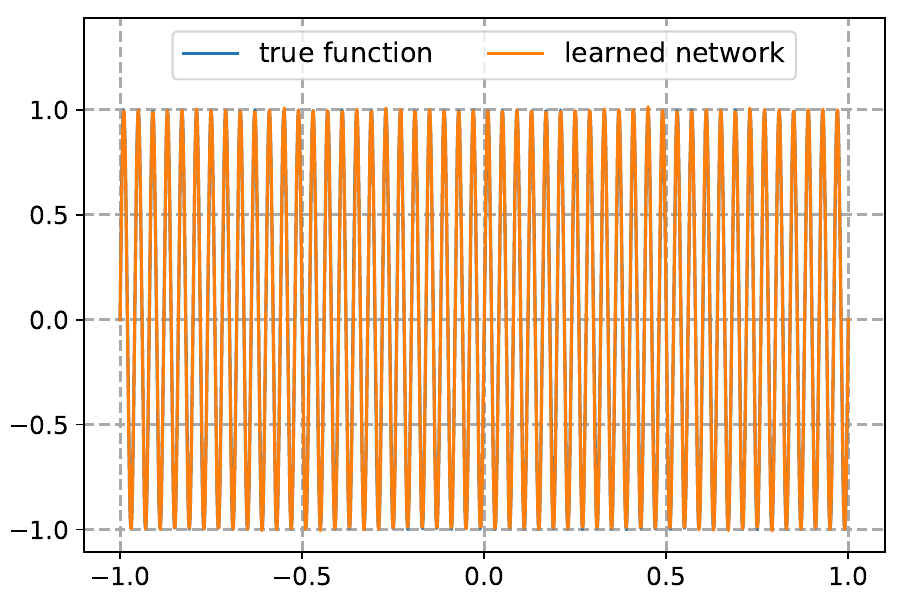}
        \subcaption{Epoch 40000.}
    \end{subfigure}
    \caption{Illustrations of the training process.}
    	\label{fig:sin:50pi:NN}
    \end{minipage}
    \hfill
    \begin{minipage}[b]{0.36\linewidth}
    \centering	
    \includegraphics[width=0.999\textwidth]{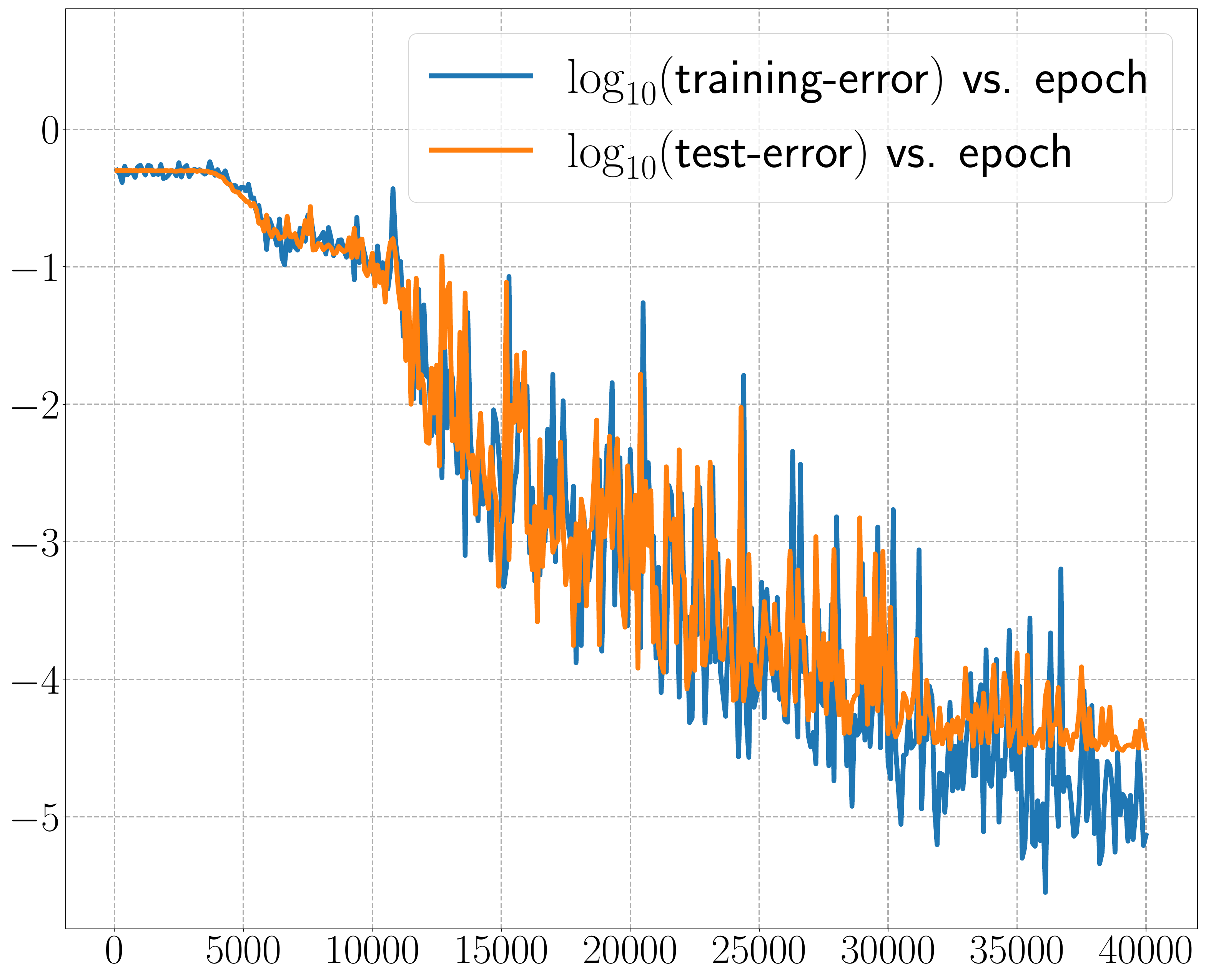}
    \caption{Training and test errors (in MSE) vs. epoch.}
    	\label{fig:50pi-error}
    \end{minipage}
\end{figure}



Next, we consider a two-dimensional target function of the following form:
\begin{equation*}
	f_{s}(x_1,x_2)= \sum_{i=1}^2\sum_{j=1}^2 a_{ij} \sin(s b_i x_i+s c_{ij} x_i x_j) \cos (s b_j x_j + s d_{ij} x_i^2),
\end{equation*}
 where
\begin{equation*}
	(a_{i,j})=	\begin{bmatrix*}
		0.3 & 0.2 \\ 0.2 & 0.3
	\end{bmatrix*}, \qquad 
	(b_{i})=	\begin{bmatrix*}
		2\pi \\ 4\pi
	\end{bmatrix*},
	\qquad
	(c_{i,j})=	\begin{bmatrix*}
		2\pi & 4\pi \\ 8\pi & 4\pi
	\end{bmatrix*},  \quad 
	(d_{i,j})=	\begin{bmatrix*}
		4\pi & 6\pi \\ 8\pi & 6\pi
	\end{bmatrix*}.
\end{equation*}
In our test, we choose $s=3$ to ensure the function exhibits significant oscillations and contains diverse Fourier modes as illustrated by Figure~\ref{fig:2D-oscillation}. Given the complexity of the function, we employ an MMNN with size $(600, 30, 15)$. Again, ResMMNN is used due to the depth. For this test, a total of $400^2$ data are sampled on a uniform grid in $[-1,1]^2$ with a mini-batch size of $1000$ and a learning rate of $10^{-3}\times 0.9^{\lfloor k/40 \rfloor}$, where $k=1,2,\cdots,2000$ is the epoch number. The training process is illustrated by Figure~\ref{fig:2D-oscillation-difference:and:NN}. Figure~\ref{fig:2D-oscillation-error} shows $\log$-error plot.

\begin{figure}
\begin{minipage}[b]{0.63597166\linewidth}
    \centering	
    \begin{subfigure}[c]{0.48736\textwidth}
    \centering            \includegraphics[height=0.7850\textwidth]{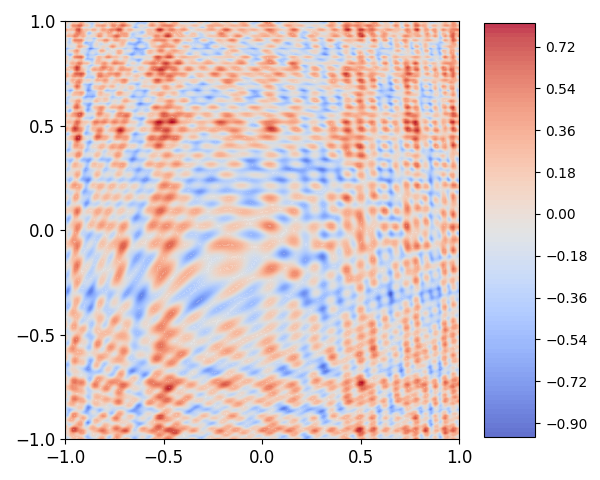}
    \end{subfigure}
    \hfill
    \begin{subfigure}[c]{0.48736\textwidth}
    \centering            \includegraphics[height=0.7850\textwidth]{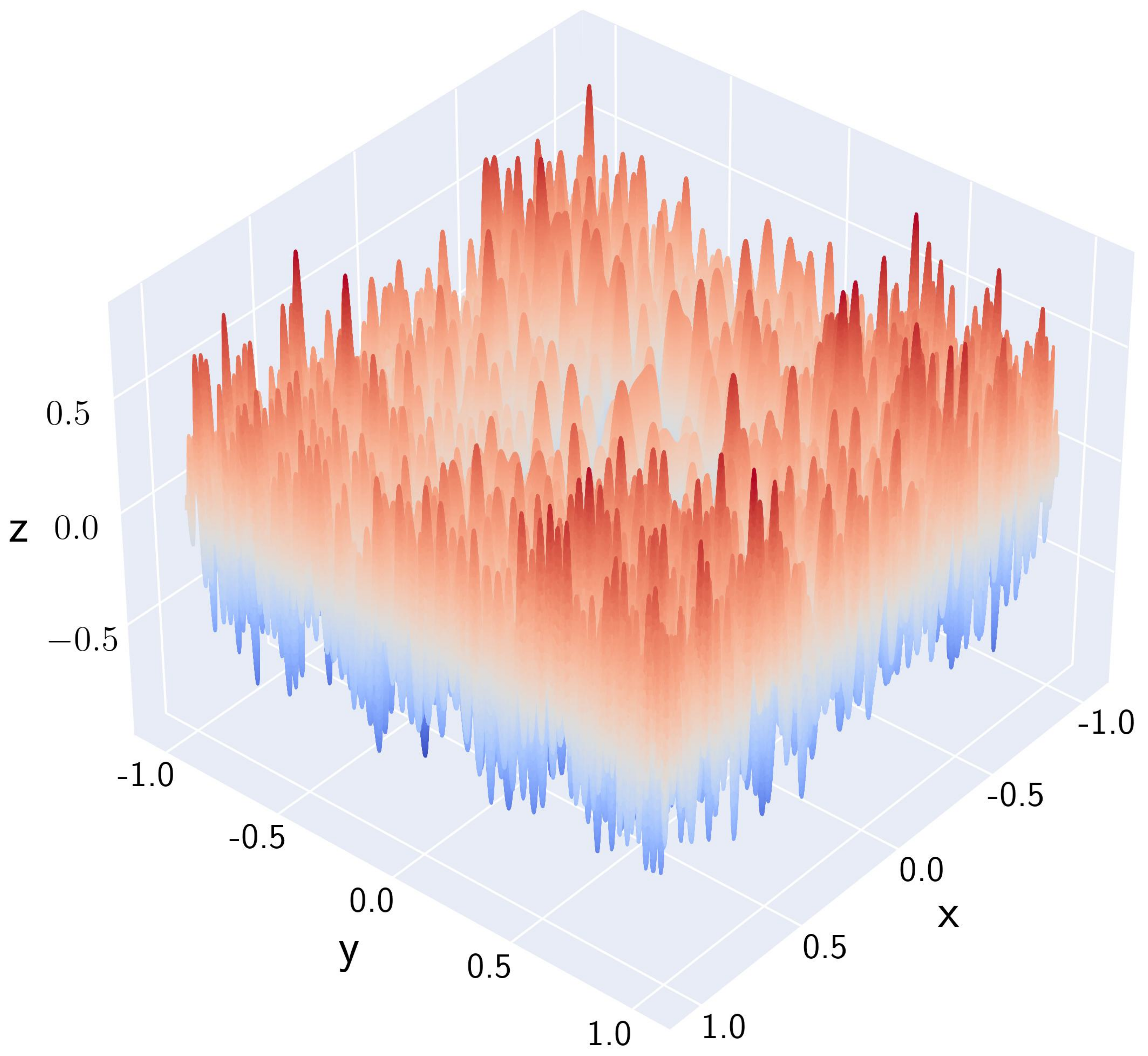}
    \end{subfigure}
    \caption{Illustrations of the target function.}
    	\label{fig:2D-oscillation}
\end{minipage}
\hfill
\begin{minipage}[b]{0.318\linewidth}
    \centering	
    \includegraphics[width=0.9985\textwidth]{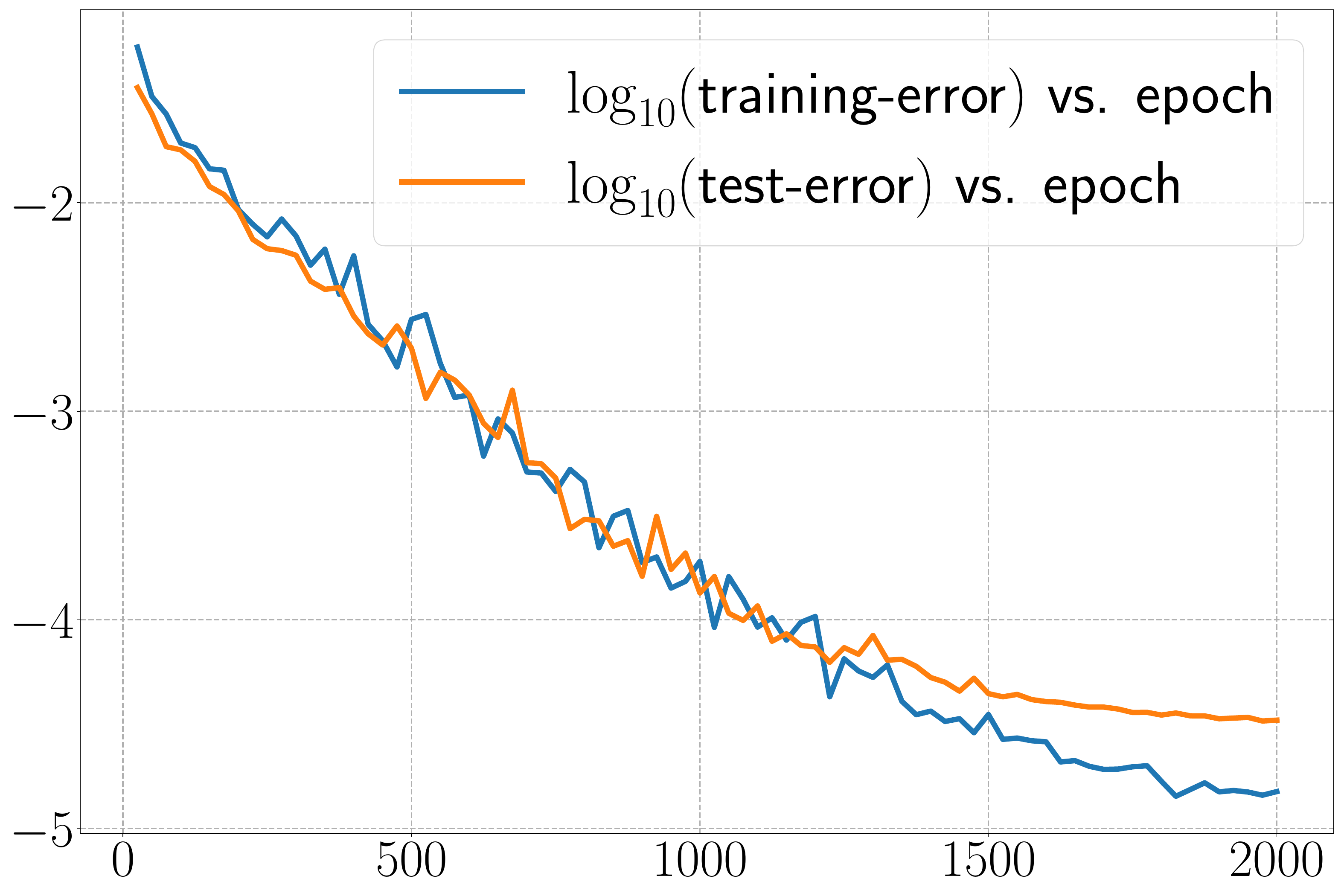}
    \caption{Training and test errors  (in MSE) vs. epoch.}
    	\label{fig:2D-oscillation-error}
\end{minipage}
\end{figure}



\begin{figure}
    \centering
    \begin{subfigure}[c]{0.234\textwidth}
        \centering
        \includegraphics[width=0.998055\textwidth]{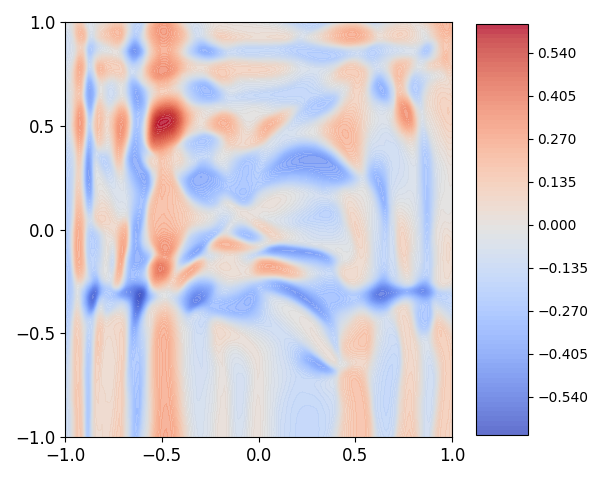}
        \subcaption{Epoch 25.}
    \end{subfigure}
    \hfill
    \begin{subfigure}[c]{0.234\textwidth}
        \centering
        \includegraphics[width=0.998055\textwidth]{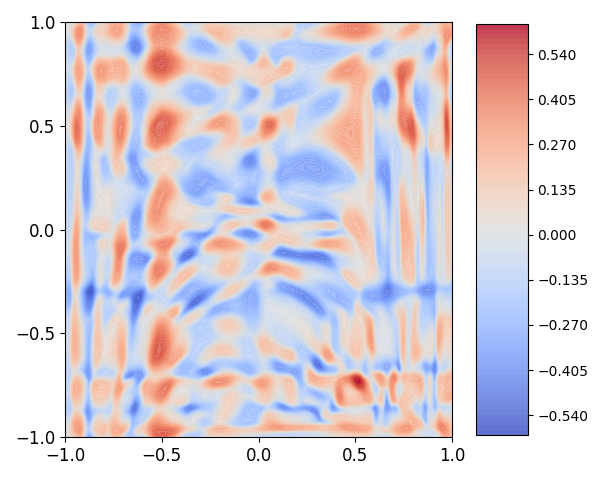}
        \subcaption{Epoch 50.}
    \end{subfigure}
    \hfill
    \begin{subfigure}[c]{0.234\textwidth}
        \centering
        \includegraphics[width=0.998055\textwidth]{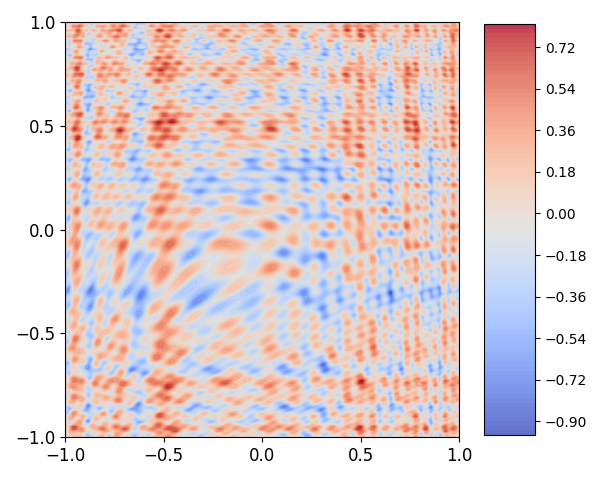}
        \subcaption{Epoch 1000.}
    \end{subfigure}
    \hfill
    \begin{subfigure}[c]{0.234\textwidth}
        \centering
        \includegraphics[width=0.998055\textwidth]{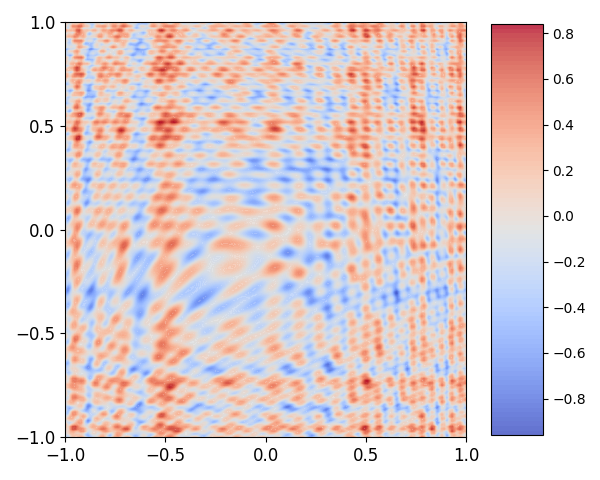}
        \subcaption{Epoch 2000.}
    \end{subfigure}
\\
    \begin{subfigure}[c]{0.234\textwidth}
        \centering
        \includegraphics[width=0.998055\textwidth]{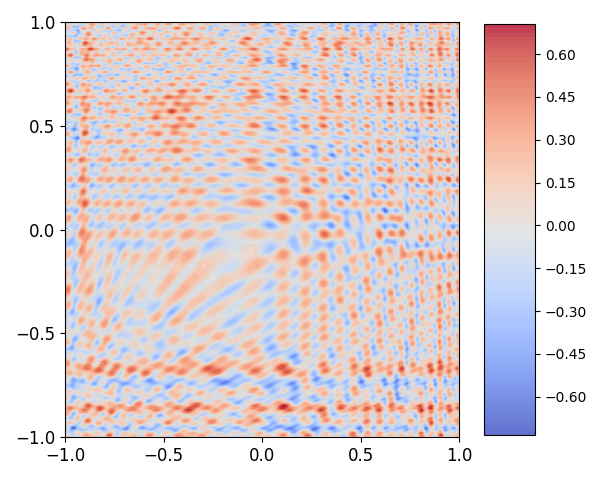}
        \subcaption{Epoch 25.}
    \end{subfigure}
    \hfill
    \begin{subfigure}[c]{0.234\textwidth}
        \centering
        \includegraphics[width=0.998055\textwidth]{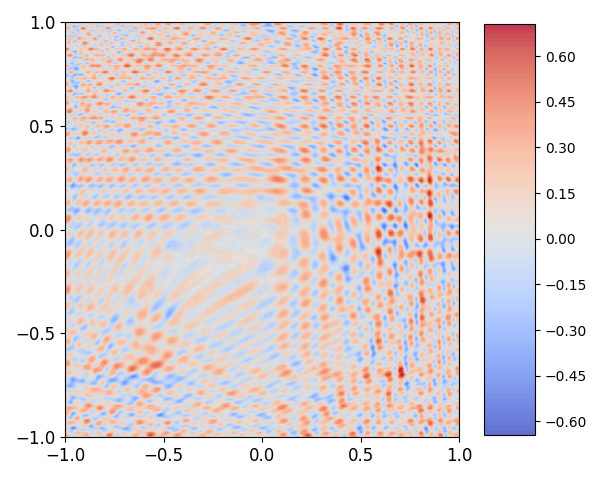}
        \subcaption{Epoch 50.}
    \end{subfigure}
    \hfill
    \begin{subfigure}[c]{0.234\textwidth}
        \centering
        \includegraphics[width=0.998055\textwidth]{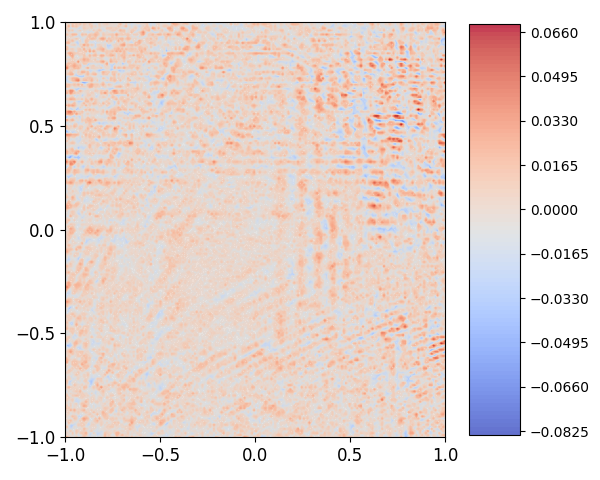}
        \subcaption{Epoch 1000.}
    \end{subfigure}
    \hfill
    \begin{subfigure}[c]{0.234\textwidth}
        \centering
        \includegraphics[width=0.998055\textwidth]{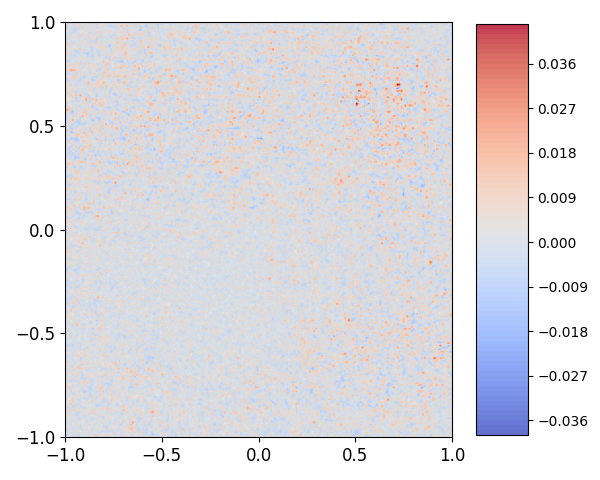}
        \subcaption{Epoch 2000.}
    \end{subfigure}
    
    \caption{The top row: the learned neural network;  the bottom row: error.
    }
    \label{fig:2D-oscillation-difference:and:NN}
\end{figure}


We trained the same function using identical network settings, except we limited the domain of interest to a unit disc. We sampled $452^2$ data points uniformly distributed over the $[-1,1]^2$ area, then filtered to retain only those points that fall within the unit disk, totaling approximately $159692$ $(\approx 400^2)$ samples. As illustrated in Figure~\ref{fig:2D-oscillation-unitBall}, our network successfully learned the target function in the disc with no adjustments or modifications. This test highlights the network's flexibility for domain geometry, an advantage over traditional mesh or grid-based methods, especially in higher dimensions.

 \begin{figure}
    \centering	
    \begin{subfigure}[c]{0.31246\textwidth}
    \centering            \includegraphics[height=0.6790\textwidth]{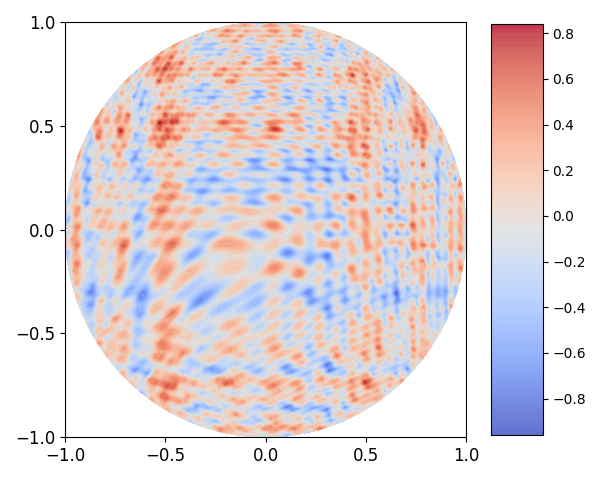}
    \subcaption{Approximation.}
    \end{subfigure}
    \hfill
    \begin{subfigure}[c]{0.31246\textwidth}
    \centering            \includegraphics[height=0.6790\textwidth]{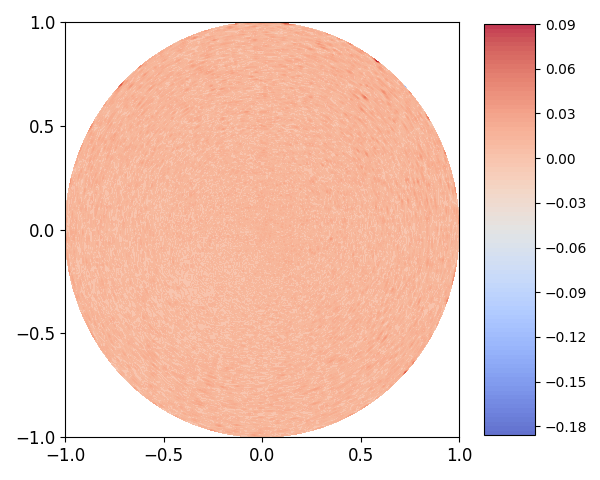}
    \subcaption{Error.}
    \end{subfigure}
    \hfill
    \begin{subfigure}[c]{0.31246\textwidth}
    \centering            \includegraphics[height=0.6675\textwidth]{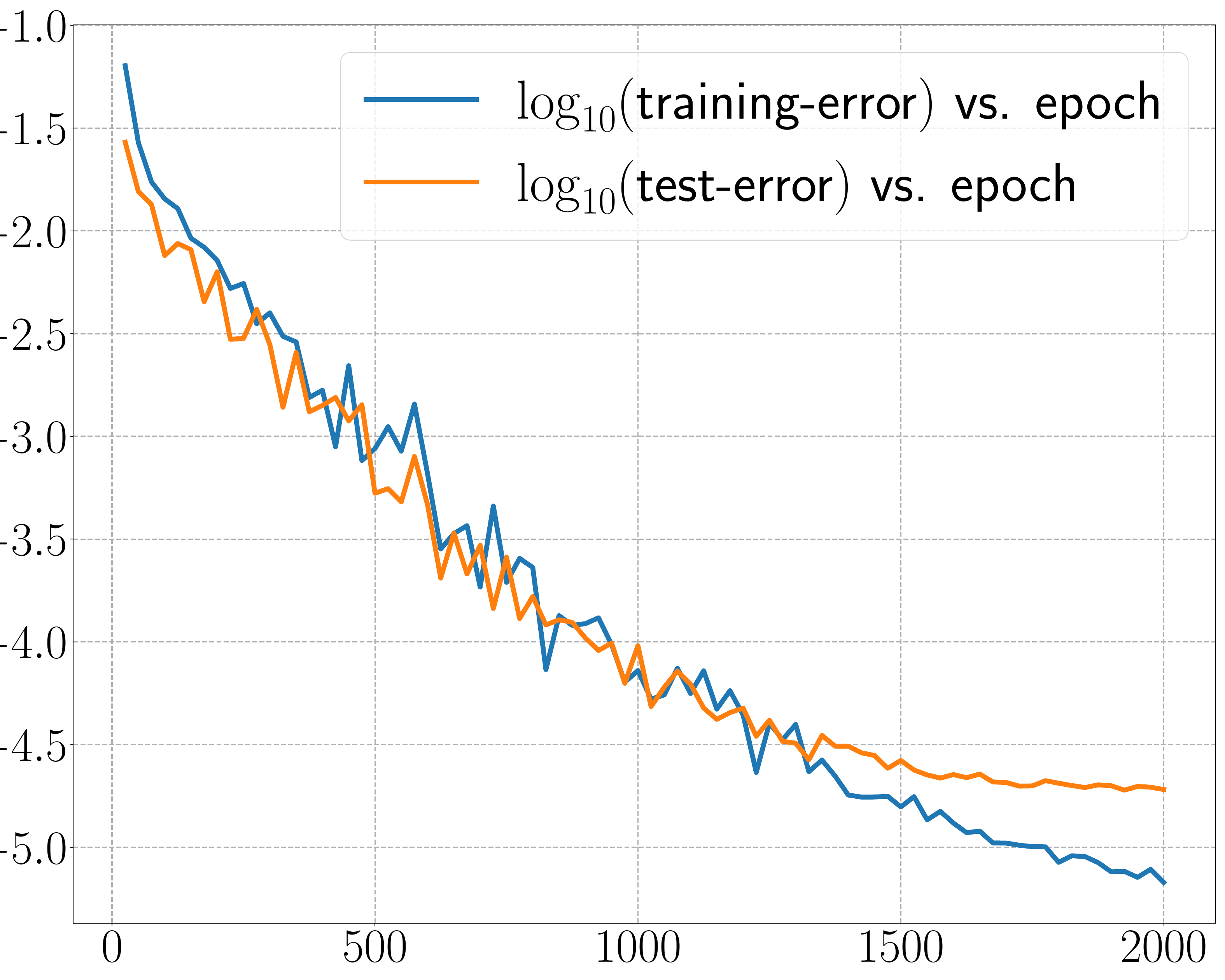}
    \subcaption{Errors (in MSE).
    }
    \end{subfigure}
\caption{MMNN approximation in a unit disk.}
    	\label{fig:2D-oscillation-unitBall}
\end{figure}

\subsection{Discontinuous functions with porous structures}
\label{sec:discontinuous:func:porous}

Earlier, we tested highly oscillatory functions that are still continuous. To be more inclusive in our experiments, we now consider piecewise smooth functions with porous structures and complicated interfaces. Two target functions are shown in Figure~\ref{fig:discont:funcs}(a,d). The first is a constant function with holes of various shapes removed (piecewise constant), while the second is based on the function in Figure~\ref{fig:vsFCNN:f2D} with holes introduced (piecewise smooth). We choose an MMNN of size (256, 12, 6), denoted MMNN1, and another of size (1024, 32, 6), denoted MMNN2, to learn these two functions, respectively. For training, we sample \(600^2\) data points on a uniform grid in \([-1,1]^2\), using a mini-batch size of 1000 and a learning rate of \(0.001 \times 0.9^{\lfloor k/20 \rfloor}\) for epochs \(k = 1, 2, \dots, 1600\).

\begin{figure}
    \centering	
    \begin{subfigure}[b]{0.294311\textwidth}
    \centering            \includegraphics[width=0.998055\textwidth]{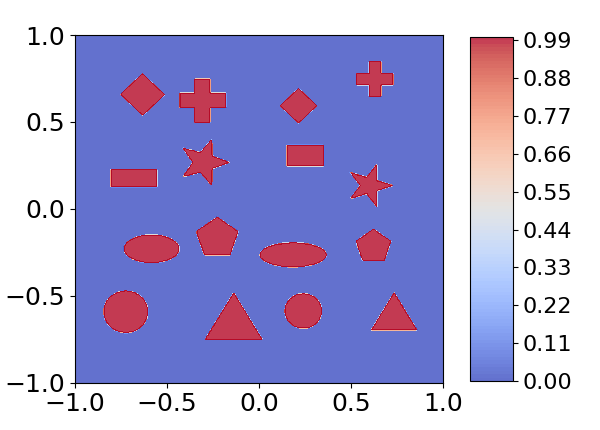}
    \subcaption{Truth.}
    \label{subfig:discont:fun1}
    \end{subfigure}
        \hfill
              \begin{subfigure}[b]{0.294311\textwidth}
    \centering            \includegraphics[width=0.985\textwidth]{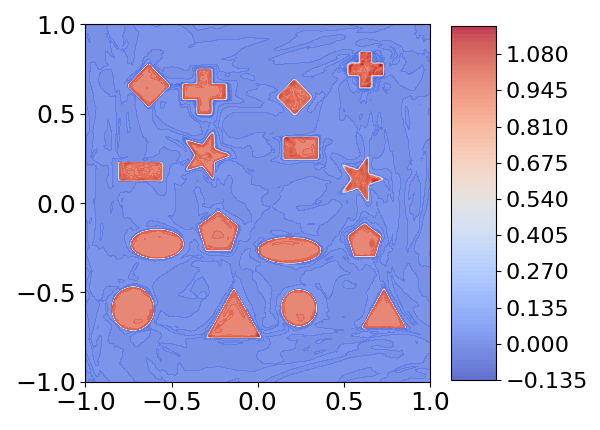}
    \subcaption{MMNN1.}
    \end{subfigure}
    \hfill
    \begin{subfigure}[b]{0.294311\textwidth}
    \centering            \includegraphics[width=0.985\textwidth]{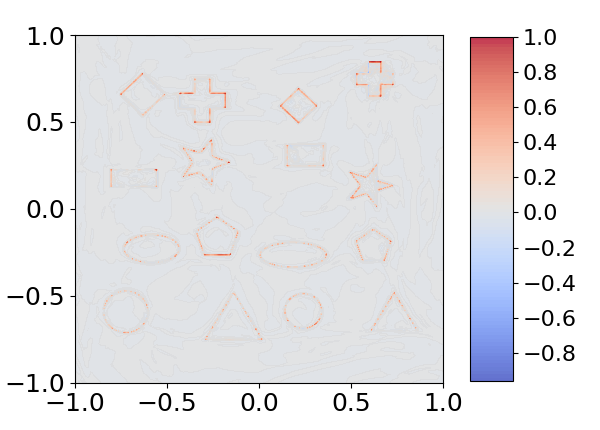}
    \subcaption{Error.}
    \end{subfigure}
    \\
         \begin{subfigure}[b]{0.294311\textwidth}
    \centering            \includegraphics[width=0.998055\textwidth]{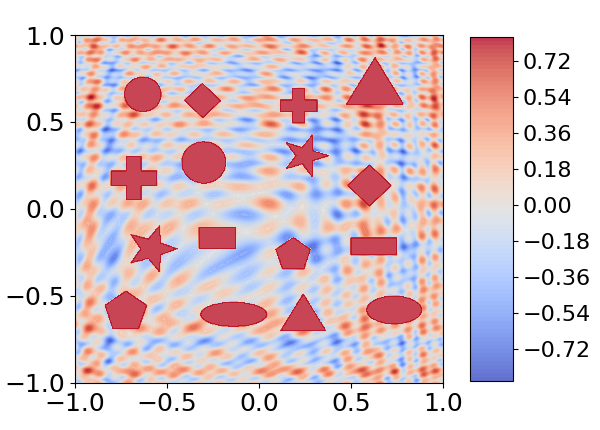}
    \subcaption{Truth.}
    \label{subfig:discont:fun2}
    \end{subfigure}
    \hfill
    \begin{subfigure}[b]{0.294311\textwidth}
    \centering            \includegraphics[width=0.998055\textwidth]{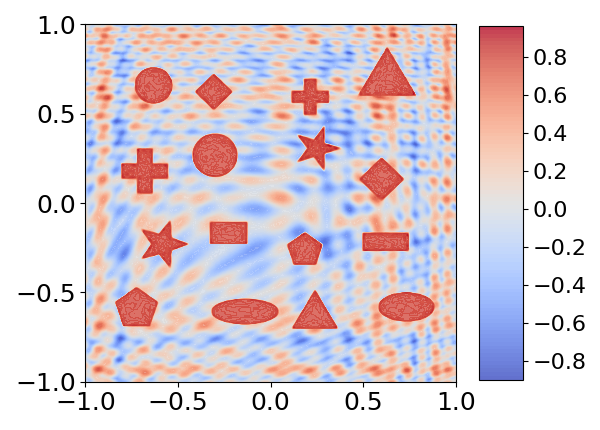}
    \subcaption{MMNN2.}
    
    \end{subfigure}
    \hfill
    \begin{subfigure}[b]{0.294311\textwidth}
    \centering            \includegraphics[width=0.998055\textwidth]{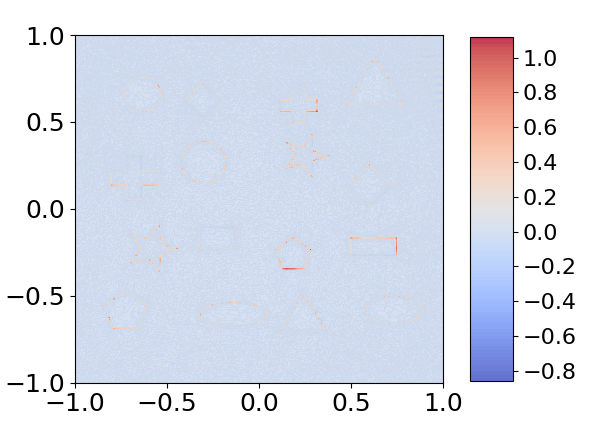}
    \subcaption{Error.}
    
    \end{subfigure}
    \caption{
   True functions, corresponding learned networks, and their differences (errors). The top and bottom rows correspond to two different target functions.}
  \label{fig:discont:funcs}
\end{figure}

\begin{table}[htbp!]
	\centering  
 \setlength{\tabcolsep}{0.68em} 
 \renewcommand{\arraystretch}{1.15}
\caption{
Test errors for two approximation results in Figure~\ref{fig:discont:funcs}.
}
\label{tab:discont:porous}
	\resizebox{0.948\textwidth}{!}{ 
		\begin{tabular}{ccccccccc} 
			\toprule
			    target function & network &  (width, rank, depth) & {\#parameters (trained / all)}  &      {test error (MSE)} 
    & test error (MAX)  \\
			\midrule
			 \rowcolor{mygray}
			Figure~\ref{fig:discont:funcs}(a) &  MMNN1  & (256, 12, 6) &  15677 / 33085 & 
$1.35 \times 10^{-3}$  &  $9.94 \times 10^{-1}$ \\	
   		Figure~\ref{fig:discont:funcs}(d) &	 MMNN2  & (1024, 32, 6) & 165025 / 337057 & 
$1.14 \times 10^{-3}$  &  $1.12\times 10^{0}$
   \\	
			\bottomrule
		\end{tabular} 
	}
\end{table} 

As shown in Figure~\ref{fig:discont:funcs}, MMNNs demonstrate an impressive ability to simultaneously localize and capture discontinuities, geometric features, and oscillatory behaviors. This indicates that MMNNs are adaptive in both spatial and frequency domains. 
While the errors presented in Table~\ref{tab:discont:porous} are somewhat larger, they are primarily concentrated near the discontinuous parts, as illustrated in Figure~\ref{fig:discont:funcs}, which is reasonable.

\subsection{Learning dynamics}\label{sec:dynamics}

Here,
we show some interesting learning dynamics observed during the training process. As the first example in Section~\ref{sec:oscillatory} and the following examples show, the training process not just learns from low frequency first but can also learn feature by feature, i.e., can be localized in both frequency domain and spatial domain. We believe this is due to the combination of MMNN's ``divide and conquer" ability and the Adam optimizer which utilizes momentum. More understanding is needed and will be studied in our future research.

We again start with a one-dimensional example, $f(x)=\sin\big(36\pi |x|^{1.5}\big), x\in[-1,1]$. 
An MMNN of size $(600,30,8)$ produces a good approximation of this highly oscillatory function, as illustrated by the error plot in Figure~\ref{fig:36pi:power1.5-error}. For this test, a total of $1000$ uniformly sampled points in $[-1,1]$ are used with a mini-batch size of $100$ and a learning rate of $10^{-3}\times 0.9^{\lfloor k/200 \rfloor}$, where $k=1,2,\cdots,10000$ is the epoch number. 
As illustrated in Figure~\ref{fig:36pi:power1.5-derivative}, the function is less oscillatory near 0. Therefore, we might anticipate that the network will initially learn the part near 0 and then feature by feature from the middle to the boundary. The experimental results presented in Figure~\ref{fig:sin:36pi:power:1.5:learning:process} agree with our expectations.

\begin{figure}[h]
\centering
\begin{minipage}[b]{0.4568\linewidth}
    \centering	    \includegraphics[width=0.59\textwidth]{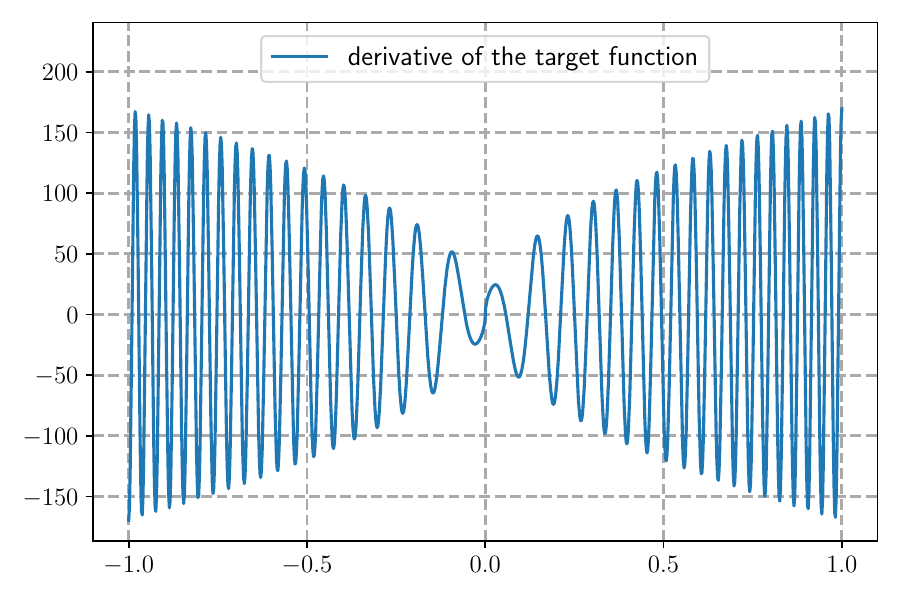}
    \caption{Derivative of $f$.
    }
    	\label{fig:36pi:power1.5-derivative}
    \end{minipage}
    \hspace{8pt}
    \begin{minipage}[b]{0.4568\linewidth}        \centering\includegraphics[width=0.59\textwidth]{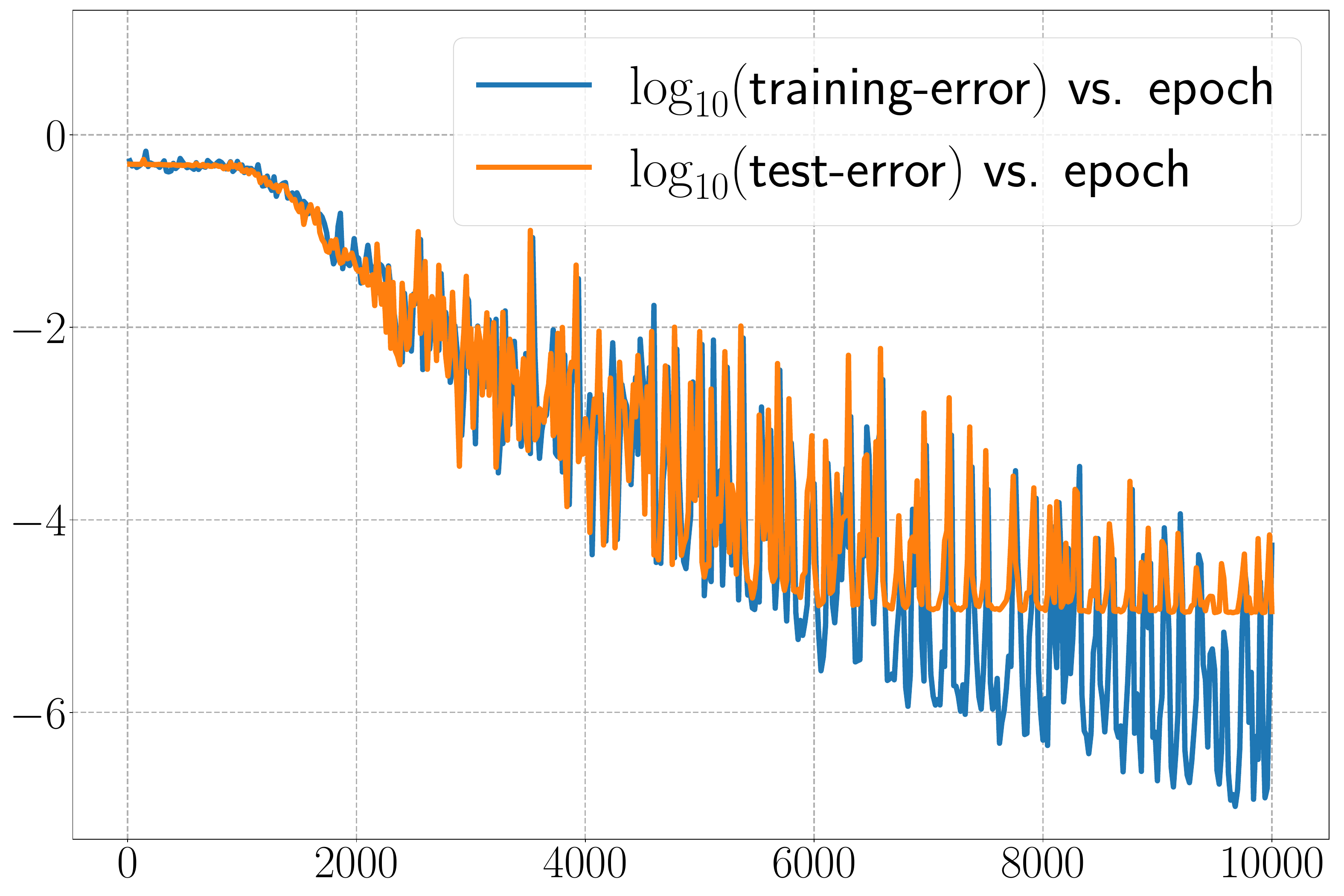}
    \caption{Errors (in MSE) vs. epoch.}
    	\label{fig:36pi:power1.5-error}
    \end{minipage}
\end{figure}

\begin{figure}
    \centering
    \begin{subfigure}[c]{0.2431\textwidth}
        \centering
        \includegraphics[width=0.99\textwidth]{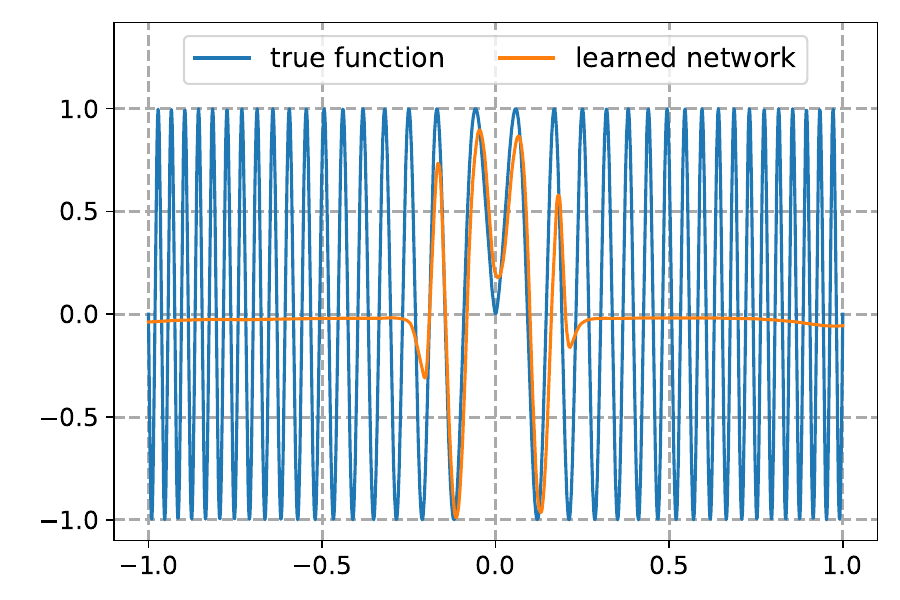}
        \subcaption{Epoch 400.}
    \end{subfigure}
    \hfill
    \begin{subfigure}[c]{0.2431\textwidth}
        \centering
        \includegraphics[width=0.99\textwidth]{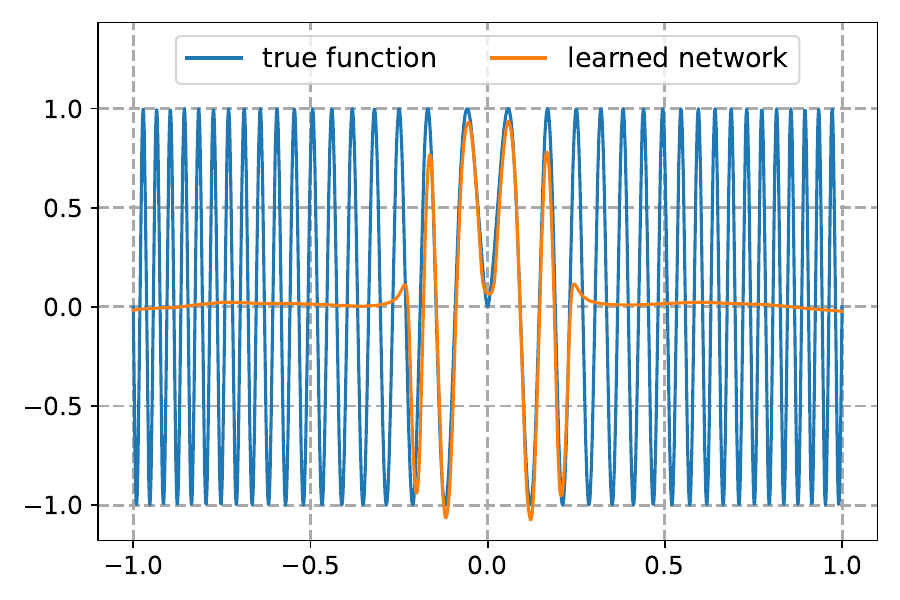}
        \subcaption{Epoch 500.}
    \end{subfigure}
    \hfill
    \begin{subfigure}[c]{0.2431\textwidth}
        \centering
        \includegraphics[width=0.99\textwidth]{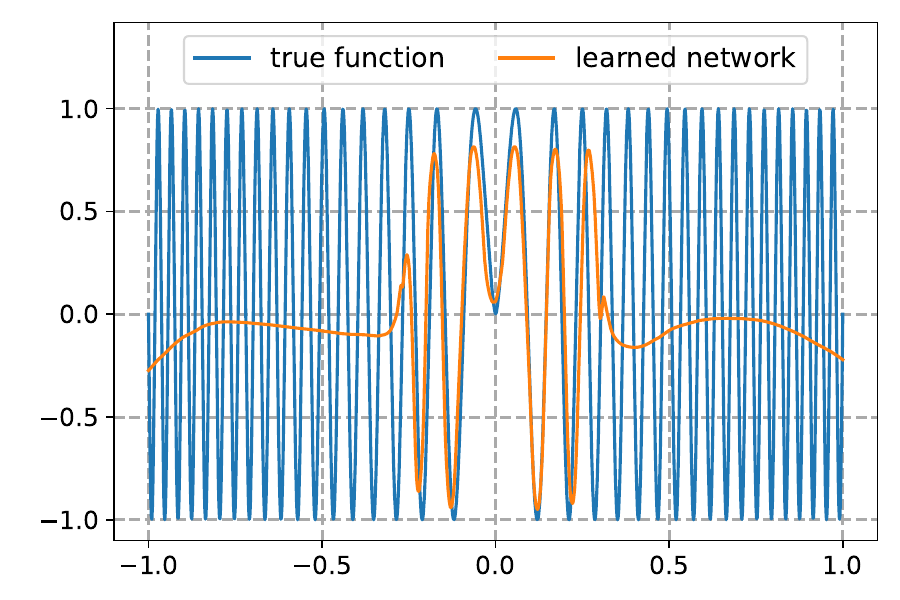}
        \subcaption{Epoch 600.}
    \end{subfigure}
    \hfill
    \begin{subfigure}[c]{0.2431\textwidth}
        \centering
        \includegraphics[width=0.99\textwidth]{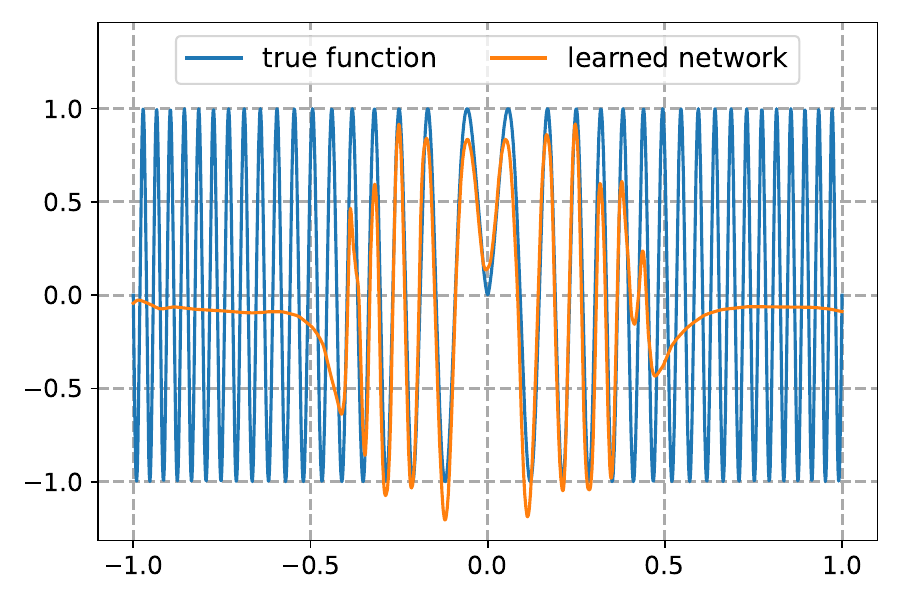}
        \subcaption{Epoch 700.}
    \end{subfigure}
\\
    \begin{subfigure}[c]{0.2431\textwidth}
        \centering
        \includegraphics[width=0.99\textwidth]{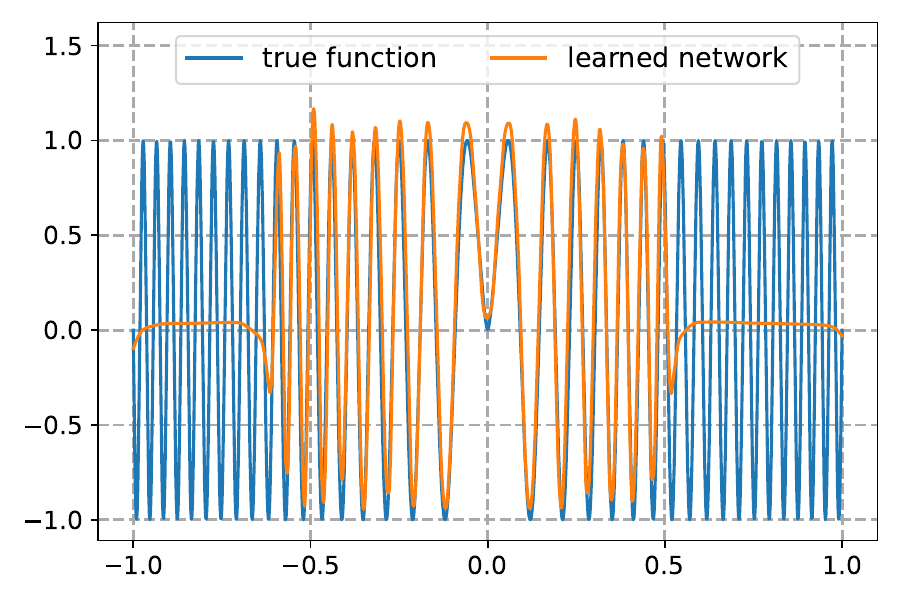}
        \subcaption{Epoch 1000.}
    \end{subfigure}
    \hfill
    \begin{subfigure}[c]{0.2431\textwidth}
        \centering
        \includegraphics[width=0.99\textwidth]{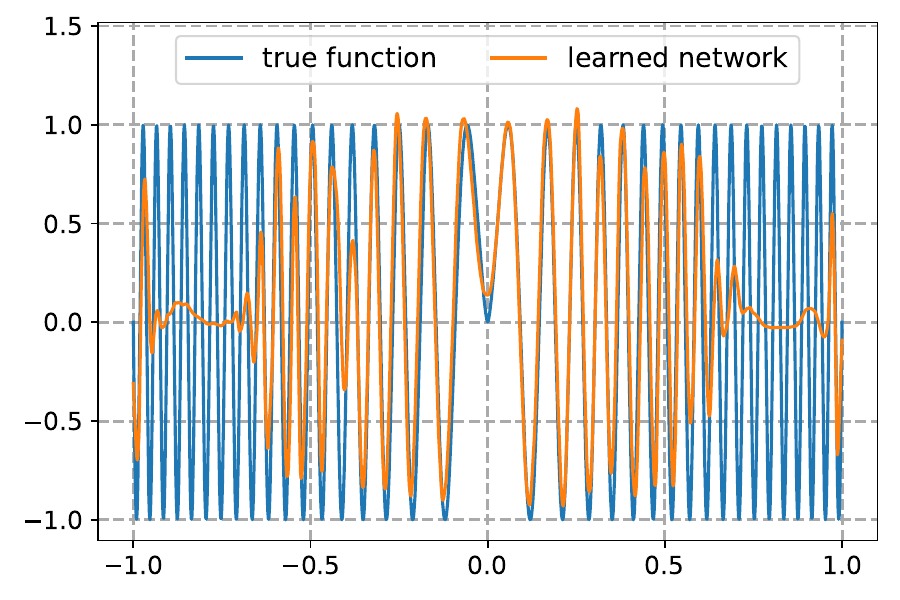}
        \subcaption{Epoch 1700.}
    \end{subfigure}
    \hfill
    \begin{subfigure}[c]{0.2431\textwidth}
        \centering
        \includegraphics[width=0.99\textwidth]{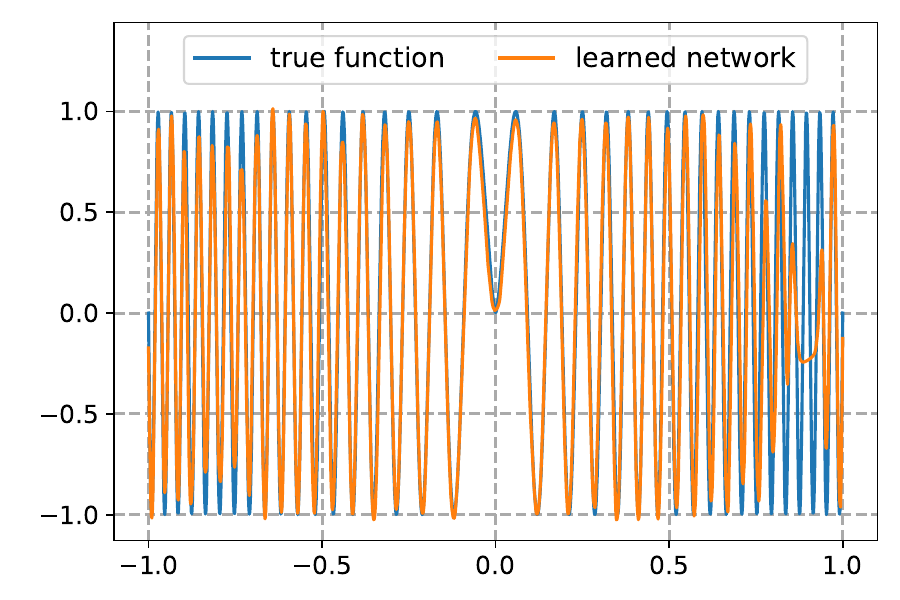}
        \subcaption{Epoch 2400.}
    \end{subfigure}
    \hfill
    \begin{subfigure}[c]{0.2431\textwidth}
        \centering
        \includegraphics[width=0.99\textwidth]{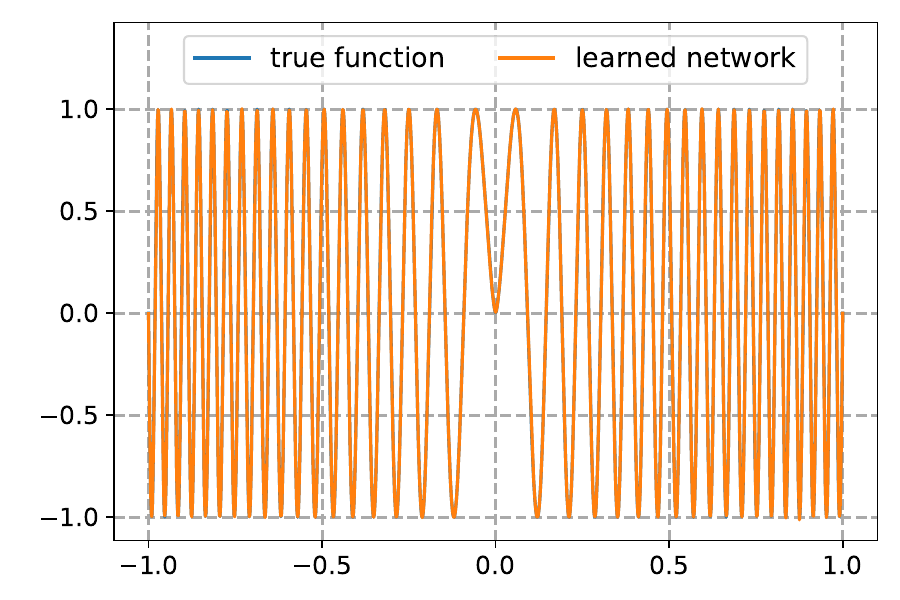}
        \subcaption{Epoch 10000.}
    \end{subfigure}
    
    \caption{Illustration of the training process.}
    \label{fig:sin:36pi:power:1.5:learning:process}
\end{figure}

Now we show an example of two-dimensional function $f(r, \theta)$ (see Figure~\ref{fig:LD:fun2D}) defined in polar coordinates $(r, \theta)$ as
\begin{equation*}
    f(r,\theta)=\begin{cases}
        0 & \tn{if  } 0.5+5\rho-5r\le 0,\\
        1 & \tn{if  } 0.5+5\rho-5r\ge 1,\\
        0.5+5\rho-5r & \tn{otherwise},\\
    \end{cases}
    \quad \tn{where}\quad \rho=0.5+0.1\cos(\pi^2 \theta^2).
\end{equation*}

\begin{figure}
\begin{minipage}[b]{0.6357604\linewidth}
    \centering	
    \begin{subfigure}[c]{0.458\textwidth}
    \centering            \includegraphics[height=0.795\textwidth]{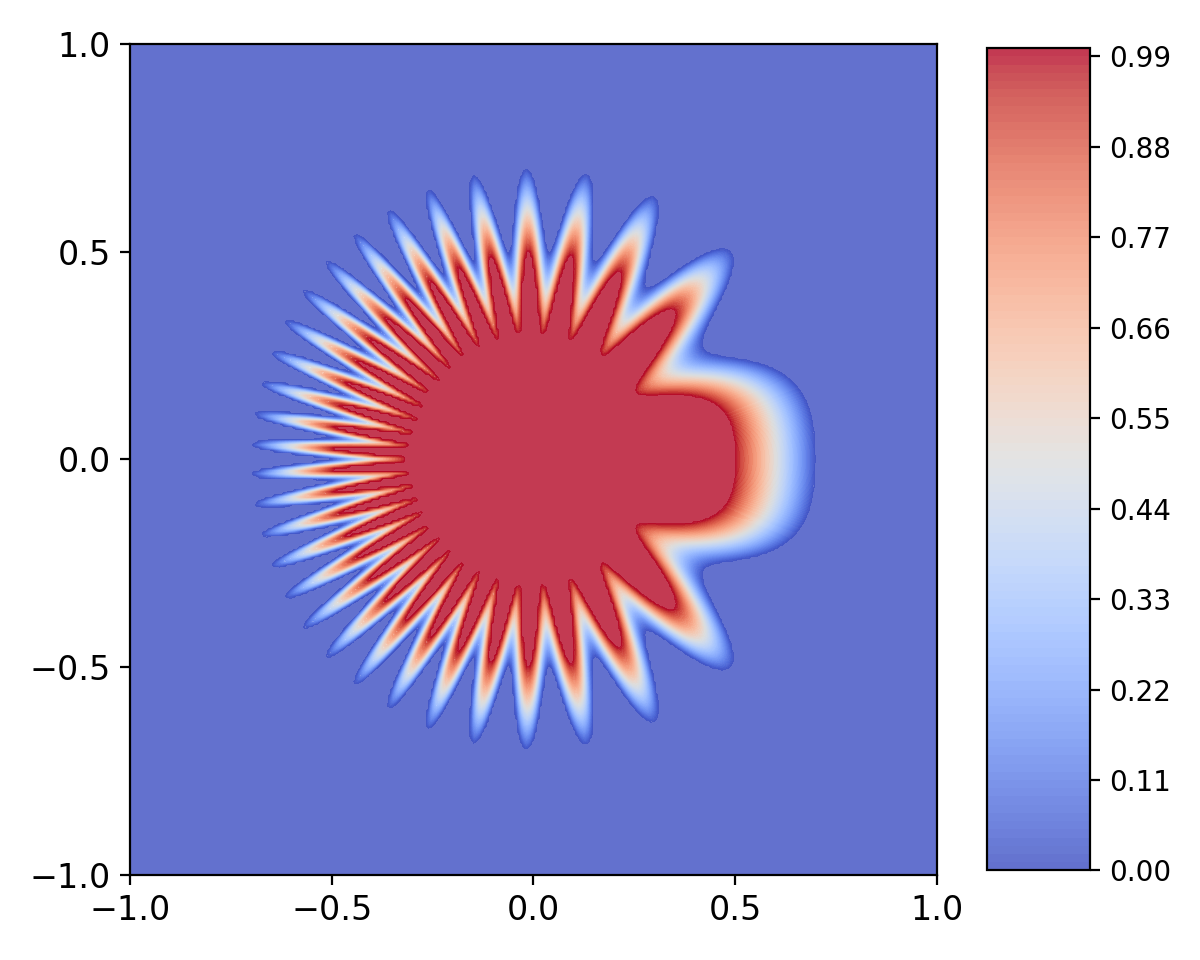}
    \end{subfigure}
    \hfill
    \begin{subfigure}[c]{0.458\textwidth}
    \centering            \includegraphics[height=0.7925\textwidth]{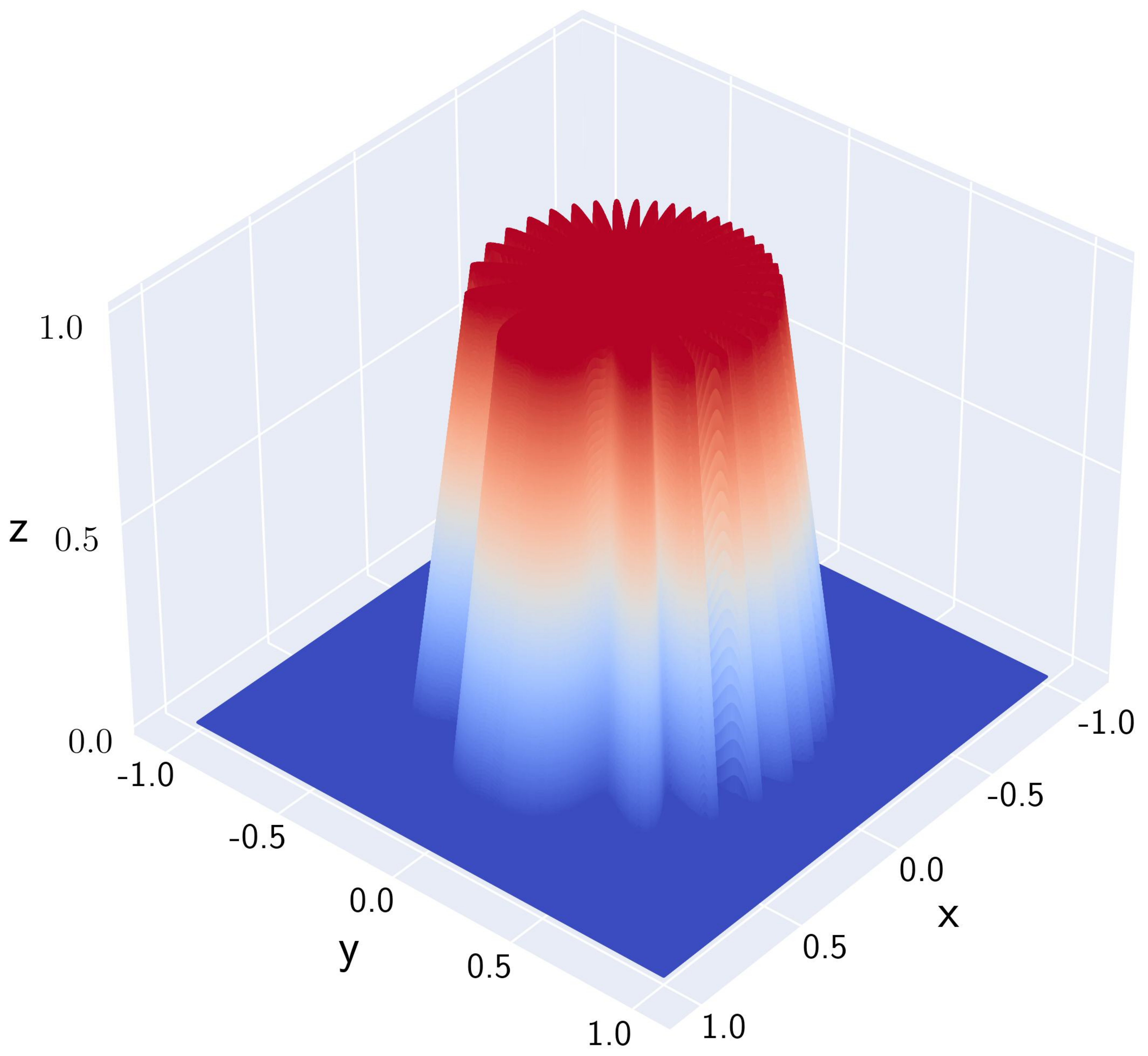}
    \end{subfigure}
    \caption{Illustration of the target function.}
    	\label{fig:LD:fun2D}
    \end{minipage}
    \hfill
\begin{minipage}[b]{0.318\linewidth}
    \centering	
    \includegraphics[width=0.95\textwidth]{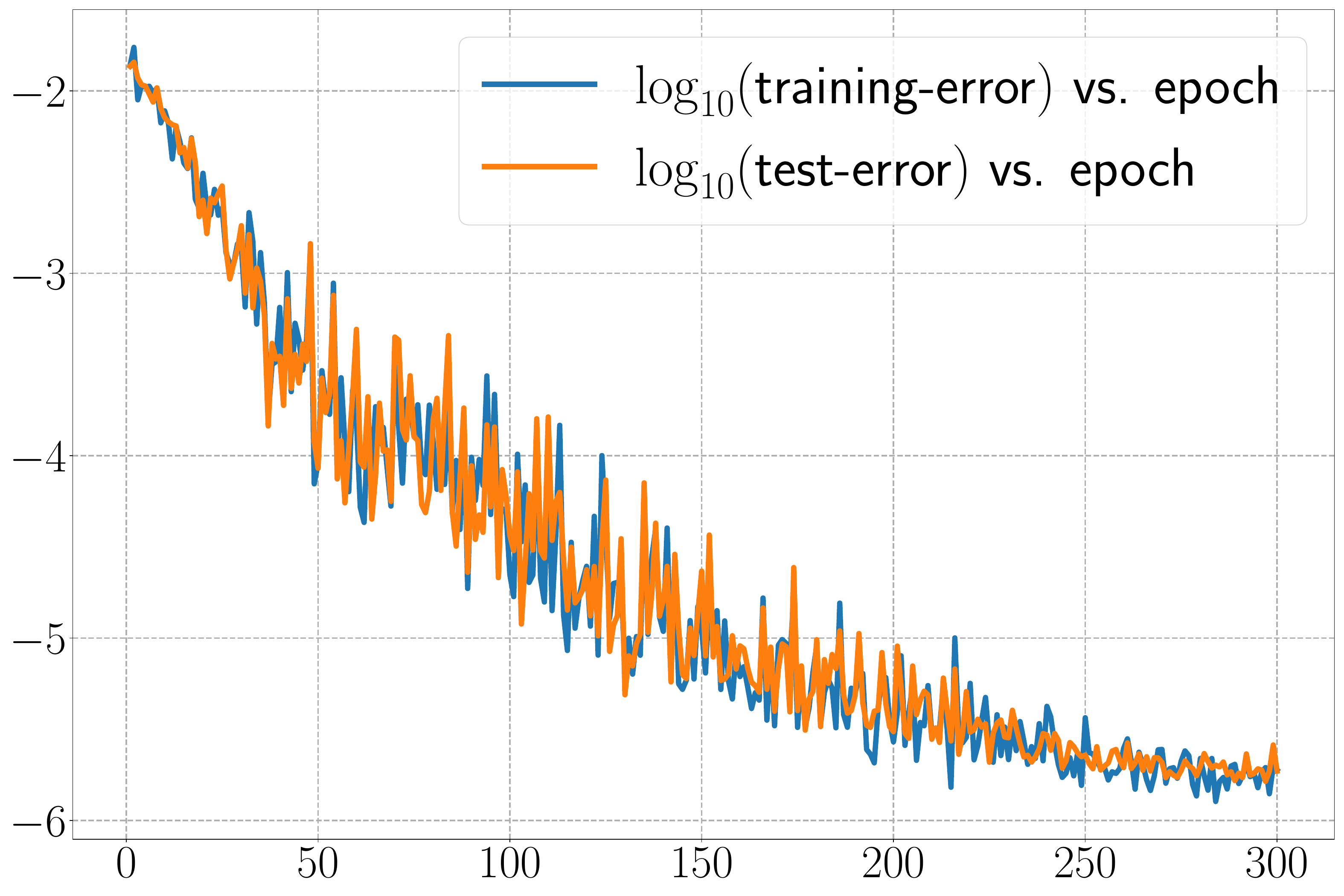}
    \caption{Training and test errors (in MSE) vs. epoch.}
    	\label{fig:errors:LD:fun2D}
    \end{minipage}
\end{figure}

Our MMNN is of a compact size $(500, 20, 8)$. For this test, a total of $600^2$ data are sampled on a uniform grid in $[-1,1]^2$ with a mini-batch size of $1000$ and a learning rate of $0.001 \times 0.9^{\lfloor k/6 \rfloor}$ for epochs $k=1,2,\cdots,300$.
Figure~\ref{fig:errors:LD:fun2D} gives the error plot. The training process shown in Figure~\ref{fig:2D-dynamics} illustrates that an overall coarse scale or low-frequency component of the shape is learned first and then localized features are learned one by one from coarse to fine.

\subsection{Tests in three dimensions and higher}\label{sec:higher}
Here,
we test a few examples in three and four dimensions. Even sampling an interesting function becomes challenging as the dimension becomes higher. Although our examples are limited by our computation power using a laptop, our tests show that MMNN performs well and is more effective than a fully connected network. 

The first example is a three-dimensional function a level set of which is shown in Figure~\ref{fig:f3D:spikyBall}. Using polar coordinates \((r, \theta, \phi)\), $\theta \in [0,\pi]$, $\phi\in[0,2\pi)$,
the target function \( f(x, y, z) \) is defined as:
\begin{equation*}
    f(r, \theta, \phi) = \begin{cases}
        0 & \text{if } 0.5 + 5\rho - 5r \leq 0, \\
        1 & \text{if } 0.5 + 5\rho - 5r \geq 1, \\
        0.5 + 5\rho - 5r & \text{otherwise},
    \end{cases}
\end{equation*}
where
\begin{equation*}
    \rho=\rho(\theta,\phi)=0.5+0.2\sin(6  \theta)\cos(6 \phi)\sin^2(\theta).
\end{equation*}

\begin{figure}
    \begin{minipage}[b]{0.30\linewidth}
    \centering    \includegraphics[width=0.9585\textwidth]{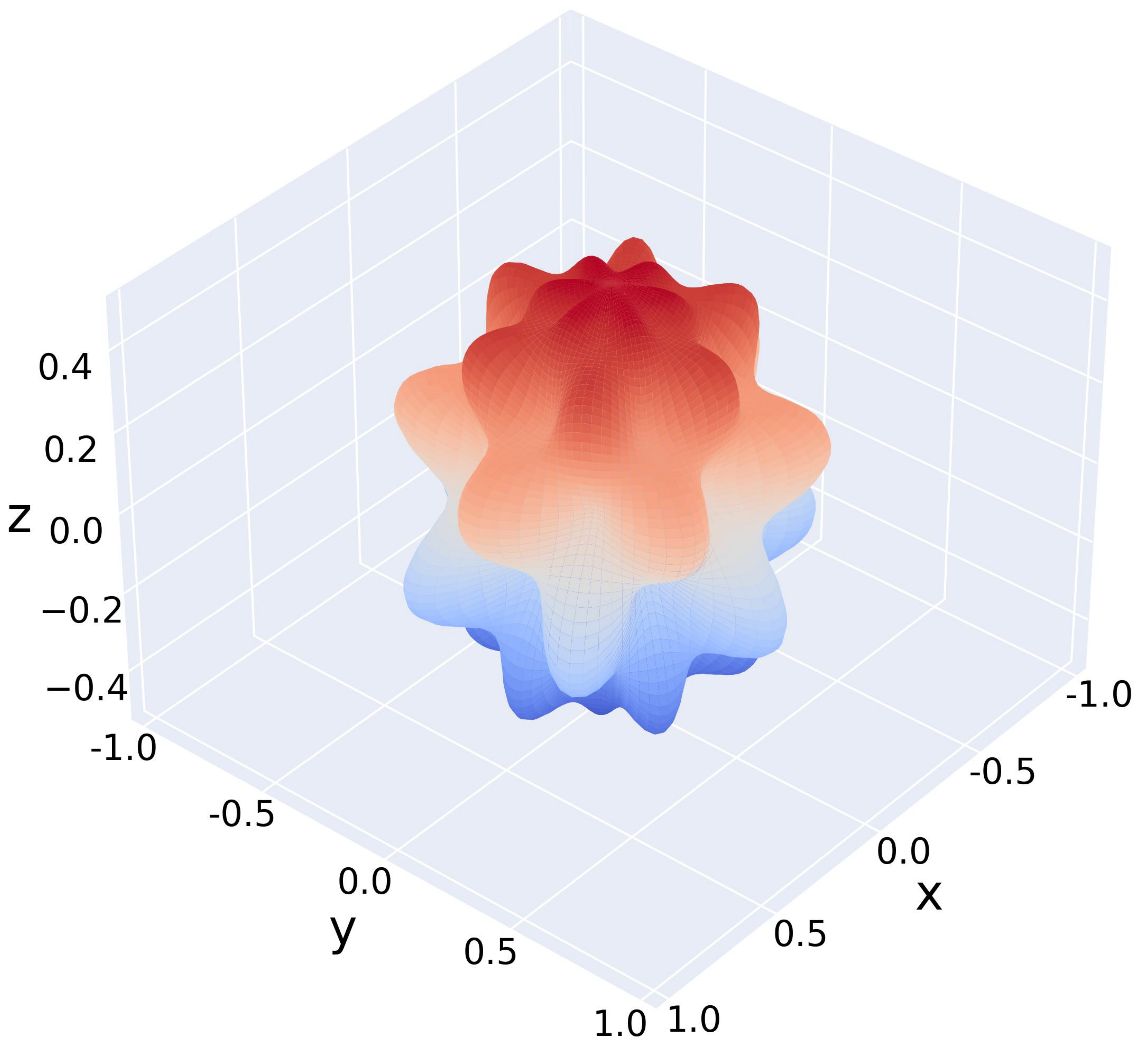}
    \caption{Surface of the levelset $f(r,\theta,\phi)=0.5$.}
    \label{fig:f3D:spikyBall}
\end{minipage}
\hfill
    \begin{minipage}[b]{0.30\linewidth}
    \centering    \includegraphics[width=0.9585\textwidth]{figures/HighDim/f3DspikyBall.pdf}
    \caption{Surface of the levelset $h(r,\theta,\phi)=0.5$.}
    \label{fig:h3D:spikyBall}
\end{minipage}
\hfill
\begin{minipage}[b]{0.33\linewidth}
    \centering
    \includegraphics[width=0.969985\textwidth]{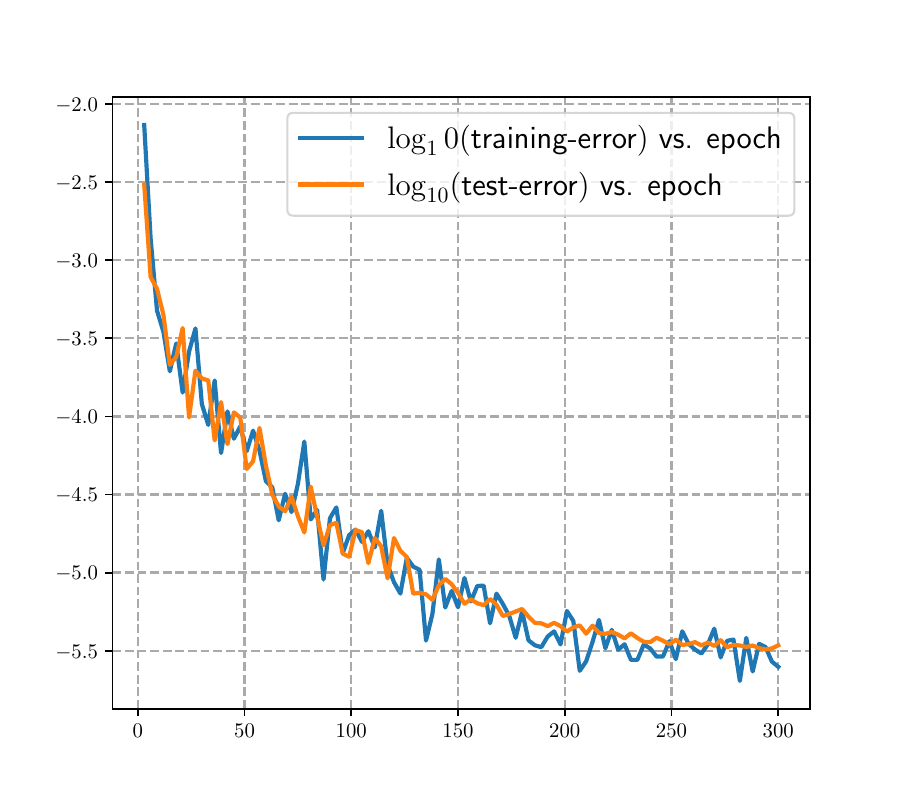}
 \caption{Training and test errors (MSE) vs. epoch.}
    \label{fig:3DlogErrorVsEpoch1}
    \end{minipage}
\end{figure}

Our MMNN is of a compact size $(600, 20, 8)$. For this test, a total of $111^3$ data are sampled on a uniform grid in $[-1,1]^3$ with a mini-batch size of $999$ and a learning rate of $0.0005 \times 0.9^{\lfloor k/6 \rfloor}$ for epochs $k=1,2,\cdots,300$.
Figure~\ref{fig:3DlogErrorVsEpoch1} gives the error plot. 
As shown in Figures~\ref{fig:f3D:spikyBall} and \ref{fig:h3D:spikyBall}, the levelsets corresponding to the target function $f$ and the learned MMNN approximation $h$ are nearly identical.
To visually demonstrate the quality of the approximation and complex structure of the three-dimensional function, we present several slices of the target function and the MMNN approximation by fixing either $x$, $y$, or $z$ in Figure~\ref{fig:3Dslices}.

\begin{figure}
    \centering	
       \begin{subfigure}[c]{0.3248\textwidth}
    \centering            \includegraphics[width=0.9869598055\textwidth]{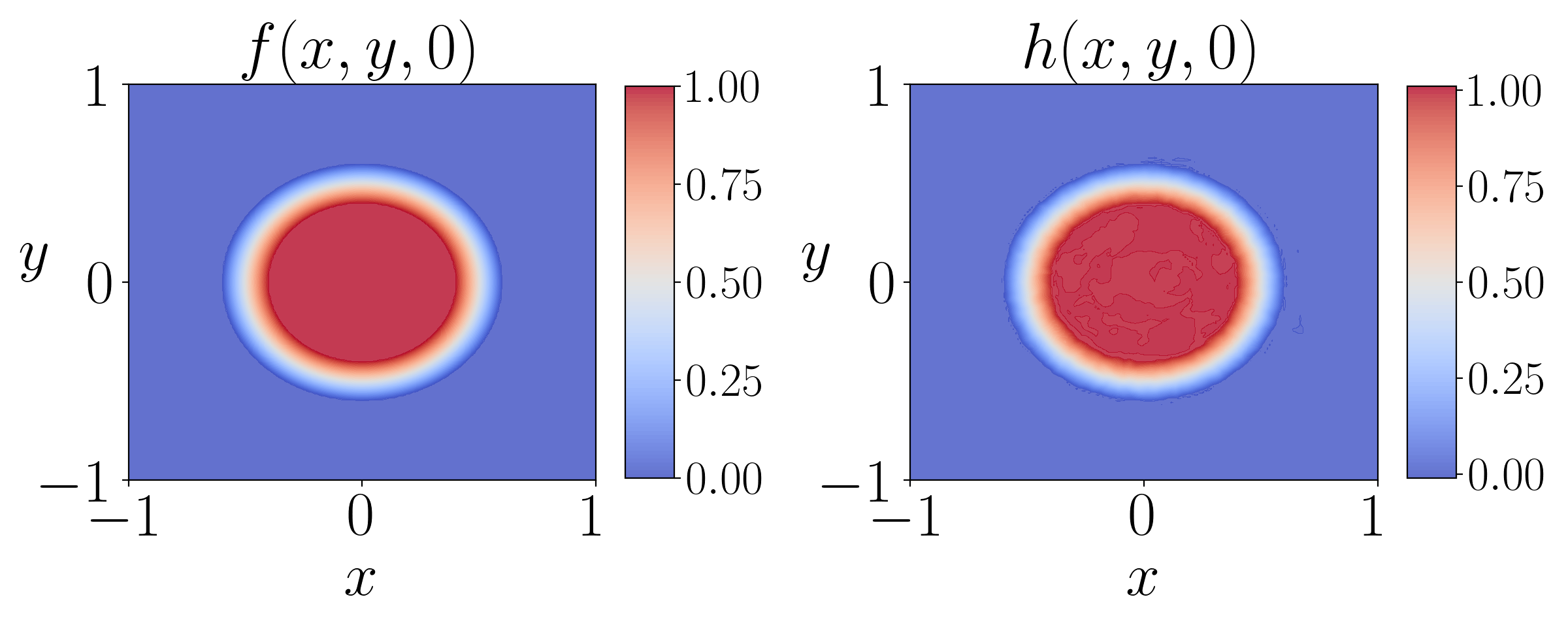}
    \subcaption{$z=0$.}
    \end{subfigure}
    \hfill
    \begin{subfigure}[c]{0.3248\textwidth}
    \centering            \includegraphics[width=0.9869598055\textwidth]{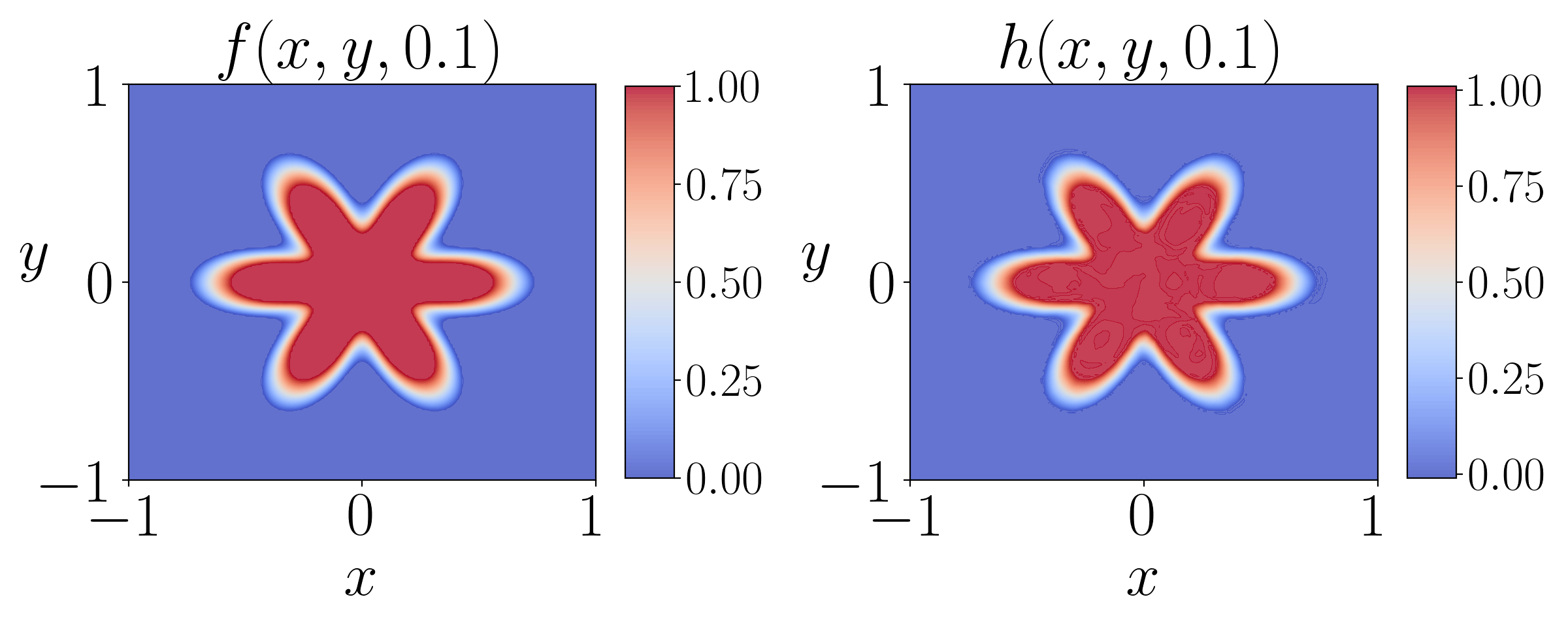}
    \subcaption{$z=0.1$.}
    \end{subfigure}\hfill             
    \begin{subfigure}[c]{0.3248\textwidth}
    \centering            \includegraphics[width=0.9869598055\textwidth]{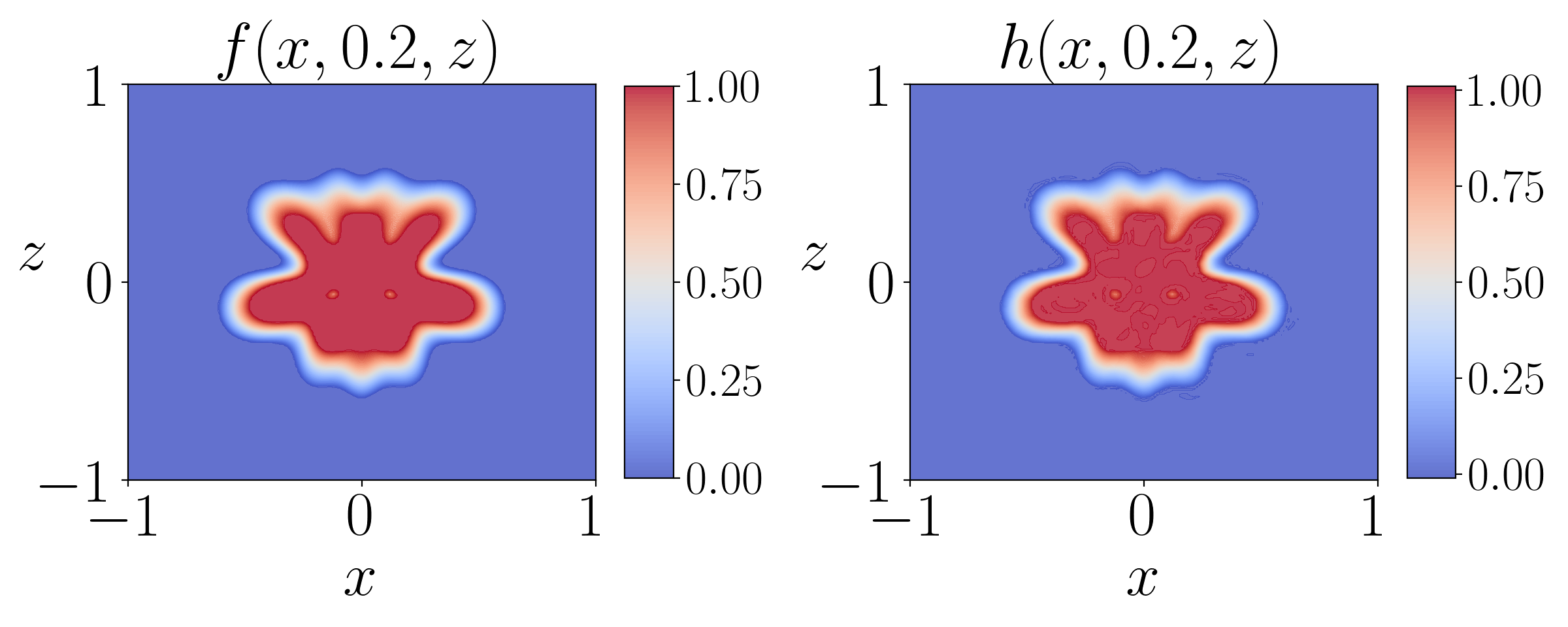}
    \subcaption{$y=0.2$.}
    \end{subfigure}\\
       \begin{subfigure}[c]{0.3248\textwidth}
    \centering            \includegraphics[width=0.9869598055\textwidth]{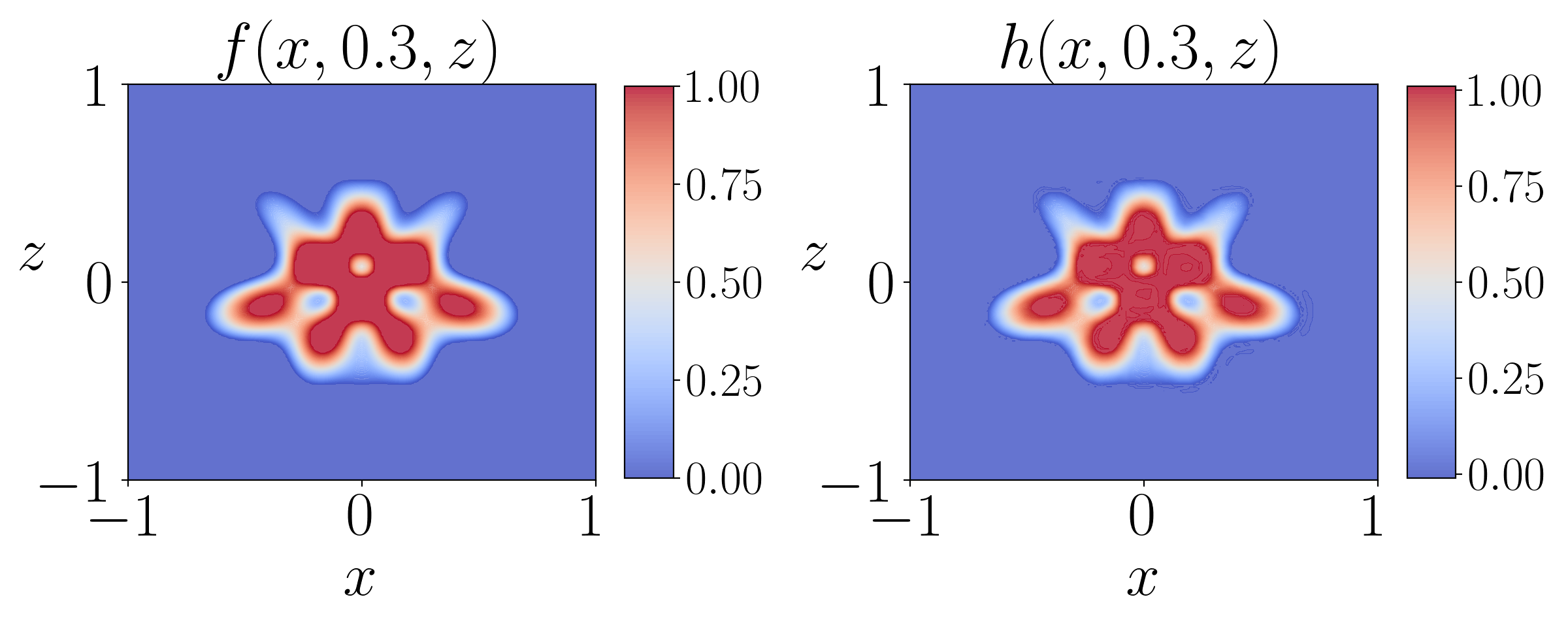}
    \subcaption{$y=0.3$.}
    \end{subfigure}
    \hfill
        \begin{subfigure}[c]{0.3248\textwidth}
    \centering            \includegraphics[width=0.9869598055\textwidth]{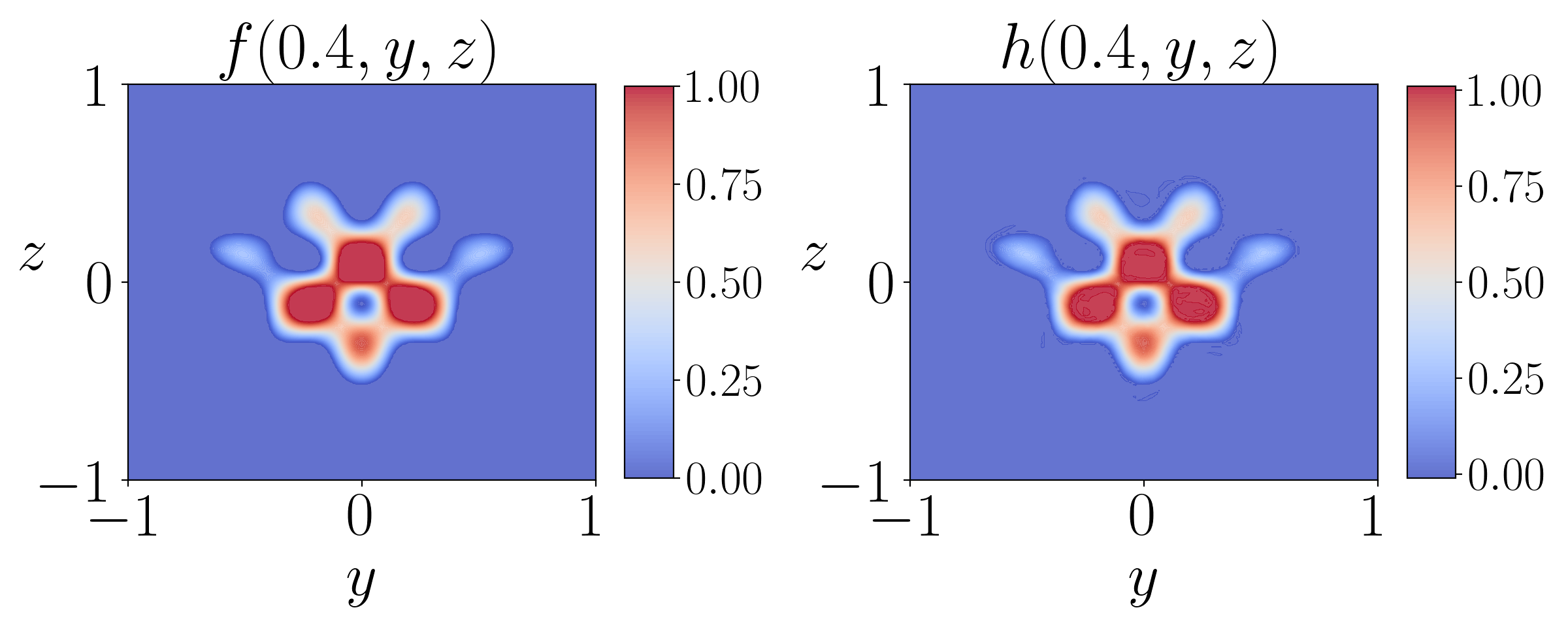}
    \subcaption{$x=0.4$.}
    \end{subfigure}
    \hfill
                \begin{subfigure}[c]{0.3248\textwidth}
    \centering            \includegraphics[width=0.9869598055\textwidth]{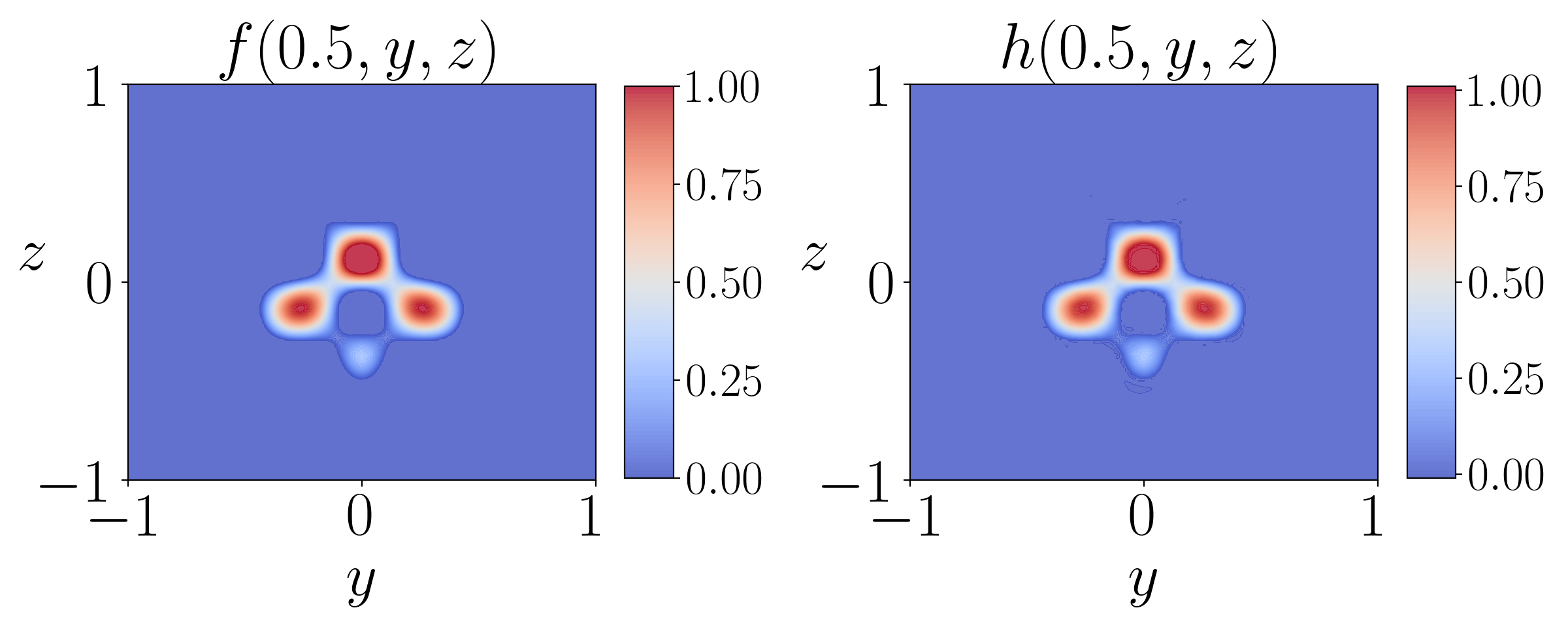}
    \subcaption{$x=0.5$.}
    \end{subfigure}
    \caption{Slices of the true function $f(x, y, z)$ vs. those of the MMNN approximation $h(x, y, z)$.}   \label{fig:3Dslices}
\end{figure}




Next, we consider the probability density function (PDF) of a Gaussian (normal) distribution in four dimensions
\[
f(\bm{x}) = f(x_1, \dots, x_4) = \frac{\exp\left(-\frac{1}{2} (\bm{x}-\boldsymbol{\mu})^\ts  \bmSigma^{-1} (\bm{x}-\bm{\mu})\right)}{\sqrt{(2\pi)^k \det(\bmSigma)}}
\]
where \(\bmSigma\) is the covariance matrix. 
We set \(\bm{\mu} = \bm{0}\)  and
\begin{equation*}
    \bmSigma^{-1}=20\begin{bmatrix}
    1.0 & 0.9 & 0.8 & 0.7 \\
    0.9 & 2.0 & 1.9 & 1.8 \\
    0.8 & 1.9 & 3.0 & 2.9 \\
    0.7 & 1.8 & 2.9 & 4.0 \\
\end{bmatrix}.
\end{equation*}
We remark that the eigenvalues of 
$\bmSigma^{-1}$ are $6.82,  9.93, 25.28, 158.05$ which means that the distribution is quite anisotropic and concentrated near the center.

A compact MMNN with size of $(500,12,6)$ produces a good approximation as shown in the error plot Figure~\ref{fig:logErrorVsEpoch2}. Figure~\ref{fig:4D:f:vs:h} compares the true function $f(x, y, z, u)$ and  the MMNN approximation $h(x, y, z, u)$ with $z=u=0.2$. For this test a total of $35^4$ data are sampled on a uniform grid in $[-1,1]^4$ with a mini-batch size of $35^2$ and a learning rate at $10^{-3} \times 0.9^{\lfloor k/6 \rfloor}$ for epochs $k=1,2,\cdots,300$.

\begin{figure}
\begin{minipage}[b]{0.5861\linewidth}
    \centering
    \includegraphics[width=0.94569965\textwidth]{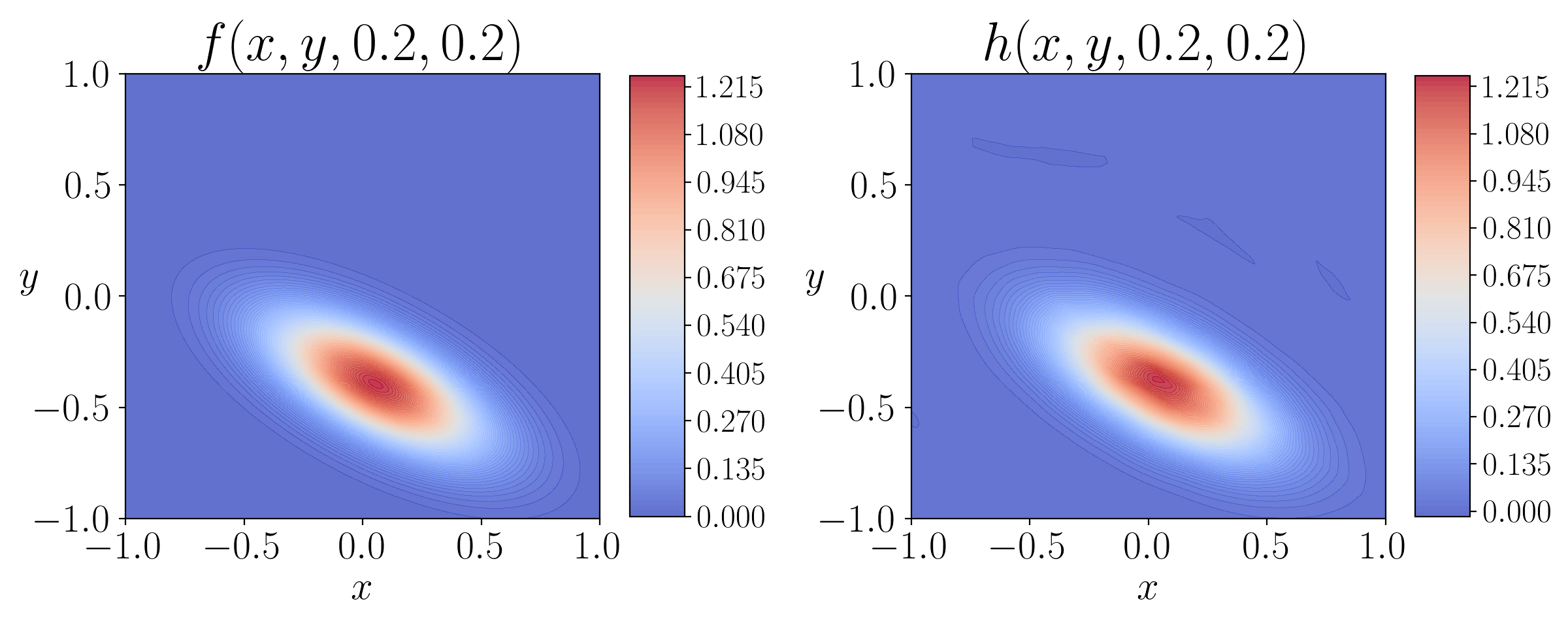}
    \caption{True function $f(x, y, z, u)$ versus the learned network $h(x, y, z, u)$ with $z=u=0.2$.}  
    \label{fig:4D:f:vs:h}
\end{minipage}
\hfill
\begin{minipage}[b]{0.3565\linewidth}
    \centering
    \includegraphics[width=0.97965\textwidth]{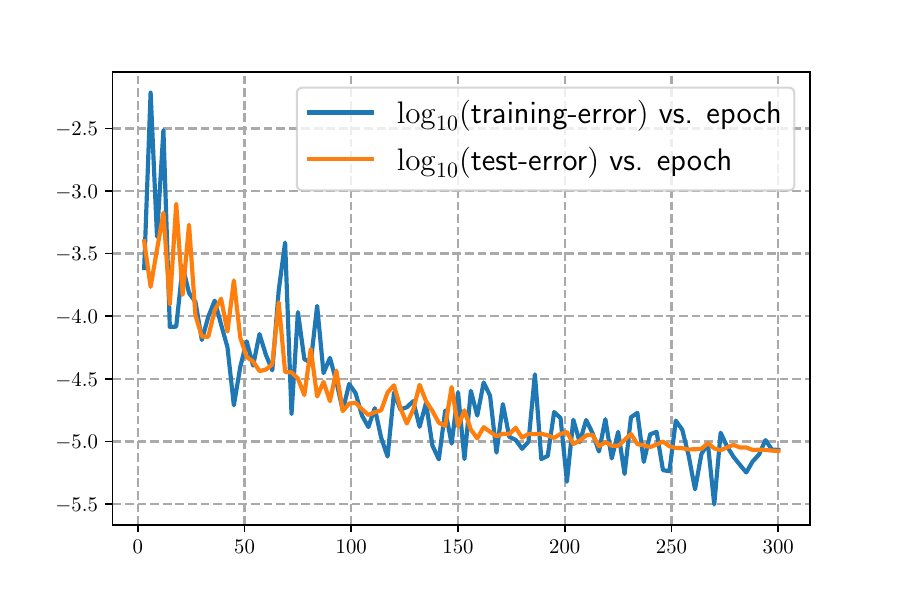}
 \caption{Training and test errors (MSE) vs. epoch.}
    \label{fig:logErrorVsEpoch2}
    \end{minipage}
\end{figure}


\begin{figure}[htbp!]
    \centering
    \begin{subfigure}[c]{0.23\textwidth}
        \centering
        \includegraphics[width=0.998055\textwidth]{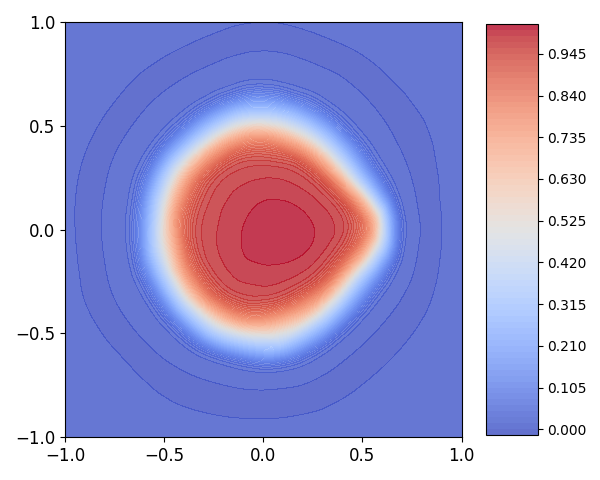}
        \subcaption{Epoch 1.}
    \end{subfigure}
    \hfill
    \begin{subfigure}[c]{0.23\textwidth}
        \centering
        \includegraphics[width=0.998055\textwidth]{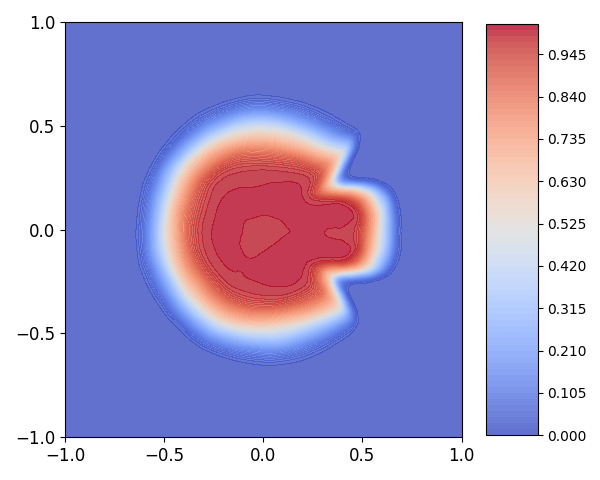}
        \subcaption{Epoch 3.}
    \end{subfigure}
    \hfill
    \begin{subfigure}[c]{0.23\textwidth}
        \centering
        \includegraphics[width=0.998055\textwidth]{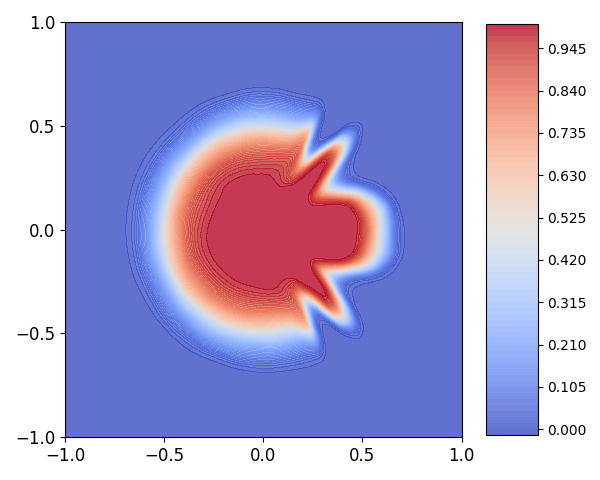}
        \subcaption{Epoch 5.}
    \end{subfigure}
    \hfill
    \begin{subfigure}[c]{0.23\textwidth}
        \centering
        \includegraphics[width=0.998055\textwidth]{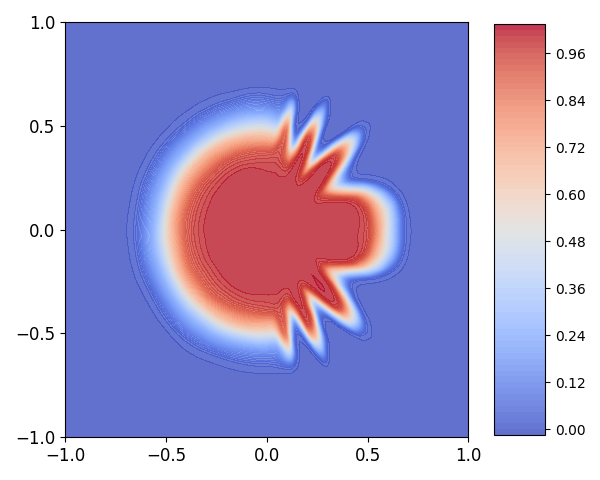}
        \subcaption{Epoch 7.}
    \end{subfigure}
\\
\begin{subfigure}[c]{0.23\textwidth}
        \centering
        \includegraphics[width=0.998055\textwidth]{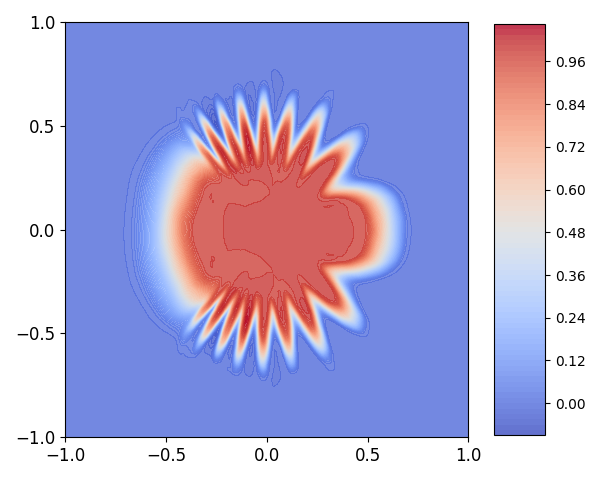}
        \subcaption{Epoch 14.}
    \end{subfigure}
    \hfill
    \begin{subfigure}[c]{0.23\textwidth}
        \centering
        \includegraphics[width=0.998055\textwidth]{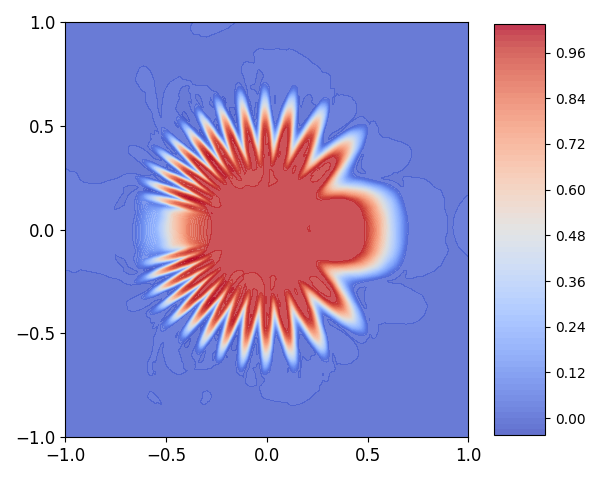}
        \subcaption{Epoch 22.}
    \end{subfigure}
    \hfill
    \begin{subfigure}[c]{0.23\textwidth}
        \centering
        \includegraphics[width=0.998055\textwidth]{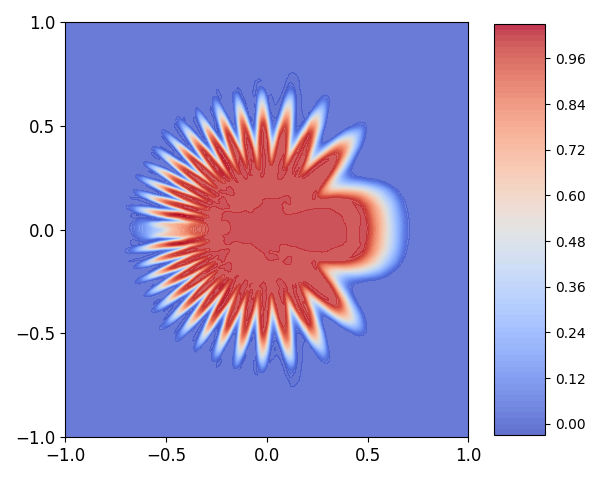}
        \subcaption{Epoch 30.}
    \end{subfigure}
    \hfill
    \begin{subfigure}[c]{0.23\textwidth}
        \centering
        \includegraphics[width=0.998055\textwidth]{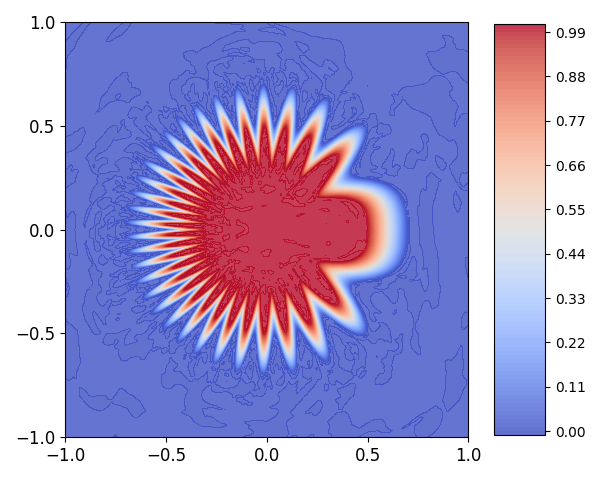}
        \subcaption{Epoch 300.}
    \end{subfigure}
    
    \caption{Illustration of the learning dynamics.}
    \label{fig:2D-dynamics}
\end{figure}

\section{Further discussion}
\label{sec:further:discussion}

In this section, we provide a few more comments about MMNNs. First, in Section~\ref{sec:advantages:MMNN:vs:FCNN}, we discuss the advantages of MMNNs over fully connected networks (FCNNs) or multi-layer perceptrons (MLPs). Next, in Section~\ref{sec:guideline:MMNN:size}, we offer practical guidelines for determining the appropriate MMNN size based on our theoretical understanding and extensive numerical experiments. Finally in Section~\ref{sec:other:activation:functions}, we discuss the use of alternative activation functions beyond \ReLU\ in MMNNs.

\subsection{Advantages compared to FCNNs or MLPs}
\label{sec:advantages:MMNN:vs:FCNN}

The two key differences between a standard FCNN or MLP  and an MMNN are (1) the introduction of the weights $\bmA, \bmc$ for different linear combinations of hidden neurons (or perceptrons) as the multi-components in each layer, and (2) the training strategy that fixes those randomly initialized $\bmW,\bmb$ (random features) in the hidden neurons. Hence it is extremely easy to modify a FCNN or MLP to an MMNN. 

    MMNNs are much more effective than FCNNs in terms of representation, training, and accuracy especially for complex functions. In comparison, as shown in those experiments in Section~\ref{sec:comparison}, MMNNs (1) have much fewer training parameters, (2) converge much faster in training, (3) achieve much better accuracy. Moreover, experiments show that training process of MMNNs converges not only faster but also with a steady rate while FCNNs saturates pretty early to a quite low accuracy, as commonly observed in practices. These nice  behaviors of MMNNs are due to their balanced structure for smooth decomposition as well as the training strategy. In practice, 
    the introduction of $\bmA, \bmc$ in MMNNs provides an important balance between the network width, which is the number of hidden neurons (basis functions) and can be very large, and the dimension of the input space, which is the number of components from the previous layer and can be much smaller than the network width. In other words, using a few linear combinations of the basis functions can capture smooth structures in the input space well. On the other hand, for FCNNs the two are the same and no balance is exerted.  

\subsection{Practical guidelines for MMNNs}
\label{sec:guideline:MMNN:size}

    There are three hyperparameters for the configuration of MMNN sizes, the network width, the number of components (rank), and the number of layers (depth). Here are the general guidelines based on our mathematical construction and extensive experiments:
    \begin{enumerate}
   \item The network width should provide enough resolution to capture fine details of the target function. This means that the width should be at least comparable to the size of an adaptive mesh that can approximate the target function well. 
\item The number of components (rank) is related to the overall complexity of the target function which depends on its spatial domain and Fourier domain representation as well as the input dimension. As indicated by our mathematical multi-component construction, it is related to the ``divide and conquer" strategy. 
\item The number of layers (depth) is also related to the overall complexity of the target function as for the number of components. Rank and depth are complementary but work together effectively for a smooth decomposition of the target function. The rule of thumb for depth is similar to that for the rank. 
    \end{enumerate}

    Here we use more concrete examples to illustrate the guidelines. For simplicity we fix the input dimension and domain of interest. As the domain size and dimension increases, the network size needs to increase correspondingly. For a smooth target function, a compact MMNN in terms of width, rank, and depth is enough and easy training process can render accurate results. Larger MMNNs are needed for target functions with localized rapid changes. Even with a relative compact size, the training process can allocate resources adaptive to the target function and render good approximation. The most difficult situation is to approximate globally highly oscillatory functions with diverse Fourier modes for which large MMNNs are needed. For instance, if the oscillation frequency doubles, the network width should increase by $2^d$ where $d$ is the dimension. In general the network width needs to deal with the curse of dimensionality just like a mesh based method. However, the growth of the number of components and layers with the increase of complexity seems to be relative mild (maybe polylogarithmic suggested by our mathematical construction).

    Overall, for a given target function, MMNNs can work well with quite a large range of configuration with a trade-off between the network size and training process. For example, the training process for a network more on the compact size with respect to the complexity of a given target function may become more subtle and challenging, e.g., choosing the appropriate learning rate and batch size, due to the lack of flexibility (or redundancy) of the representation. On the other hand, a network of too large size (or redundancy) with respect to the complexity of a given target function requires unnecessarily expensive training cost. There is a trade-off between representation and optimization one needs to balance in practice. An important question for future research is how to develop a self-adaptive strategy to adjust the network size.
    
    The most advantageous situation for using MMNNs is when approximating a function in relative high dimension which is mostly smooth except for localized fine features, e.g., a distribution in high dimensions concentrated on a low dimensional manifold. Through training, MMNNs can provide an automatic adaptive approximation of the underlying structure which can be challenging for a mesh based method. 
 
    We would like to remark that learning rate scheduler can be a subtle and important issue for all training process in practice. For all our training process, the Step Learning Rate suffices. However, one could consider using other learning rate schedulers, such as the Cosine Scheduler \cite{loshchilov2017sgdr} or the gradual warm-up strategy \cite{2017arXiv170602677G}. Exploring and designing a more efficient learning rate scheduler with some automatic restart mechanism is a potential interesting topic for future work.


\subsection{Beyond \texttt{ReLU} to other activation functions}
\label{sec:other:activation:functions}

We also tried using different activation functions for MMNNs, e.g., \GELU\ \cite{2016arXiv160608415H}, \Swish\ \cite{2017arXiv171005941R}, \Sigmoid, and \Tanh. In general, \texttt{ReLU} provides the overall best results for various target functions. 
However, in situations where a smooth (e.g., $C^s$ or real analytic) approximation  is needed, one might consider using smooth alternatives to \ReLU\ such as \GELU\  or \Swish, which generally yield results comparable to \ReLU.

Additionally, other popular S-shaped activation functions like \Sigmoid\ and \Tanh\ have demonstrated poor performance in our tests, possibly due to the vanishing gradient problem. For highly oscillatory target functions, when using \Sigmoid\ or \Tanh\ training errors did not even decrease  during the training process.

\section{Conclusion}
\label{sec:conclusion}

In this work, we introduced the Multi-component and Multi-layer Neural Network (MMNN) and demonstrated its effectiveness in approximating complex functions. By incorporating the principles of structured and balanced decomposition, the MMNN architecture addresses the limitations of shallow networks, particularly in capturing high-frequency components and localized fine features. Our proposed network structure as confirmed by extensive numerical experiments can approximate highly oscillatory functions and functions with rapid transitions efficiently and accurately. Additionally, we highlight the advantages of our training strategy, which optimizes only the linear combination weights of basis functions for each component while keeping the parameters within the activation (basis) functions fixed, leading to a more efficient and stable training process.

The theoretical underpinnings and practical implementations presented in this paper suggest that MMNNs offer a promising direction for constructing neural networks capable of handling complex tasks with fewer parameters and reduced computational overhead. Future research can explore further generalizations and applications of MMNNs, as well as investigate the interplay between representation and optimization in more depth.

\section*{Acknowledgments}

S. Zhang was partially supported by the 
start-up fund P0053092
from The Hong Kong Polytechnic University.
H. Zhao was partially supported by NSF grants DMS-2309551, and DMS-2012860. Y. Zhong was partially supported by NSF grant DMS-2309530. H. Zhou was partially supported by NSF grant DMS-2307465.


{
\small
\let\oldhref\href
\def\href#1#2{\oldhref{#1}{\nolinkurl{#2}}}
\bibliographystyle{plain}   
\bibliography{references}

\begin{thebibliography}{10}

\bibitem{NEURIPS2020_288cd256}
Jordan Ash and Ryan~P Adams.
\newblock On warm-starting neural network training.
\newblock In H.~Larochelle, M.~Ranzato, R.~Hadsell, M.F. Balcan, and H.~Lin,
  editors, {\em Advances in Neural Information Processing Systems}, volume~33,
  pages 3884--3894. Curran Associates, Inc., 2020.

\bibitem{doi:10.1137/18M118709X}
Helmut. B{\"o}lcskei, Philipp. Grohs, Gitta. Kutyniok, and Philipp. Petersen.
\newblock Optimal approximation with sparsely connected deep neural networks.
\newblock {\em SIAM Journal on Mathematics of Data Science}, 1(1):8--45, 2019.

\bibitem{10.3389/fams.2018.00014}
Charles~K. Chui, Shao-Bo Lin, and Ding-Xuan Zhou.
\newblock Construction of neural networks for realization of localized deep
  learning.
\newblock {\em Frontiers in Applied Mathematics and Statistics}, 4:14, 2018.

\bibitem{Cybenko1989ApproximationBS}
George Cybenko.
\newblock Approximation by superpositions of a sigmoidal function.
\newblock {\em Mathematics of Control, Signals, and Systems}, 2:303--314, 1989.

\bibitem{pmlr-v9-glorot10a}
Xavier Glorot and Yoshua Bengio.
\newblock Understanding the difficulty of training deep feedforward neural
  networks.
\newblock In Yee~Whye Teh and Mike Titterington, editors, {\em Proceedings of
  the Thirteenth International Conference on Artificial Intelligence and
  Statistics}, volume~9 of {\em Proceedings of Machine Learning Research},
  pages 249--256, Chia Laguna Resort, Sardinia, Italy, 13--15 May 2010. PMLR.

\bibitem{2017arXiv170602677G}
Priya {Goyal}, Piotr {Doll{\'a}r}, Ross {Girshick}, Pieter {Noordhuis}, Lukasz
  {Wesolowski}, Aapo {Kyrola}, Andrew {Tulloch}, Yangqing {Jia}, and Kaiming
  {He}.
\newblock Accurate, large minibatch sgd: Training imagenet in 1 hour.
\newblock {\em arXiv e-prints}, page arXiv:1706.02677, June 2017.

\bibitem{2019arXiv190501208G}
R{\'e}mi Gribonval, Gitta Kutyniok, Morten Nielsen, and Felix Voigtlaender.
\newblock Approximation spaces of deep neural networks.
\newblock {\em Constructive Approximation}, 55:259--367, 2022.

\bibitem{2019arXiv190207896G}
Ingo {G{\"u}hring}, Gitta {Kutyniok}, and Philipp {Petersen}.
\newblock Error bounds for approximations with deep {ReLU} neural networks in
  ${W}^{s,p}$ norms.
\newblock {\em Analysis and Applications}, 18(05):803--859, 2020.

\bibitem{7780459}
Kaiming He, Xiangyu Zhang, Shaoqing Ren, and Jian Sun.
\newblock Deep residual learning for image recognition.
\newblock In {\em 2016 IEEE Conference on Computer Vision and Pattern
  Recognition (CVPR)}, pages 770--778, June 2016.

\bibitem{2016arXiv160608415H}
Dan {Hendrycks} and Kevin {Gimpel}.
\newblock Gaussian error linear units {(GELUs)}.
\newblock {\em arXiv e-prints}, page arXiv:1606.08415, June 2016.

\bibitem{HORNIK1991251}
Kurt Hornik.
\newblock Approximation capabilities of multilayer feedforward networks.
\newblock {\em Neural Networks}, 4(2):251--257, 1991.

\bibitem{HORNIK1989359}
Kurt Hornik, Maxwell Stinchcombe, and Halbert White.
\newblock Multilayer feedforward networks are universal approximators.
\newblock {\em Neural Networks}, 2(5):359--366, 1989.

\bibitem{9157223}
Yerlan Idelbayev and Miguel~{\'A}. Carreira-Perpi{\~n}{\'a}n.
\newblock Low-rank compression of neural nets: Learning the rank of each layer.
\newblock In {\em 2020 IEEE/CVF Conference on Computer Vision and Pattern
  Recognition (CVPR)}, pages 8046--8056, 2020.

\bibitem{ISMAYILOVA2024106333}
Aysu Ismayilova and Vugar~E. Ismailov.
\newblock On the {Kolmogorov} neural networks.
\newblock {\em Neural Networks}, 176:106333, 2024.

\bibitem{NEURIPS2024_ca98452d}
Dayal~Singh Kalra and Maissam Barkeshli.
\newblock Why warmup the learning rate? underlying mechanisms and improvements.
\newblock In A.~Globerson, L.~Mackey, D.~Belgrave, A.~Fan, U.~Paquet,
  J.~Tomczak, and C.~Zhang, editors, {\em Advances in Neural Information
  Processing Systems}, volume~37, pages 111760--111801. Curran Associates,
  Inc., 2024.

\bibitem{DBLP:journals/corr/KingmaB14}
Diederik~P. Kingma and Jimmy Ba.
\newblock Adam: {A} method for stochastic optimization.
\newblock In Yoshua Bengio and Yann LeCun, editors, {\em 3rd International
  Conference on Learning Representations, {ICLR} 2015, San Diego, CA, USA, May
  7-9, 2015, Conference Track Proceedings}, 2015.

\bibitem{kolmogorov1957}
A.~N. Kolmogorov.
\newblock On the representation of continuous functions of several variables by
  superposition of continuous functions of one variable and addition.
\newblock {\em Dokl. Akad. Nauk SSSR}, pages 953--956, 1957.

\bibitem{9495136}
Fanghui Liu, Xiaolin Huang, Yudong Chen, and Johan A.~K. Suykens.
\newblock Random features for kernel approximation: A survey on algorithms,
  theory, and beyond.
\newblock {\em IEEE Transactions on Pattern Analysis and Machine Intelligence},
  44(10):7128--7148, 2022.

\bibitem{2024arXiv240419756L}
Ziming {Liu}, Yixuan {Wang}, Sachin {Vaidya}, Fabian {Ruehle}, James
  {Halverson}, Marin {Solja{\v{c}}i{\'c}}, Thomas~Y. {Hou}, and Max {Tegmark}.
\newblock {KAN: Kolmogorov-Arnold} networks.
\newblock {\em arXiv e-prints}, page arXiv:2404.19756, April 2024.

\bibitem{loshchilov2017sgdr}
Ilya Loshchilov and Frank Hutter.
\newblock {SGDR:} stochastic gradient descent with warm restarts.
\newblock In {\em 5th International Conference on Learning Representations,
  ICLR 2017, Toulon, France, April 24-26, 2017, Conference Track Proceedings},
  2017.

\bibitem{shijun:smooth:functions}
Jianfeng Lu, Zuowei Shen, Haizhao Yang, and Shijun Zhang.
\newblock Deep network approximation for smooth functions.
\newblock {\em SIAM Journal on Mathematical Analysis}, 53(5):5465--5506, 2021.

\bibitem{MAIOROV199981}
Vitaly Maiorov and Allan Pinkus.
\newblock Lower bounds for approximation by {MLP} neural networks.
\newblock {\em Neurocomputing}, 25(1):81--91, 1999.

\bibitem{MO}
Hadrien Montanelli and Haizhao Yang.
\newblock Error bounds for deep {ReLU} networks using the {Kolmogorov-Arnold}
  superposition theorem.
\newblock {\em Neural Networks}, 129:1--6, 2020.

\bibitem{novikov2015tensorizing}
Alexander Novikov, Dmitrii Podoprikhin, Anton Osokin, and Dmitry~P Vetrov.
\newblock Tensorizing neural networks.
\newblock {\em Advances in neural information processing systems}, 28, 2015.

\bibitem{peng2021random}
Hao Peng, Nikolaos Pappas, Dani Yogatama, Roy Schwartz, Noah Smith, and
  Lingpeng Kong.
\newblock Random feature attention.
\newblock In {\em International Conference on Learning Representations}, 2021.

\bibitem{NIPS2007_013a006f}
Ali Rahimi and Benjamin Recht.
\newblock Random features for large-scale kernel machines.
\newblock In J.~Platt, D.~Koller, Y.~Singer, and S.~Roweis, editors, {\em
  Advances in Neural Information Processing Systems}, volume~20. Curran
  Associates, Inc., 2007.

\bibitem{RAISSI2019686}
Maziar Raissi, Paris Perdikaris, and George~Em Karniadakis.
\newblock Physics-informed neural networks: A deep learning framework for
  solving forward and inverse problems involving nonlinear partial differential
  equations.
\newblock {\em Journal of Computational Physics}, 378:686--707, 2019.

\bibitem{2017arXiv171005941R}
Prajit {Ramachandran}, Barret {Zoph}, and Quoc~V. {Le}.
\newblock Searching for activation functions.
\newblock {\em arXiv e-prints}, page arXiv:1710.05941, October 2017.

\bibitem{2022arXiv220913569R}
Siddhartha {Rao Kamalakara}, Acyr {Locatelli}, Bharat {Venkitesh}, Jimmy {Ba},
  Yarin {Gal}, and Aidan~N. {Gomez}.
\newblock Exploring low rank training of deep neural networks.
\newblock {\em arXiv e-prints}, page arXiv:2209.13569, September 2022.

\bibitem{NIPS2017_61b1fb3f}
Alessandro Rudi and Lorenzo Rosasco.
\newblock Generalization properties of learning with random features.
\newblock In I.~Guyon, U.~Von Luxburg, S.~Bengio, H.~Wallach, R.~Fergus,
  S.~Vishwanathan, and R.~Garnett, editors, {\em Advances in Neural Information
  Processing Systems}, volume~30. Curran Associates, Inc., 2017.

\bibitem{6638949}
Tara~N. Sainath, Brian Kingsbury, Vikas Sindhwani, Ebru Arisoy, and Bhuvana
  Ramabhadran.
\newblock Low-rank matrix factorization for deep neural network training with
  high-dimensional output targets.
\newblock In {\em 2013 IEEE International Conference on Acoustics, Speech and
  Signal Processing}, pages 6655--6659, 2013.

\bibitem{semenova2022existence}
Lesia Semenova, Cynthia Rudin, and Ronald Parr.
\newblock On the existence of simpler machine learning models.
\newblock In {\em Proceedings of the 2022 ACM Conference on Fairness,
  Accountability, and Transparency}, pages 1827--1858, 2022.

\bibitem{shijun:NonlineArpprox}
Zuowei Shen, Haizhao Yang, and Shijun Zhang.
\newblock Nonlinear approximation via compositions.
\newblock {\em Neural Networks}, 119:74--84, 2019.

\bibitem{shijun:Characterized:by:Numer:Neurons}
Zuowei Shen, Haizhao Yang, and Shijun Zhang.
\newblock Deep network approximation characterized by number of neurons.
\newblock {\em Communications in Computational Physics}, 28(5):1768--1811,
  2020.

\bibitem{shijun:arbitrary:error:with:fixed:size}
Zuowei Shen, Haizhao Yang, and Shijun Zhang.
\newblock Deep network approximation: Achieving arbitrary accuracy with fixed
  number of neurons.
\newblock {\em Journal of Machine Learning Research}, 23(276):1--60, 2022.

\bibitem{shijun:intrinsic:parameters}
Zuowei Shen, Haizhao Yang, and Shijun Zhang.
\newblock Deep network approximation in terms of intrinsic parameters.
\newblock In Kamalika Chaudhuri, Stefanie Jegelka, Le~Song, Csaba Szepesvari,
  Gang Niu, and Sivan Sabato, editors, {\em Proceedings of the 39th
  International Conference on Machine Learning}, volume 162 of {\em Proceedings
  of Machine Learning Research}, pages 19909--19934. PMLR, 17--23 Jul 2022.

\bibitem{shijun:net:arc:beyond:width:depth}
Zuowei {Shen}, Haizhao {Yang}, and Shijun {Zhang}.
\newblock Neural network architecture beyond width and depth.
\newblock In S.~Koyejo, S.~Mohamed, A.~Agarwal, D.~Belgrave, K.~Cho, and A.~Oh,
  editors, {\em Advances in Neural Information Processing Systems}, volume~35,
  pages 5669--5681. Curran Associates, Inc., 2022.

\bibitem{shijun:optimal:rate:in:width:and:depth}
Zuowei Shen, Haizhao Yang, and Shijun Zhang.
\newblock Optimal approximation rate of {ReLU} networks in terms of width and
  depth.
\newblock {\em Journal de Mathématiques Pures et Appliquées}, 157:101--135,
  2022.

\bibitem{NIPS2016_e7061188}
Aman Sinha and John~C Duchi.
\newblock Learning kernels with random features.
\newblock In D.~Lee, M.~Sugiyama, U.~Luxburg, I.~Guyon, and R.~Garnett,
  editors, {\em Advances in Neural Information Processing Systems}, volume~29.
  Curran Associates, Inc., 2016.

\bibitem{yarotsky2017}
Dmitry Yarotsky.
\newblock Error bounds for approximations with deep {ReLU} networks.
\newblock {\em Neural Networks}, 94:103--114, 2017.

\bibitem{yarotsky18a}
Dmitry Yarotsky.
\newblock Optimal approximation of continuous functions by very deep {ReLU}
  networks.
\newblock In S\'ebastien Bubeck, Vianney Perchet, and Philippe Rigollet,
  editors, {\em Proceedings of the 31st Conference On Learning Theory},
  volume~75 of {\em Proceedings of Machine Learning Research}, pages 639--649.
  PMLR, 06--09 Jul 2018.

\bibitem{shijun:thesis}
Shijun Zhang.
\newblock Deep neural network approximation via function compositions.
\newblock {\em PhD Thesis, National University of Singapore}, 2020.

\bibitem{shijun:RCNet}
Shijun Zhang, Jianfeng Lu, and Hongkai Zhao.
\newblock On enhancing expressive power via compositions of single fixed-size
  {R}e{LU} network.
\newblock In Andreas Krause, Emma Brunskill, Kyunghyun Cho, Barbara Engelhardt,
  Sivan Sabato, and Jonathan Scarlett, editors, {\em Proceedings of the 40th
  International Conference on Machine Learning}, volume 202 of {\em Proceedings
  of Machine Learning Research}, pages 41452--41487. PMLR, 23--29 Jul 2023.

\bibitem{JMLR:v25:23-0912}
Shijun Zhang, Jianfeng Lu, and Hongkai Zhao.
\newblock Deep network approximation: Beyond relu to diverse activation
  functions.
\newblock {\em Journal of Machine Learning Research}, 25(35):1--39, 2024.

\bibitem{ZZZZ-23}
Shijun Zhang, Hongkai Zhao, Yimin Zhong, and Haomin Zhou.
\newblock Why shallow networks struggle to approximate and learn high
  frequencies.
\newblock {\em arXiv e-prints}, page arXiv:2306.17301, June 2023.

\bibitem{ZHOU2019}
Ding-Xuan Zhou.
\newblock Universality of deep convolutional neural networks.
\newblock {\em Applied and Computational Harmonic Analysis}, 48(2):787--794,
  2020.

\end{thebibliography}
}

\end{document}